\documentclass{article}

    \PassOptionsToPackage{numbers, compress}{natbib}

\usepackage[final]{neurips_2022}

\usepackage[utf8]{inputenc} 
\usepackage[T1]{fontenc}    
\usepackage{minitoc}

\usepackage{hyperref}       
\usepackage{url}            
\usepackage{booktabs}       
\usepackage{amsfonts}       
\usepackage{nicefrac}       
\usepackage{microtype}      
\usepackage[dvipsnames]{xcolor}         

\usepackage{amsmath,amsthm,amssymb,amsfonts}
\usepackage[disable]{todonotes}
\usepackage{dsfont}
\usepackage{multirow}
\usepackage{algorithm}
\usepackage{algorithmic}
\usepackage{xspace}
\usepackage{pifont}
\usepackage{enumitem}

\usepackage{array}
\newcommand{\PreserveBackslash}[1]{\let\temp=\\#1\let\\=\temp}
\newcolumntype{C}[1]{>{\PreserveBackslash\centering}p{#1}}
\newcolumntype{R}[1]{>{\PreserveBackslash\raggedleft}p{#1}}
\newcolumntype{L}[1]{>{\PreserveBackslash\raggedright}p{#1}}

\usepackage{mathtools}

\newcommand{\wt}[1]{\widetilde{#1}}

\newcommand{\wb}[1]{\overline{#1}}
\newcommand{\norm}[1]{\left\|#1\right\|}
\newcommand{\Reals}{\mathbb{R}}
\DeclareMathOperator{\EV}{\mathbb{E}}

\newcommand{\indi}[1]{\mathds{1}\left\{#1\right\}}

\newcommand{\transp}{\mathsf{T}}

\newcommand{\argmin}{\operatornamewithlimits{argmin}}
\newcommand{\argmax}{\operatornamewithlimits{argmax}}

\newtheorem{theorem}{Theorem}[section]
\newtheorem{lemma}[theorem]{Lemma}
\newtheorem{corollary}[theorem]{Corollary}
\newtheorem{definition}{Definition}[section]

\newtheorem{assumption}{Assumption}

\newcommand{\cB}{\mathcal{B}}
\newcommand{\cC}{\mathcal{C}}

\newcommand{\cE}{\mathcal{E}}
\newcommand{\cA}{\mathcal{A}}

\newcommand{\cF}{\mathcal{F}}

\newcommand{\cL}{\mathcal{L}}
\newcommand{\cN}{\mathcal{N}}

\newcommand{\cX}{\mathcal{X}}

\newcommand{\cD}{\mathcal{D}}

\newcommand{\cU}{\mathcal{U}}

\newcommand{\bE}{\mathbb{E}}
\newcommand{\bP}{\mathbb{P}}
\newcommand{\bR}{\mathbb{R}}
\newcommand{\bN}{\mathbb{N}}

\newcommand{\algo}{\textsc{BanditSRL}\xspace}
\newcommand{\deepalgo}{\textsc{NN-BanditSRL}\xspace}
\newcommand{\hls}{\textsc{HLS}\xspace}
\newcommand{\leader}{\textsc{Leader}\xspace}
\newcommand{\linucb}{\textsc{LinUCB}\xspace}
\newcommand{\oful}{\textsc{OFUL}\xspace}

\title{Scalable Representation Learning in Linear Contextual Bandits with Constant Regret Guarantees}

\author{%
  Andrea Tirinzoni \\
  META \\
  \texttt{tirinzoni@fb.com} \\
  \And
  Matteo Papini \\
  Universitat Pompeu Fabra\\
  \texttt{matteo.papini@upf.edu}\\
  \And
  Ahmed Touati\\
  META\\
  \texttt{atouati@fb.com}\\
  \AND
  Alessandro Lazaric\\
  META\\
  \texttt{lazaric@fb.com}\\
  \And
  Matteo Pirotta \\
  META \\
  \texttt{pirotta@fb.com} \\
}

\begin{document}



\maketitle

\doparttoc 
\faketableofcontents 

\begin{abstract}
We study the problem of representation learning in stochastic contextual linear bandits. While the primary concern in this domain is usually to find \textit{realizable} representations (i.e., those that allow predicting the reward function at any context-action pair exactly), it has been recently shown that representations with certain spectral properties (called \textit{HLS}) may be more effective for the exploration-exploitation task, enabling \textit{LinUCB} to achieve constant (i.e., horizon-independent) regret. In this paper, we propose \textsc{BanditSRL}, a representation learning algorithm that combines a novel constrained optimization problem to learn a realizable representation with good spectral properties with a generalized likelihood ratio test to exploit the recovered representation and avoid excessive exploration. We prove that \textsc{BanditSRL} can be paired with any no-regret algorithm and achieve constant regret whenever an \textit{HLS} representation is available. Furthermore, \textsc{BanditSRL} can be easily combined with deep neural networks and we show how regularizing towards \textit{HLS} representations is beneficial in standard benchmarks.
\end{abstract}



\section{Introduction}\label{sec:introduction}

The contextual bandit is a general framework to formalize the exploration-exploitation dilemma arising in sequential decision-making problems such as recommendation systems, online advertising, and clinical trials~\citep[e.g.,][]{bouneffouf2019survey}. When solving real-world problems, where contexts and actions are complex and high-dimensional (e.g., users' social graph, items' visual description), it is crucial to provide the bandit algorithm with a suitable representation of the context-action space. While several representation learning algorithms have been proposed in supervised learning and obtained impressing empirical results~\citep[e.g.,][]{oord2018representation,EricssonGLH22}, how to \textit{efficiently} learn representations that are effective for the exploration-exploitation problem is still relatively an open question.

The primary objective in representation learning is to find features that map the context-action space into a lower-dimensional embedding that allows fitting the reward function accurately, i.e., \textit{realizable} representations~\citep[e.g.,][]{AgarwalDKLS12,Agarwal2014taming,RiquelmeTS18,Foster2020beyond,foster2019nested,Lattimore2020good,SimchiLevi2020falcon}. Within the space of realizable representations, bandit algorithms leveraging features of smaller dimension are expected to learn faster and thus have smaller regret. Nonetheless, Papini et al.~\citep{PapiniTRLP21hlscontextual} have recently shown that, even among realizable features, certain representations are naturally better suited to solve the exploration-exploitation problem. In particular, they proved that \linucb{}~\citep{ChuLRS11,Abbasi-YadkoriPS11} can achieve constant regret when provided with a ``good'' representation. Interestingly, this property is not related to ``global'' characteristics of the feature map (e.g., dimension, norms), but rather on a spectral property of the representation (the space associated to optimal actions should cover the context-action space, see \hls{} property in Def.~\ref{ref:hls}). This naturally raises the question whether it is possible to learn such representation at the same time as solving the contextual bandit problem. Papini et al.~\citep{PapiniTRLP21hlscontextual} provided a first positive answer with the \leader algorithm, which is proved to perform as well as the best realizable representation in a given set up to a logarithmic factor in the number of representations. While this allows constant regret when a realizable \hls representation is available, the algorithm suffers from two main limitations: \textbf{1)} it is entangled with \linucb and it can hardly be generalized to other bandit algorithms; \textbf{2)} it learns a different representation for each context-action pair, thus making it hard to extend beyond finite representations to arbitrary functional space (e.g., deep neural networks).

In this paper, we address those limitations through \algo{}, a novel algorithm that decouples representation learning and exploration-exploitation so as to work with any no-regret contextual bandit algorithm and to be easily extended to general representation spaces. \algo{} combines two components: 1) a representation learning mechanism based on a constrained optimization problem that promotes ``good'' representations while preserving realizability; and 2) a generalized likelihood ratio test (GLRT) to avoid over exploration and fully exploit the properties of ``good'' representations. The main contributions of the paper can be summarized as follows:
\begin{enumerate}[leftmargin=20pt]
    \item We show that adding a GLRT on the top of any no-regret algorithm enables it to exploit the properties of a \hls representation and achieve constant regret. This generalizes the constant regret result for \linucb in~\citep{PapiniTRLP21hlscontextual} to any no-regret algorithm.
    \item Similarly, we show that \algo{} can be paired with any no-regret algorithm and perform effective representation selection, including achieving constant regret whenever a \hls representation is available in a given set. This generalizes the result of \leader beyond \linucb. In doing this we also improve the analysis of the misspecified case and prove a tighter bound on the time to converge to realizable representations. Furthermore, numerical simulations in synthetic problems confirm that \algo{} is empirically competitive with \leader.
    \item Finally, in contrast to \leader, \algo{} can be easily scaled to complex problems where representations are encoded through deep neural networks. In particular, we show that the Lagrangian relaxation of the constrained optimization problem for representation learning becomes a regression problem with an auxiliary representation loss promoting \hls-like representations. We test different variants of the resulting \deepalgo{} algorithm showing how the auxiliary representation loss improves performance in a number of dataset-based benchmarks.
\end{enumerate}

\section{Preliminaries}\label{sec:preliminaries}

We consider a stochastic contextual bandit problem with context space $\cX$ and finite action set $\cA$. At each round $t\geq1$, the learner observes a context $x_t$ sampled i.i.d.\ from a distribution $\rho$ over $\cX$, selects an action $a_t \in \cA$, and receives a reward $y_t = \mu(x_t,a_t) + \eta_t$ where $\eta_t$ is a zero-mean noise and $\mu:\mathcal{X}\times\mathcal{A}\rightarrow \mathbb{R}$ is the expected reward. 
The objective of a learner $\mathfrak{A}$ is to minimize its pseudo-regret $R_T := \sum_{t=1}^T \big (\mu^\star(x_t) -\mu(x_t,a_t) \big)$ for any $T \geq 1$, where $\mu^\star(x_t) := \max_{a\in\cA} \mu(x_t,a)$. We assume that for any $x \in \cX$ the optimal action $a^\star_x := \argmax_{a\in\cA} \mu(x,a)$ is unique and we define the gap $\Delta(x,a) := \mu^\star(x) - \mu(x,a)$. We say that $\mathfrak{A}$ is a no-regret algorithm if, for any instance of $\mu$, it achieves sublinear regret, i.e., $R_T = o(T)$. 

We consider the problem of representation learning in given a candidate function space $\Phi \subseteq \big\{ \phi : \cX \times \cA \to \mathbb{R}^{d_{\phi}}\big\}$, where the dimensionality $d_\phi$ may depend on the feature $\phi$. Let $\theta_\phi^\star = \argmin_{\theta \in \bR^{d_\phi}} \mathbb{E}_{x \sim \rho}\big[ \sum_a (\phi(x,a)^\transp \theta - \mu(x,a))^2 \big]$ be the best linear fit of $\mu$ for representation $\phi$. 
We assume that $\Phi$ contains a linearly realizable representation.
\begin{assumption}[Realizability]\label{asm:set.contains.realizable.phi}
    There exists an (unknown) subset $\Phi^\star\subseteq\Phi$ such that, for each $\phi\in\Phi^\star$, $\mu(x,a) = \phi(x,a)^\transp \theta^\star_\phi, \forall x\in\cX,a\in\cA$.
\end{assumption}

\begin{assumption}[Regularity]\label{asm:boundedness}
	Let $\cB_\phi := \{\theta\in\bR^{d_\phi} : \|\theta\|_2 \leq B_\phi\}$ be a ball in $\bR^{d_\phi}$. We assume that, for each $\phi\in\Phi$, $\sup_{x,a}\|\phi(x,a)\|_2 \leq L_\phi$, $\|\theta_\phi^\star\|_2 \leq B_\phi$, $\sup_{x,a}|\phi(x,a)^\transp\theta| \leq 1$ for any $\theta \in \mathcal{B}_\phi$ and $|y_t| \leq 1$ almost surely for all $t$. We assume parameters $L_\phi$ and $B_\phi$ are known. We also assume the minimum gap $\Delta = \inf_{x\in \cX: \rho(x) > 0, a \in \cA, \Delta(x,a)>0} \{\Delta(x,a)\} > 0$ and that 
    $\lambda_{\min} \Big(\frac{1}{|\cA|} \sum_{a} \mathbb{E}_{x \sim \rho} [\phi(x,a)\phi(x,a)^\transp] \Big) >0$ for any $\phi \in \Phi^\star$, i.e, all realizable representations are non-redundant.
\end{assumption}

Under Asm.~\ref{asm:set.contains.realizable.phi}, when $|\Phi|=1$, the problem reduces to a stochastic linear contextual bandit and can be solved using standard algorithms, such as  \linucb/\oful{}~\citep{ChuLRS11,Abbasi-YadkoriPS11}, LinTS~\citep{AbeilleL17}, and $\epsilon$-greedy~\citep{lattimore2020bandit}, which enjoy sublinear regret and, in some cases, logarithmic problem-dependent regret. 
Recently, Papini et al.~\citep{PapiniTRLP21hlscontextual} showed that \linucb only suffers constant regret when a \emph{realizable} representation is \hls, i.e., when the features of optimal actions span the entire $d_\phi$-dimensional space. \hls{} 
\begin{definition}[\hls{} Representation]\label{ref:hls}
    A representation $\phi$ is \hls{} (the acronym refers to the last names of the authors of~\citep{hao2020adaptive}) if
    \begin{equation*}
        \lambda^\star(\phi) := \lambda_{\min}\left( \mathbb{E}_{x \sim \rho} \left[ \phi(x, a^\star_x) \phi(x, a^\star_x)^\transp \right] \right) > 0
    \end{equation*}
    where $\lambda_{\min}(A)$ denotes the minimum eigenvalue of a matrix $A$.
\end{definition}
Papini et al. showed that \hls, together with realizability, is a sufficient and necessary property for achieving constant regret in contextual stochastic linear bandits for non-redundant  representations.

In order to deal with the general case where $\Phi$ may contain non-realizable representations, we rely on the following misspecification assumption from~\citep{PapiniTRLP21hlscontextual}.
\begin{assumption}[Misspecification]\label{asm:icml.misspecification}
    For each $\phi \notin \Phi^\star$, there exists $\epsilon_\phi > 0$ such that
    \begin{align*}
    \min_{\theta \in \mathcal{B}_\phi} \min_{\pi : \cX \to \cA} \mathbb{E}_{x\sim\rho}\left[\left(\phi(x,\pi(x))^\transp \theta - \mu(x,\pi(x))\right)^2 \right] \geq \epsilon_\phi.
    \end{align*}
\end{assumption}
This assumption states that any non-realizable representation has a minimum level of misspecification on average over contexts and for any context-action policy. In the finite-context case, a sufficient condition for Asm.~\ref{asm:icml.misspecification} is that, for each $\phi\notin\Phi^\star$, there exists a context $x\in\cX$ with $\rho(x)>0$ such that $\phi(x,a)^\transp\theta \neq \mu(x,a)$ for all $a\in\cA$ and $\theta\in\cB_\phi$. 

\textbf{Related work.}
Several papers have focused on contextual bandits with an arbitrary function space to estimate the reward function under realizability assumptions~\citep[e.g.,][]{AgarwalDKLS12,Agarwal2014taming,Foster2020beyond}. While these works consider a similar setting to ours, they do not aim to learn ``good'' representations, but rather focus on the exploration-exploitation problem to obtain sublinear regret guarantees. This often corresponds to recovering the maximum likelihood representation, which may not lead to the best regret. 
After the work in~\citep{PapiniTRLP21hlscontextual}, the problem of representation learning with constant regret guarantees has also been studied in reinforcement learning~\citep{PapiniTPRLP21unisoftmdp,zhang2021lowrankunisoft}. As these approaches build on the ideas in~\citep{PapiniTRLP21hlscontextual}, they inherit the same limitations as~\citep{PapiniTRLP21hlscontextual}.

Another related literature is the one of expert learning and model selection in bandits~\citep[e.g.,][]{auer2002nonstochastic,maillardM11,agarwal2017corral,abbasiyadkori2020regret,pacchiano2020stochcorral,lee2020online,CutkoskyDDGPP21}, where the objective is to select the best candidate among a set of base learning algorithms or experts.
While these algorithms are general and can be applied to different settings, including representation learning with a finite set of candidates, they may not be able to effectively leverage the specific structure of the problem. Furthermore, at the best of our knowledge, these algorithms suffers a polynomial dependence in the number of base algorithms ($|\Phi|$ in our setting) and are limited to worst-case regret guarantees. Whether the $\sqrt{T}$ or $\mathrm{poly}(|\Phi|)$ dependency can be improved in general is an open question (see ~\citep{CutkoskyDDGPP21} and ~\citep[][App. A]{PapiniTRLP21hlscontextual}). Finally, \citep{foster2019nested,ghosh2021problem} studied the specific problem of model selection with nested linear representations, where the best representation is the one with the smallest dimension for which the reward is realizable.

Several works have recently focused on theoretical and practical investigation of contextual bandits with neural networks (NNs)~\citep{Zhou2020neural,xu2020neuralcb,Deshmukh2020vision}. While their focus was on leveraging the representation power of NNs to correctly predict the rewards, here we focus on learning representations with good spectral properties through a novel auxiliary loss. A related approach to our is~\citep{Deshmukh2020vision} where the authors leverage self-supervised auxiliary losses for representation learning in image-based bandit problems.

\section{A General Framework for Representation Learning}\label{sec:replearn.algo}

\begin{algorithm}[t]
    \caption{\algo}\label{alg:replearnin.icml.asm}
    \begin{algorithmic}[1]
        \STATE \textbf{Input:} representations $\Phi$, no-regret algorithm $\mathfrak{A}$, confidence $\delta \in (0,1)$, update schedule $\gamma > 1$
        \STATE Initialize $j=0$, $\phi_j, \theta_{\phi_j,0}$ arbitrarily, $V_0(\phi_j) = \lambda I_{d_{\phi_j}}$, $t_j = 1$, let $\delta_j := \delta / (2(j+1)^2)$
        \FOR{$t = 1, \ldots$}
            \STATE Observe context $x_t$
            \IF{$\mathrm{GLR_{t-1}(x_t;\phi_j)} > \beta_{t-1,\delta/|\Phi|}(\phi_{j})$}
            \STATE Play $a_t = \argmax_{a\in\cA} \big\{ \phi_{j}(x_t,a)^\transp \theta_{\phi_j,t-1} \big\}$ and observe reward $y_t$
            \ELSE
            \STATE Play $a_t = \mathfrak{A}_t\big(x_t;\phi_{j}, \delta_j/|\Phi|\big)$, observe reward $y_t$, and feed it into $\mathfrak{A}$ 
            \ENDIF
            \IF{$t = \lceil \gamma t_j \rceil$ \textbf{and} $|\Phi|>1$}
            \STATE Set $j = j +1$ and $t_j = t$
            \STATE Compute $\phi_{j} = \argmin_{\phi\in\Phi_t} \big\{ \cL_t(\phi) \big\}$ and reset $\mathfrak{A}$
            \ENDIF
        \ENDFOR
    \end{algorithmic}
\end{algorithm}

We introduce \algo{} (\emph{Bandit Spectral Representation Learner}), an algorithm for stochastic contextual linear bandit that efficiently decouples representation learning from  exploration-exploitation. As illustrated in Alg.~\ref{alg:replearnin.icml.asm}, \algo{} has access to a fixed-representation contextual bandit algorithm $\mathfrak{A}$, the \textit{base algorithm}, and it is built around two key mechanisms: \ding{182} a constrained optimization problem where the objective is to minimize a representation loss $\cL$ to favor representations with $\hls{}$ properties, whereas the constraint ensures realizability; \ding{183} a generalized likelihood ratio test (GLRT) to ensure that, if a \hls{} representation is learned, the base algorithm $\mathfrak{A}$ does not over-explore and the ``good'' representation is exploited to obtain constant regret. 

\textbf{Mechanism \ding{182} (line 12).} The first challenge when provided with a generic set $\Phi$ is to ensure that the algorithm does not converge to selecting misspecified representations, which may lead to linear regret. This is achieved by introducing a hard constraint in the representation optimization, so that \algo{} only selects representations in the set (see also~\citep[][App. F]{PapiniTRLP21hlscontextual}),
\begin{equation}
    \Phi_t := \left\{ \phi\in\Phi : \min_{\theta\in\cB_\phi}E_{t}(\phi, \theta) \leq \min_{\phi'\in\Phi}\min_{\theta\in\mathcal{B}_{\phi'}} \big\{ E_{t}(\phi',\theta) + \alpha_{t,\delta}(\phi') \big\} \right\}
\end{equation}
where $E_t(\phi,\theta) := \frac{1}{t} \sum_{s=1}^t \left(\phi(x_s,a_s)^T \theta - y_s\right)^2$ is the empirical mean-square error (MSE) of model $(\phi,\theta)$ and $\alpha_{t,\delta}(\phi) := \frac{40}{t}\log \left(\frac{8|\Phi|^2(12L_\phi B_\phi t)^{d_\phi}t^3}{\delta} \right)+ \frac{2}{t}$. 
This condition leverages the existence of a realizable representation in $\Phi_t$ to eliminate representations whose MSE is not compatible with the one of the realizable representation, once accounted for the statistical uncertainty (i.e., $\alpha_{t,\delta}(\phi)$).

Subject to the realizability constraint, the representation loss $\cL_t(\phi)$ favours learning a \hls{} representation (if possible).
As illustrated in Def.~\ref{ref:hls}, a \hls{} representation is such that the expected design matrix associated to the optimal actions has a positive minimum eigenvalue. Unfortunately it is not possible to directly optimize for this condition, since we have access to neither the context distribution $\rho$ nor the optimal action in each context. Nonetheless, we can design a loss that works as a proxy for the \hls{} property whenever $\mathfrak{A}$ is a no-regret algorithm. Let $V_{t}(\phi) = \lambda I_{d_\phi} + \sum_{s=1}^{t} \phi(x_s,a_s)\phi(x_s,a_s)^\transp$ be the empirical design matrix built on the context-actions pairs observed up to time $t$, then we define $\cL_{\mathrm{eig},t}(\phi):= -\lambda_{\min} \big( V_{t}(\phi) -\lambda I_{d_\phi} \big) / L_\phi^2$, where the normalization factor ensures invariance w.r.t.\ the feature norm.
Intuitively, the empirical distribution of contexts $(x_t)_{t\ge 1}$  converges to $\rho$ and the frequency of optimal actions selected by a no-regret algorithm increases over time, thus ensuring that $V_{t}(\phi) / t$ tends to behave as the design matrix under optimal arms $\mathbb{E}_{x\sim \rho} [\phi(x,a^\star_x)\phi(x,a^\star_x)^\transp]$.
As discussed in Sect.~\ref{sec:exp.and.practical.algo} alternative losses can be used to favour learning \hls{} representations.

\textbf{Mechanism \ding{183} (line 5).}
While Papini et al.~\citep{PapiniTRLP21hlscontextual} proved that \linucb is able to exploit HLS representations, other algorithms such as $\epsilon$-greedy may keep forcing exploration and do not fully take advantage of HLS properties, thus failing to achieve constant regret. In order to prevent this, we introduce a \emph{generalized likelihood ratio} test (GLRT). At each round $t$, let $\phi_{t-1}$ be the representation used at time $t$, then \algo{} decides whether to act according to the base algorithm $\mathfrak{A}$ with representation $\phi_{t-1}$ or fully exploit the learned representation and play greedily w.r.t.\ it. Denote by $\theta_{\phi,t-1} = V_{t-1}(\phi)^{-1} \sum_{s=1}^{t-1} \phi(x_s,a_s) y_s$ the regularized least-squares parameter at time $t$ for representation $\phi$ and by $\pi^\star_{t-1}(x;\phi) = \argmax_{a \in\cA} \big\{ \phi(x,a)^\transp \theta_{\phi,t-1} \big\}$ the associated greedy policy. Then, \algo{} selects the greedy action $\pi^\star_{t-1}(x_t;\phi_{t-1})$ when the GLR test is active, otherwise it selects the action proposed by the base algorithm $\mathfrak{A}$. 
Formally, for any $\phi\in\Phi$ and $x\in\cX$, we define the generalized likelihood ratio as
\begin{align}\label{eq:glrt.main.paper}
  \mathrm{GLR}_{t-1}(x;\phi) := \min_{a\neq \pi^\star_{t-1}(x;\phi)} \frac{\big(\phi(x,\pi^\star_{t-1}(x;\phi)) - \phi(x,a)\big)^\transp\theta_{\phi,t-1}}{\|\phi(x, \pi^\star_{t-1}(x;\phi)) - \phi(s,a)\|_{V_{t-1}(\phi)^{-1}}}
\end{align}
and, given $\beta_{t-1,\delta}(\phi)=\sigma\sqrt{2\log(1/\delta)+d_{\phi}\log(1+(t-1)L_{\phi}^2/(\lambda d_{\phi}))} + \sqrt{\lambda}B_{\phi}$, the GLR test is $\mathrm{GLR}_{t-1}(x;\phi)  > \beta_{t-1,\delta/|\Phi|}(\phi)$ \citep{hao2020adaptive,tirinzoni2020asymptotically,degenne2020gamification}.
If this happens at time $t$ and $\phi_{t-1}$ is realizable, then we have enough confidence to conclude that the greedy action is optimal, i.e., $\pi^\star_{t-1}(x_t;\phi_{t-1})=a^\star_{x_t}$.
An important aspect of this test is that it is run on the current context $x_t$ and it does not require evaluating global properties of the representation. While at any time $t$ it is possible that a non-HLS non-realizable representation may pass the test, the GLRT is sound as \textbf{1)} exploration through $\mathfrak{A}$ and the representation learning mechanism work in synergy to guarantee that \textit{eventually} a realizable representation is always provided to the GLRT; \textbf{2)} only \hls{} representations are guaranteed to consistently trigger the test at any context $x$.

In practice, \algo{} does not update the representation at each step but in phases. 
This is necessary to avoid too frequent representation changes  and control the regret, but also to make the algorithm more computationally efficient and practical.
Indeed, updating the representation may be computationally expensive in practice (e.g., retraining a NN) and a phased scheme with $\gamma$ parameter reduces the number of representation learning steps to $J \approx \lceil \log_{\gamma}(T) \rceil$. 
The algorithm $\mathfrak{A}$ is reset at the beginning of a phase $j$ when the representation is selected and it is run on the samples collected during the current phase when the base algorithm is selected. If $\mathfrak{A}$ is able to leverage off-policy data, at the beginning of a phase $j$, we can warm-start it by providing $\phi_j$ and all the past data $(x_s,a_s,y_s)_{s \leq t_j}$. While the reset is necessary for dealing with \textit{any} no-regret algorithm, it can be removed for algorithms such as \linucb and $\epsilon$-greedy without affecting the theoretical guarantees. 

\textbf{Comparison to \leader.} We first recall the basic structure of \leader. Denote by $\mathrm{UCB}_t(x, a, \phi)$ the upper-confidence bound computed by \linucb{} for the context-action pair $(x,a)$ and representation $\phi$ after $t$ steps. Then \leader{} selects the action 
%
	$a_t \in \argmax_{a\in\cA} \min_{\phi \in \Phi_t} \mathrm{UCB}_t(x_t, a, \phi)$.
%
Unlike the constrained optimization problem in \algo{}, this mechanism couples representation learning and exploration-exploitation and it requires optimizing a representation for the current $x_t$ and for each action $a$. Indeed, \leader does not output a single representation and possibly chooses different representations for each context-action pair. 
While this enables \leader to mix representations and achieve constant regret in some cases even when $\Phi$ does not include any \hls{} representation, it leads to two major drawbacks: \textbf{1)} the representation selection is directly entangled with the \linucb exploration-exploitation strategy, \textbf{2)} it is impractical in problems where $\Phi$ is an infinite functional space (e.g., a deep neural network). The mechanisms \ding{182} and \ding{183} successfully address these limitations and enable \algo{} to be paired with any no-regret algorithm and to be scaled to any representation class as illustrated in the next section.

\subsection{Extension to Neural Networks}
We now consider a representation space $\Phi$ defined by the last layer of a NN. We denote by $\phi : \cX \times \cA \to \mathbb{R}^d$ the last layer and by $f(x,a) = \phi(x,a)^\transp \theta$ the full NN, where $\theta$ are the last-layer weights. 
We show how \algo{} can be easily adapted to work with deep neural networks (NN). 

\textit{First}, the GLRT requires only to have access to the current context $x_t$ and representation $\phi_{j}$, i.e., the features defined by the last layer of the current network, and its cost is linear in the number of actions. \textit{Second}, the phased scheme allows lazy updates, where  we retrain the network only $\log_{\gamma}(T)$ times. \textit{Third}, we can run any bandit algorithm with a representation provided by the NN, including \linucb, LinTS, and $\epsilon$-greedy. \textit{Fourth}, the representation learning step can be adapted to allow efficient optimization of a NN. We consider a regularized problem obtained through an approximation of the constrained problem:
\begin{align}\label{eq:reg.problem}
& \argmin_{\phi} \left\{ \cL_t(\phi) - c_{\mathrm{reg}} \left( \min_{\phi',\theta'} \big\{ E_{t}(\phi',\theta') + \alpha_{t,\delta}(\phi') \big\}- \min_{\theta}E_{t}(\phi, \theta)  \right) \right\}\nonumber\\
    & = \argmin_{\phi} \min_{\theta} \left\{ \cL_t(\phi) + c_{\mathrm{reg}}\, E_{t}(\phi, \theta) \right\}.
\end{align}
where $c_{\mathrm{reg}} \geq 0$ is a tunable parameter. The fact we consider $c_{\mathrm{reg}}$ constant allows us to ignore terms that do not depend on either $\phi$ or $\theta$.
This leads to a convenient regularized loss that aims to minimize the MSE (second term) while enforcing some spectral property on the last layer of the NN (first term). In practice, we can optimize this loss by stochastic gradient descent over a \textit{replay buffer} containing the samples observed over time. The resulting algorithm, called \deepalgo, is a direct and elegant generalization of the theoretically-grounded algorithm.

While in theory we can optimize the regularized  loss~\eqref{eq:reg.problem} with all the samples, in practice it is important to better control the sample distribution. As the algorithm progresses, we expect the replay buffer to contain an increasing number of samples obtained by optimal actions, which may lead the representation to solely fit optimal actions while increasing  misspecification on suboptimal actions. This may compromise the behavior of the algorithm and ultimately lead to high regret. This is an instance of \emph{catastrophic forgetting} induced by a biased/shifting sample distribution~\citep[e.g.,][]{Goodfellow2013forgetting}.
To prevent this phenomenon, we store two replay buffers: \textit{i)} an explorative buffer $\cD_{\mathfrak{A},t}$ with samples obtained when $\mathfrak{A}$ was selected; \textit{ii)} an exploitative buffer $\cD_{\mathrm{glrt},t}$ with samples obtained when GLRT triggered and greedy actions were selected. 
The explorative buffer $\cD_{\mathfrak{A},t}$ is used to compute the MSE $E_t(\phi,\theta)$. While this reduces the number of samples, it improves the robustness of the algorithm by promoting realizability. On the other hand, we use all the samples $\cD_t = \cD_{\mathfrak{A},t} \cup \cD_{\mathrm{glrt},t}$ for the representation loss $\mathcal{L}(\phi)$. This is coherent with the intuition that mechanism \ding{182} works when the design matrix $V_t$ drifts towards the design matrix of optimal actions, which is at the core of the \hls property. 
Refer to App.~\ref{app:algo.variations} for a more detailed description of \deepalgo.

\section{Theoretical Guarantees}\label{sec:theory.res}

In this section, we provide a complete characterization of the theoretical guarantees of \algo when $\Phi$ is a finite set of representations, i.e., $|\Phi|<\infty$. We consider the update scheme with $\gamma=2$.
\subsection{Constant Regret Bound for \hls Representations}
We first study the case where a realizable \hls{} representation is available. For the characterization of the behavior of the algorithm, we need to introduce the following times:
\begin{itemize}[leftmargin=20pt]
  \item $\tau_{\mathrm{elim}}$: an upper-bound to the time at which all non-realizable representations are eliminated, i.e., for all $t \geq \tau_{\mathrm{elim}}$, $\Phi_t = \Phi^\star$;
  \item $\tau_{\hls}$: an upper-bound to the time (if it exists) after which the \hls{} representation is selected, i.e., $\phi_t = \phi^\star$ for all $t\geq \tau_{\hls}$, where $\phi^\star \in \Phi^\star$ is the unique \hls{} realizable representation; 
  \item $\tau_{\mathrm{glrt}}$: an upper-bound to the time (if it exists) such that the GLR test triggers for the \hls{} representation $\phi^\star$ for all $t \geq \tau_{\mathrm{glrt}}$.
\end{itemize}

We begin by deriving a constant problem-dependent regret bound for \algo{} with \hls{} representations. The proof and explicit values of the constants are reported in  App.~\ref{app:analysis}.\footnote{While Thm.~\ref{th:icmlams.regret.lambda_min.hls} provides high-probability guarantees, we can easily derive a constant expected-regret bound by running \algo{} with a decreasing schedule for $\delta$ and with a slightly different proof.}
\begin{theorem}\label{th:icmlams.regret.lambda_min.hls}
  Let $\mathfrak{A}$ be any no-regret algorithm for stochastic contextual linear bandits, $\Phi$ satisfy Asm.~\ref{asm:set.contains.realizable.phi}-~\ref{asm:icml.misspecification}, $|\Phi| < \infty$, $\gamma=2$, and $\cL_t(\phi)= \cL_{\mathrm{eig},t}(\phi) :=-\lambda_{\min}(V_t(\phi) - \lambda I_{d_{\phi}})/L_{\phi}^2$. Moreover, let $\Phi^\star$ contains a unique \hls{} representation $\phi^\star$. Then, for any $\delta \in (0,1)$ and $T\in\mathbb{N}$, the regret of \algo{} is bounded, with probability at least  $1-4\delta$, as\footnote{We denote by $a \wedge b$ (resp. $a\vee b$) the minimum (resp. the maximum) between $a$ and $b$.}
  \begin{align*}
    R_T \leq 2\tau_{\mathrm{elim}} + \max_{\phi\in\Phi^\star} \wb{R}_{\mathfrak{A}}(
      (\tau_{\mathrm{opt}} - \tau_{\mathrm{elim}}) \wedge T
      , \phi, \delta_{\log_2(\tau_{\mathrm{opt}} \wedge T)}/|\Phi|) \log_{2}(\tau_{\mathrm{opt}} \wedge T),
  \end{align*}
  where $\delta_j := \delta / (2(j+1)^2)$ and
  \begin{align}\label{eq:tau.opt.main}
    \tau_{\mathrm{opt}} = \tau_{\mathrm{glrt}} \vee \tau_{\hls} \vee \tau_{\mathrm{elim}} \lesssim \tau_{\mathrm{alg}} + \frac{L_{\phi^\star}^2\log(|\Phi|/\delta)}{\lambda^\star(\phi^\star)} \left( \frac{L_{\phi^\star}^2}{\lambda^\star(\phi^\star)} + \frac{d_{\phi^\star}}{\Delta^2} + \frac{d}{(\min_{\phi\notin\Phi^\star}\epsilon_\phi)\Delta} \right),
  \end{align}
  with $\tau_{\mathrm{alg}}$ a \emph{finite} (independent from the horizon $T$) constant depending on algorithm $\mathfrak{A}$ (see Tab.~\ref{tab:meta.regret.bounds}) and $\wb{R}_{\mathfrak{A}}(\tau, \phi, \delta)$ an anytime bound (non-decreasing in $\tau$ and $1/\delta$) on the regret accumulated over $\tau$ steps by $\mathfrak{A}$ using representation $\phi$ and confidence level $\delta$.
\end{theorem}

The key finding of the previous result is that \algo{} achieves constant regret whenever a realizable \hls{} representation is available in the set $\Phi$, which may contain non-realizable as well as realizable non-\hls{} representations. The regret bound above also illustrates the ``dynamics'' of the algorithm and three main regimes. In the early stages, non-realizable representations may be included in $\Phi_t$, which may lead to suffering linear regret until time $\tau_{\mathrm{elim}}$ when the constraint in the representation learning step filters out all non-realizable representations (first term in the regret bound). At this point, \algo{} leverages the loss $\mathcal{L}$ to favor \hls representations and the base algorithm $\mathfrak{A}$ to perform effective exploration-exploitation. This leads to the second term in the bound, which corresponds to an upper-bound to the sum of the regrets of $\mathfrak{A}$ in each phase in between $\tau_{\mathrm{elim}}$ and $\tau_{\mathrm{glrt}} \vee \tau_{\hls}$, which is roughly $\sum_{j_{\tau_{\mathrm{elim}}} < j < j_{\tau_\mathrm{opt}}} \wb{R}_{\mathfrak{A}}(t_{j+1}-t_j, \phi_j) \leq \max_{\phi\in\Phi^\star}\wb{R}_{\mathfrak{A}}(\tau_{\mathrm{opt}} - \tau_{\mathrm{elim}}, \phi) \log_2(\tau_{\mathrm{opt}})$.
In this second regime, in some phases the algorithm may still select non-\hls{} representations, which leads to a worst-case bound over all realizable representations in $\Phi^\star$. Finally, after $\tau_{\mathrm{glrt}} \vee \tau_{\hls}$ the GLRT consistently triggers over time. During this last regime, \algo{} has reached enough accuracy and confidence so that the greedy policy of the \hls representation is indeed optimal and no additional regret is incurred.

We notice that the only dependency on the number of representations $|\Phi|$ in Thm.~\ref{th:icmlams.regret.lambda_min.hls} is due to the rescaling of the confidence level $\delta \mapsto \delta/|\Phi|$. Since standard algorithms have a logarithmic dependence in $1/\delta$, this only leads to a logarithmic dependency in $|\Phi|$.
On the other hand, due to the resets, \algo{} has an extra logarithmic factor in the effective regret horizon $\tau_{\mathrm{opt}}$. 

\textbf{Single \hls{} representation.} A noteworthy consequence of Thm.~\ref{th:icmlams.regret.lambda_min.hls} is that any no-regret algorithm equipped with GLRT achieves constant regret when provided with a realizable \hls{} representation.
\begin{corollary}\label{cor:single-repr}
  Let $\Phi = \Phi^\star = \{\phi^\star\}$ and $\phi^\star$ is \hls. Then, $\tau_{\mathrm{elim}} = \tau_{\hls} =0$ and, with probability at least $1-4\delta$, \algo{} suffers constant regret:
  %
    $R_T \leq \wb{R}_{\mathfrak{A}}(\tau_{\mathrm{glrt}} \wedge T, \phi^\star, \delta)$.
  %
\end{corollary}
This corollary also illustrates that the performance of $\mathfrak{A}$ is not affected when $\phi^\star$ is non-\hls (i.e., $\tau_{\mathrm{glrt}} = \infty$), as \algo{} achieves the same regret of the base algorithm. Note that there is no additional logarithmic factor in this case since we do not need any reset for representation learning.

\subsection{Additional Results}

\textbf{No \hls representation.}
A consequence of Thm.~\ref{th:icmlams.regret.lambda_min.hls} is that when $|\Phi|>1$ but no realizable \hls{} exists ($\tau_{\mathrm{glrt}}=\infty$), \algo{} still enjoys a sublinear regret.
\begin{corollary}[Regret bound without \hls{} representation]\label{th:icmlams.regret.lambda_min.nohls}
  Consider the same setting in Thm.~\ref{th:icmlams.regret.lambda_min.hls} and assume that $\Phi^\star$ does not contain any $\hls{}$ representation. Then, for any $\delta \in (0,1)$ and $T\in\mathbb{N}$, the regret of \algo{} is bounded, with probability at least $1-4\delta$, as follows:
  \begin{align*}
    R_T \leq 2\tau_{\mathrm{elim}} + \max_{\phi\in\Phi^\star} \wb{R}_{\mathfrak{A}}(T, \phi,  \delta_{\log_2(T)}/|\Phi|) \log_{2}(T).
  \end{align*}
\end{corollary}
This shows that the regret of \algo{} is of the same order as the base no-regret algorithm $\mathfrak{A}$ when running with the worst realizable representation.
While such worst-case dependency is undesirable, it is common to many representation learning algorithms, both in bandits and reinforcement learning~\citep[e.g.][]{AgarwalDKLS12,zhang2022repblockmdp}.\footnote{Notice that the worst-representation dependency is often hidden in the definition of $\Phi$, which is assumed to contain features with fixed dimension and bounded norm, i.e., $\Phi = \{\phi:\cX \times \cA \to \mathbb{R}^d, \sup_{x,a}\|\phi(x,a)\|_2 \leq L\}$. As $d$ and $B$ are often the only representation-dependent terms in the regret bound $\wb{R}_{\mathfrak{A}}$, no worst-representation dependency is reported.} In App.~\ref{app:algo.variations}, we show that an alternative representation loss could address this problem and lead to a bound scaling with the regret of the \textit{best} realizable representation ($R_T \leq 2\tau_{\mathrm{elim}} + \min_{\phi\in\Phi^\star} \wb{R}_{\mathfrak{A}}(T, \phi, \delta/|\Phi|) \log_{2}(T)$), while preserving the guarantees for the HLS case. Since the representation loss requires an upper-bound on the number of suboptimal actions and a carefully tuned schedule for guessing the gap $\Delta$, it is less practical than the smallest eigenvalue, which we use as the basis for our practical version of \algo{}.

\begin{table}
  \centering \small
  \begin{tabular}{ccc}
    \hline
    Algorithm & $\wb{R}_{\mathfrak{A}}(T,\phi, \delta/|\Phi|)$ & $\tau_{\mathrm{alg}}$  \\
    \hline
    \linucb & $d_\phi^2\log(|\Phi|T/\delta)^2/\Delta$ & $\frac{L_{\phi^\star}^2 d^2 \log(|\Phi|/\delta)^2}{\lambda^\star(\phi^\star)\Delta^2}$ \\
    $\epsilon$-greedy with $\epsilon_t = t^{-1/3}$ & $\sqrt{d_\phi |\cA|} \log(|\Phi|/\delta) T^{2/3}$ & $\frac{L_{\phi^\star}^6 (d|\cA|)^{3/2} L^3 \log(|\Phi|/\delta)^3}{\lambda^\star(\phi^\star)^3\Delta^3}$\\
    \hline
  \end{tabular}
  \caption{\small Specific regret bounds when using \linucb or $\epsilon$-greedy as base algorithms. We omit numerical constants and logarithmic factors.}
    \label{tab:meta.regret.bounds}
\end{table}

\textbf{Algorithm-dependent instances and comparison to \leader.}
Table~\ref{tab:meta.regret.bounds} reports the regret bound of \algo{} for different base algorithms. These results make explicit the dependence in the number of representations $|\Phi|$ and show that the cost of representation learning is only logarithmic. 
In the specific case of \linucb for \hls representations, we highlight that the upper-bound to the time $\tau_{\mathrm{opt}}$ in Thm.~\ref{th:icmlams.regret.lambda_min.hls} improves over the result of \leader.  While \leader has no explicit concept of $\tau_{\mathrm{alg}}$, a term with the same dependence of $\tau_{\mathrm{alg}}$ in Tab.~\ref{tab:meta.regret.bounds} appears also in the \leader analysis. This term encodes an upper bound to the pulls of suboptimal actions and depends on the \linucb strategy. As a result, the first three terms in Eq.~\ref{eq:tau.opt.main} are equivalent to the ones of \leader.
The improvement comes from the last term ($\tau_{\mathrm{elim}}$), where, thanks to a refined analysis of the elimination condition, we are able to improve the dependence on the inverse minimum misspecification ($1/\min_{\phi\notin\Phi^\star} \epsilon_{\phi}$) from quadratic to linear (see App.~\ref{app:analysis} for a detailed comparison).
On the other hand, \algo{} suffers from the worst regret among realizable representations, whereas \leader scales with the \textit{best} representation. As discussed above, this mismatch can be mitigated by using by a different choice of representation loss. In the case of $\epsilon$-greedy, the $T^{2/3}$ regret upper-bound induces a worse $\tau_{\mathrm{alg}}$ due to a larger number of suboptimal pulls. This in turns reflects into a higher regret to the constant regime.
 Finally, \leader is still guaranteed to achieve constant regret by selecting different representations at different context-action pairs whenever non-\hls representations satisfy a certain mixing condition~\citep[cf.][Sec. 5.2]{PapiniTRLP21hlscontextual}. This result is not possible  with \algo{}, where one representation is selected in each phase. At the same time, it is the single-representation structure of \algo{} that allows us to accommodate different base algorithms and scale it to any representation space.

\section{Experiments}\label{sec:exp.and.practical.algo}

We provide an empirical validation of \algo{} both in synthetic contextual linear bandit problems and in non-linear contextual problems~\citep[see e.g.,][]{RiquelmeTS18,Zhou2020neural}.

\textbf{Linear Benchmarks.} 
We first evaluate \algo{} on synthetic linear problems to empirically validate our theoretical findings. In particular, we test \algo{} with different base algorithms and representation learning losses and we compare it with \leader.\footnote{We do not report the performance of model selection algorithms. An extensive analysis can be found in~\citep{PapiniTRLP21hlscontextual}, where the author showed that \leader was outperforming all the baselines.}
We consider the ``varying dimension'' problem introduced in~\citep{PapiniTRLP21hlscontextual} which consists of six realizable representations with dimension from $2$ to $6$. Of the two representations of dimension $d = 6$, one is \hls. In addition seven misspecified representations are available. Details are provided in App.~\ref{app:experiments}. We consider \linucb and $\epsilon$-greedy as base algorithms and we use the theoretical parameters, but we perform warm start using all the past data when a new representation is selected. Similarly, for \algo{} we use the theoretical parameters ($\gamma=2$) and $\cL_t(\phi) := \cL_{\mathrm{eig},t}(\phi)$.
Fig.~\ref{fig:vardim} shows that, as expected, \algo{} with both base algorithms is able to achieve constant regret when a \hls{} representation exists. As expected from the theoretical analysis, $\epsilon$-greedy leads to a higher regret than \linucb.
Furthermore, empirically \algo{} with \linucb{} obtains a performance that is comparable with the one of \leader{} both with and without realizable \hls{} representation. Note that when no \hls{} exists, the regret of \algo{} with $\epsilon$-greedy is $T^{2/3}$, while \linucb-based algorithms are able to achieve $\log(T)$ regret. When $\Phi$ contains misspecified representations (Fig.~\ref{fig:vardim}(center-left)), we can observe that in the first regime $[1,\tau_{\mathrm{elim}}]$ the algorithm suffers linear regret, after that we have the regime of the base algorithm ($[\tau_{\mathrm{elim}},\tau_{\mathrm{glrt}}\vee \tau_{\mathrm{\hls}}]$) up to the point where the GLRT leads to select only optimal actions.

\textit{Weak HLS.} Papini et al.~\citep{PapiniTRLP21hlscontextual} showed that when realizable representations are redundant (i.e., $\lambda^\star(\phi^\star) = 0$), it is still possible to achieve constant regret if the representation is ``weakly''-\hls, i.e., the features of the optimal actions span the features $\phi(x,a)$ associated to any context-action pair, but not necessarily $\mathbb{R}^{d_\phi}$. To test this case, we pad a 5-dimensional vector of ones to all the features of the six realizable representations in the previous experiment. To deal with the weak-\hls{} condition, we introduce the alternative representation loss $\cL_{\mathrm{weak},t}(\phi) = -\min_{s\leq t} \big\{\phi(x_s,a_s)^\transp (V_t(\phi) - \lambda I_{d_{\phi}}) \phi(x_s,a_s) / L_{\phi}^2 \big\}$. Since, $V_t(\phi) - \lambda I_{d_{\phi}}$ tends to behave as $\mathbb{E}_{x}[\phi^\star(x)\phi^\star(x)^\transp]$, this loss encourages representations where all the observed features are spanned by the optimal arms, thus promoting weak-\hls{} representations (see App.~\ref{app:algo.variations} for more details). As expected, Fig.~\ref{fig:vardim}(right) shows that the min-eigenvalue loss $\cL_{\mathrm{eig},t}$ fails in identifying the correct representation in this domain. On the other hand, \algo{} with the novel loss is able to achieve constant regret and converge to constant regret (we cut the figure for readability), and behaves as \leader{} when using \linucb. 

\begin{figure}[h]
    \hspace{-0.3in}\includegraphics[width=1.03\textwidth]{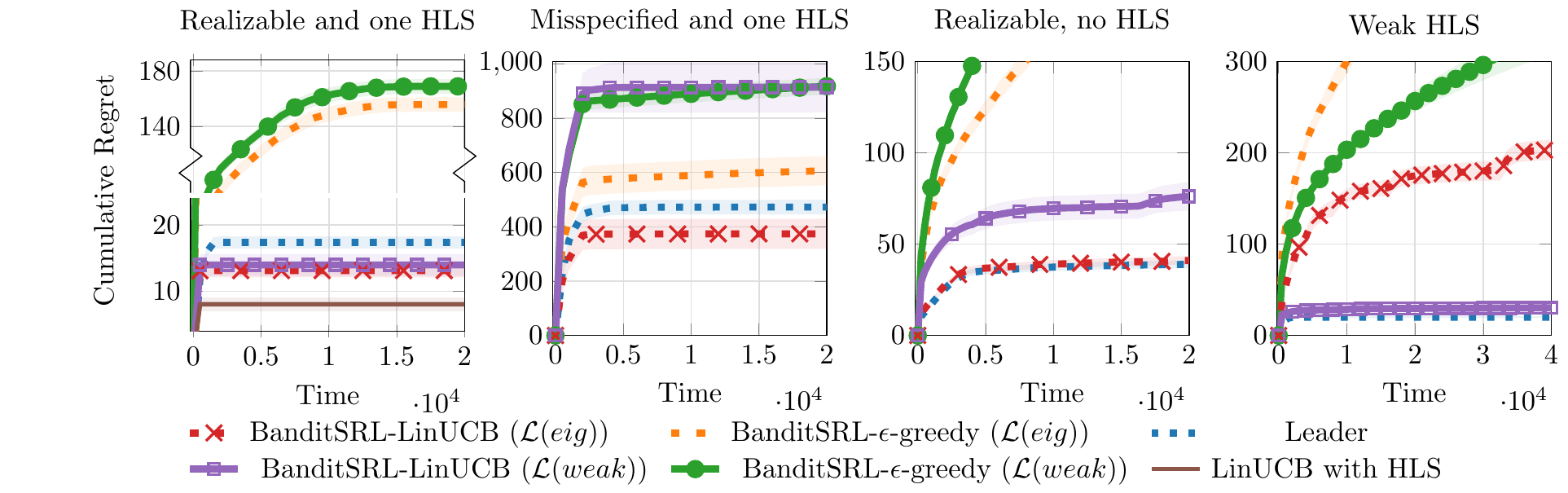}
    \caption{\small
                Varying dimension experiment with all realizable representations (left), misspecified representations (center-left), realizable non-\hls{} representations (center-right) and weak-\hls{} (right). Experiments are averaged over $40$ repetitions.
            }
    \label{fig:vardim}
\end{figure}

\begin{figure}[h]
    \includegraphics[width=1\textwidth]{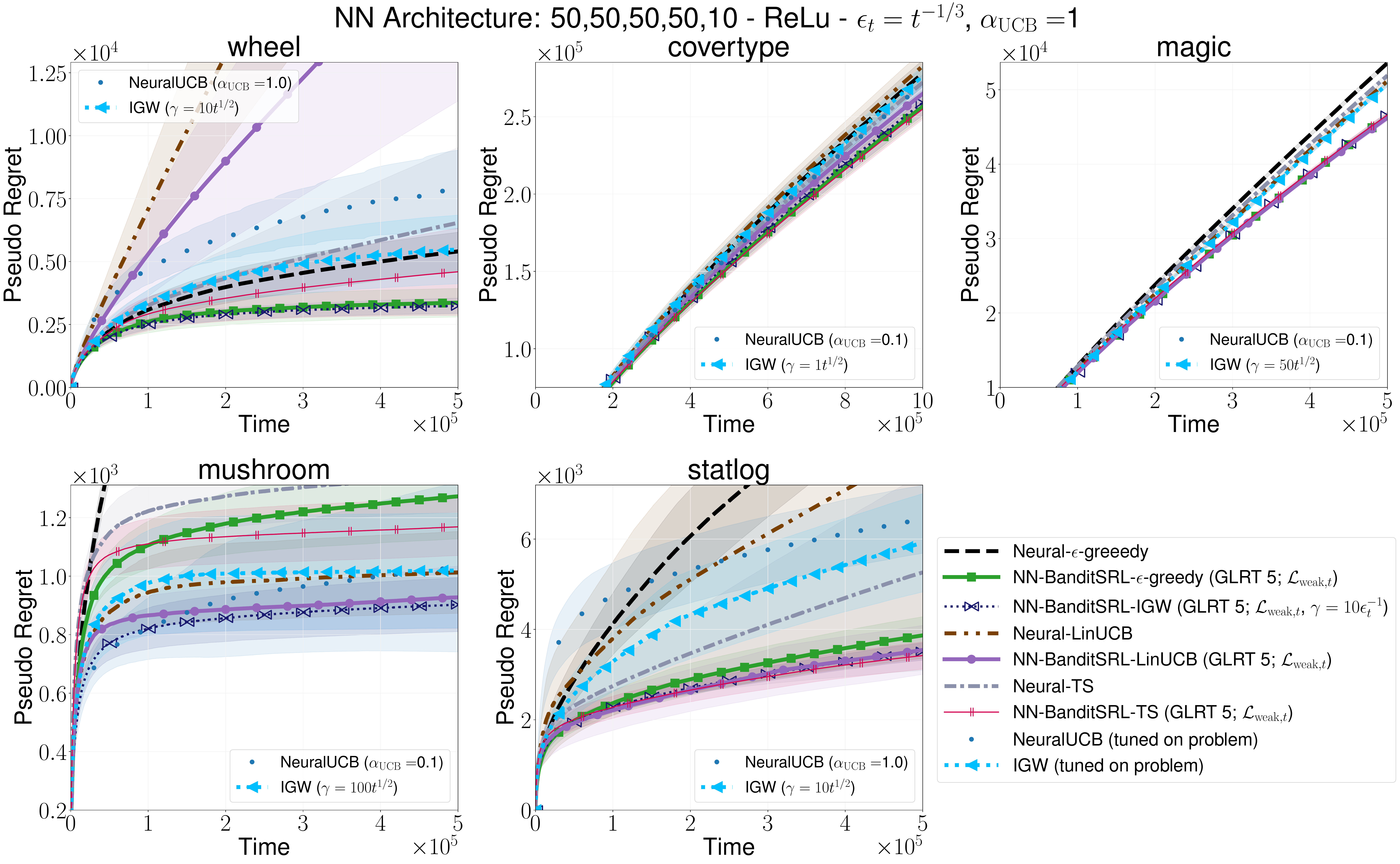}
    \caption{ Average cumulative regret (over $20$ runs) in non-linear domains.}
    \label{fig:dataset}
\end{figure}

\textbf{Non-Linear Benchmarks.}
We study the performance of \deepalgo{} in classical benchmarks where non-linear representations are required. 
We only consider the weak-\hls{} loss $\cL_{\mathrm{weak},t}(\phi)$ as it is more general than full \hls{}. As base algorithms we consider $\epsilon$-greedy and inverse gap weighting (IGW) with $\epsilon_t = t^{-1/3}$, and \linucb and \textsc{LinTS} with theoretical parameters. These algorithms are run on the representation $\phi_j$ provided by the NN at each phase $j$. 
We compare \deepalgo{} against the base algorithms using the maximum-likelihood representation (i.e., Neural-($\epsilon$-greedy, \textsc{LinTS})~\citep{RiquelmeTS18} and Neural-\linucb~\citep{xu2020neuralcb}), supervised learning with the IGW strategy~\citep[e.g.,][]{Foster2020beyond,SimchiLevi2020falcon} and NeuralUCB~\citep{Zhou2020neural}\footnote{For ease of comparison, all the algorithms use the same phased schema for fitting the reward and recomputing the parameters. NeuralUCB uses a diagonal approximation of the design matrix.}
See App.~\ref{app:algo.variations}-\ref{app:experiments} for details.

In all the problems\footnote{The dataset-based problems --statlog, magic, covertype, mushroom~\citep{Blackard1998cover,Bock2004telescope,schlimmer1987concept,Dua:2019}-- are obtained from the standard multiclass-to-bandit conversion~\citep{RiquelmeTS18,Zhou2020neural}. See appendix~\ref{app:experiments} for details.} the reward function is highly non-linear w.r.t.\ contexts and actions and we use a network composed by layers of dimension $[50,50,50,50,10]$ and ReLu activation to learn the representation (i.e., $d=10$).
Fig.~\ref{fig:dataset} shows that all the base algorithms ($\epsilon$-\textsc{greedy}, \textsc{IGW}, \linucb, \textsc{LinTS}) achieve better performance through representation learning, outperforming the base algorithms.
This provides evidence that \deepalgo{} is effective even beyond the theoretical scenario.

For the baseline algorithms (\textsc{NeuralUCB}, \textsc{IGW}) we report the regret of the best configuration on each individual dataset, while for \deepalgo{} we fix the parameters across datasets (i.e., $\alpha_{\mathrm{GLRT}}=5$). While this comparison clearly favours the baselines, it also shows that \deepalgo{} is a robust algorithm that behaves better or on par with the state-of-the-art algorithms. In particular, \deepalgo{} uses theoretical parameters while the baselines use tuned configurations. Optimizing the parameters of \deepalgo{} is outside the scope of these experiments.

\section{Conclusion}
We proposed a novel algorithm, \algo{}, for representation selection in stochastic contextual linear bandits. \algo{} combines a mechanism for representation learning that aims to recover representations with good spectral properties, with a generalized likelihood ratio test to exploit the recovered representation. We proved that, thanks to these mechanisms, \algo{} is not only able to achieve sublinear regret with any no-regret algorithm $\mathfrak{A}$ but, when a \hls{} representation exists, it is able to achieve constant regret. We demonstrated that \algo{} can be implemented using NNs and showed its effectiveness in standard benchmarks.

A direction for future investigation is to extend the approach to a weaker misspecification assumption than  Asm.~\ref{asm:icml.misspecification}. Another direction is to leverage the technical and algorithmic tools introduced in this paper for representation learning in reinforcement learning, e.g., in low-rank problems~\citep[e.g.][]{Agarwal2020flambe}.

\begin{ack}
M. Papini was supported by the European Research Council (ERC) under the European Union’s Horizon 2020 research and innovation programme (Grant agreement No.~950180).	
\end{ack}

\bibliographystyle{unsrt}
\bibliography{biblio}

\section*{Checklist}


\begin{enumerate}

\item For all authors...
\begin{enumerate}
  \item Do the main claims made in the abstract and introduction accurately reflect the paper's contributions and scope?
    \answerYes{}
  \item Did you describe the limitations of your work?
    \answerYes{}
  \item Did you discuss any potential negative societal impacts of your work?
    \answerNA{}
  \item Have you read the ethics review guidelines and ensured that your paper conforms to them?
    \answerYes{}
\end{enumerate}

\item If you are including theoretical results...
\begin{enumerate}
  \item Did you state the full set of assumptions of all theoretical results?
    \answerYes{}
        \item Did you include complete proofs of all theoretical results?
    \answerYes{}
\end{enumerate}

\item If you ran experiments...
\begin{enumerate}
  \item Did you include the code, data, and instructions needed to reproduce the main experimental results (either in the supplemental material or as a URL)?
    \answerNo{}
  \item Did you specify all the training details (e.g., data splits, hyperparameters, how they were chosen)?
    \answerYes{}
        \item Did you report error bars (e.g., with respect to the random seed after running experiments multiple times)?
    \answerYes{}
        \item Did you include the total amount of compute and the type of resources used (e.g., type of GPUs, internal cluster, or cloud provider)?
    \answerYes{}
\end{enumerate}

\item If you are using existing assets (e.g., code, data, models) or curating/releasing new assets...
\begin{enumerate}
  \item If your work uses existing assets, did you cite the creators?
    \answerYes{}
  \item Did you mention the license of the assets?
    \answerNo{}
  \item Did you include any new assets either in the supplemental material or as a URL?
    \answerNo{}
  \item Did you discuss whether and how consent was obtained from people whose data you're using/curating?
    \answerNA{}
  \item Did you discuss whether the data you are using/curating contains personally identifiable information or offensive content?
    \answerNA{}
\end{enumerate}

\item If you used crowdsourcing or conducted research with human subjects...
\begin{enumerate}
  \item Did you include the full text of instructions given to participants and screenshots, if applicable?
    \answerNA{}
  \item Did you describe any potential participant risks, with links to Institutional Review Board (IRB) approvals, if applicable?
    \answerNA{}
  \item Did you include the estimated hourly wage paid to participants and the total amount spent on participant compensation?
    \answerNA{}
\end{enumerate}

\end{enumerate}


\newpage
\appendix

\part{Appendix}


\parttoc
\newpage



\section{Notation}

\begin{table*}[h]
  \centering
  \begin{small}
  \begin{tabular}{@{}ll@{}} 
  \toprule
  Symbol & Meaning \\
  \cmidrule{1-2}
  $\cX$ & Set of contexts\\
  $\cA$ & Finite set of arms\\
  $\rho$ & Context distribution\\
  $\mu : \cX \times \cA \rightarrow \bR$ & Mean-reward function\\
  $\Phi$ & Set of representations\\
  $\Phi^\star$ & Subset of realizable representations\\
  $\pi: \cX \rightarrow \cA$ & A policy\\
  $\cF_t$ & $\sigma$-algebra generated by $(x_1,a_1,y_1,\dots, x_{t},a_{t},y_{t})$ \\
  $\mathfrak{A}_t : \cX \rightarrow \cA$ & Bandit algorithm (measurable mappings w.r.t. $\cF_{t-1}$) \\
  $ V_t(\phi) := \sum_{k=1}^{t} \phi(x_k,a_k) \phi(x_k,a_k)^\transp + \lambda I_{d_{\phi}}$ & Design matrix for representation $\phi$\\
  $\theta_{\phi,t} = V_t(\phi)^{-1} \sum_{k=1}^{t} \phi(x_k,a_k) r_k$ & Regularized least-square estimate for representation $\phi$\\
  $\pi^\star_t(x;\phi) := \argmax_{a\in\cA} \phi(x,a)^\transp\theta_{\phi,t}$ & Empirical optimal arm for context $x$ and representation $\phi$\\
  $\Delta(x,a) = \max_{a'\in\cA}\mu(x,a') - \mu(x,a)$ & Sub-optimality gap of arm $a$ in context $x$\\
  $a^\star_x$ & Optimal arm for context $x$\\
  $\pi^\star(x) = \argmax_{a\in\cA}\mu(x,a)$ & Optimal policy\\
  $\lambda^\star(\phi) := \EV_{x \sim \rho}[\phi(x,\pi^\star(x))\phi(x,\pi^\star(x))^\transp]$ & Minimum eigenvalue on optimal arms\\
  $E_t(\phi,\theta) := \frac{1}{t} \sum_{k=1}^t \left(\phi(x_k,a_k)^\transp \theta - y_k\right)^2$ & Mean square error of model $(\phi,\theta)$ at time $t$\\
  $\mathbb{E}_t$ and $\mathbb{V}_t$ & Expectation and variance conditioned on $\cF_{t-1}$\\
  $P_t(\phi,\theta) := \sum_{k=1}^t \mathbb{E}_k\left[\left(\phi(x_k,a_k)^\transp \theta - \mu(x_k,a_k)\right)^2\right]$ & Sum of mean prediction errors of model $(\phi,\theta)$\\
  $\alpha_{t,\delta}(\phi) := \frac{40}{t}\log\frac{8|\Phi|^2(12L_\phi B_\phi t)^{d_\phi}t^3}{\delta} + \frac{2}{t}$ & Threshold for MSE elimination\\
  $D_t(\phi) := 160 d_\phi \log (12L_\phi B_\phi t)$ & Dimension factor for representation $\phi$\\
  $R_T := \sum_{t=1}^T \Delta(x_t,a_t)$ & Pseudo-regret\\
  $t_j := 2^j$ & Time at which the $(j+1)$-th phase ends (with $t_0 := 0$)\\
  $N_{j}(T) := \sum_{t=t_j+1}^{T} \indi{G_t}$ & Number of calls to $\mathfrak{A}$ in phase $j$ up to time $T\leq t_{j+1}$\\
  $G_t := \{ \mathrm{GLR_{t-1}(x_t;\phi_{t-1})} \leq \beta_{t-1,\delta/|\Phi|}(\phi_{t-1}) \}$ & Event under which the GLRT does not trigger at time $t$\\
  $S_T := \sum_{t=1}^T \indi{a_t \neq \pi^\star(x_t)}$ & Total number of sub-optimal pulls at time $T$\\
  $\wb{R}_{\mathfrak{A}}(T, \phi, \delta)$ & Regret bound of algorithm $\mathfrak{A}$ over $T$ steps when using $\phi$\\
  $g_T(\Phi,\Delta, \delta)$ & Bound on the sub-optimal pulls of $\mathfrak{A}$ (see Th. \ref{lem:suboptimal-pulls-strong-missp})\\
  $\delta_j := \delta / (2 (j+1)^2)$ & Confidence level for the base algorithm\\
  \bottomrule
  \end{tabular}
  \end{small}
  \caption{The notation adopted in this paper.}
  \label{tab:notation}
  \end{table*}

\section{Analysis of \algo}\label{app:analysis}

\subsection{Assumptions}


The analysis works under the assumptions stated in Section \ref{sec:preliminaries} and for any no-regret base algorithm $\mathfrak{A}$. Here we formally state the conditions required on the




\begin{assumption}[No-regret algorithm]\label{asm:no-regret-algo}
For any $\phi\in\Phi^\star$ and $\delta\in(0,1)$, if we run algorithm $\mathfrak{A}$ with representation $\phi$ and confidence $\delta$, with probability at least $1-\delta$ we have, for any $T\in\bN$,
\begin{align*}
  \sum_{t=1}^T \Delta(x_t,\mathfrak{A}_t(x_t; \phi,\delta)) \leq \wb{R}_{\mathfrak{A}}(T, \phi, \delta),
\end{align*}
where $\mathfrak{A}_t(x; \phi,\delta)$ denotes the policy played by $\mathfrak{A}$ at time $t$ when instantiated with representation $\phi$ and confidence $\delta$, while the function $\wb{R}_{\mathfrak{A}}(T, \phi, \delta)$ is sub-linear and non-decreasing in $T$ and logarithmic and non-decreasing in $1/\delta$.
\end{assumption}


\subsection{Controlling the MSE}

The following is an extension of Lemma 4.1 in \citep{AgarwalDKLS12} and Lemma 20 in \citep{PapiniTRLP21hlscontextual}. Differently from their results, which relate the empirical MSE of any model $(\phi,\theta)$ with that of a realizable model, we also include the sum of conditional mean prediction errors $P_t(\phi,\theta) := \sum_{k=1}^t \mathbb{E}_k\left[\left(\phi(x_k,a_k)^\transp \theta - \mu(x_k,a_k)\right)^2\right]$, which roughly quantifies the misspecification of model $(\phi,\theta)$. This shall be crucial for improving the elimination times of misspecified representations later.
\begin{lemma}\label{lemma:mse-single}
    Let $\phi\in\Phi,\theta\in\bR^{d_\phi}$. Take any realizable representation $\phi^\star\in\Phi^\star$ and let $\theta^\star := \theta^\star_{\phi^\star}$. Then, for each $t\geq 1$ and $\delta \in (0,1)$,
    \begin{align}
    \mathbb{P}\left( E_t(\phi^\star, \theta^\star) > E_t(\phi,\theta) + \frac{40}{t}\log\frac{4t}{\delta} - \frac{P_t(\phi,\theta)}{2t}\right) \leq \delta.
    \end{align}
\end{lemma}
\begin{proof}
    Define $Z_k := (\phi(x_k,a_k)^T\theta - y_k)^2 - (\phi^\star(x_k,a_k)^T\theta^\star - y_k)^2$. Note that, since $|\phi(x_k,a_k)^T\theta| \leq 1$, $|\phi^\star(x_k,a_k)^T\theta^\star| \leq 1$, and $|y_k|\leq 1$, we have $|Z_k|\leq 4$. Thus, $(\mathbb{E}_k[Z_k] - Z_k)_{k\geq 1}$ is a martingale difference sequence bounded by $8$ in absolute value. Then, using Freedman's inequality (Lemma \ref{lemma:freedman}), with probability at least $1-\delta$, for any $t$,
    \begin{align*}
        \sum_{k=1}^t \mathbb{E}_k[Z_k] - \sum_{k=1}^t Z_k \leq 2\sqrt{\sum_{k=1}^t \mathbb{V}_k[Z_k]\log\frac{4t}{\delta}} + 32\log \frac{4t}{\delta}.
    \end{align*}
    Using Lemma 4.2 in \citep{AgarwalDKLS12}, we have that $\mathbb{V}_k[Z_k] \leq 4 \mathbb{E}_k[Z_k]$. Solving the resulting inequality in $\sum_{k=1}^t \mathbb{E}_k[Z_k] $ and using $(x+y)^2 \leq 2x^2+2y^2$,
    \begin{align*}
        \sum_{k=1}^t \mathbb{E}_k[Z_k] \leq \left( 2\sqrt{\log\frac{4t}{\delta}} + \sqrt{36\log\frac{4t}{\delta} + \sum_{k=1}^t Z_k} \right)^2 \leq 80\log\frac{4t}{\delta} + 2\sum_{k=1}^t Z_k.
    \end{align*}
    The proof is concluded by using $\sum_{k=1}^t Z_k = t(E_t(\phi,\theta) - E_t(\phi^\star, \theta^\star))$ and $\sum_{k=1}^t \mathbb{E}_k[Z_k] = P_t(\phi,\theta)$.
  \end{proof}

\begin{lemma}\label{lemma:mse-multi}
    For each $\delta \in (0,1)$,
    \begin{align*}
    \mathbb{P}\left(\exists t\geq 1,\phi\in\Phi, \phi^\star\in\Phi^\star, \theta\in\mathcal{B}_\phi : E_t(\phi^\star, \theta^\star_{\phi^\star}) > E_t(\phi,\theta) - \frac{P_t(\phi,\theta)}{4t} + \alpha_{t,\delta}(\phi) \right) \leq \delta.
    \end{align*}
\end{lemma}
\begin{proof}
    We shall use a covering argument for each representation $\phi\in\Phi$. First note that, for any $\xi >0$, there always exists a finite set $\mathcal{C}_\phi \subset \mathbb{R}^{d_\phi}$ of size at most $(3B_\phi/\xi)^{d_\phi}$ such that, for each $\theta \in \mathcal{B}_\phi$, there exists ${\theta'}\in\mathcal{C}_\phi$ with $\|\theta-{\theta'}\|_2 \leq \xi$ (see e.g. Lemma 20.1 in \citep{lattimore2020bandit}). Moreover, suppose that all vectors in $\mathcal{C}_\phi$ have $\ell_2$-norm bounded by $B_\phi$ (otherwise we can always remove vectors with large norm). Now take any two vectors $\theta,{\theta'}\in\mathcal{B}_\phi$ with $\|\theta-{\theta'}\|_2 \leq \xi$. We have
    \begin{align*}
    E_t(\phi,\theta) &= \frac{1}{t} \sum_{k=1}^t \left(\phi(x_k,a_k)^\transp \theta \pm \phi(x_k,a_k)^\transp {\theta}' - y_k\right)^2
    \\ &= \frac{1}{t} \sum_{k=1}^t \left(\phi(x_k,a_k)^\transp (\theta - {\theta}')\right)^2 + \frac{1}{t} \sum_{k=1}^t \left(\phi(x_k,a_k)^\transp {\theta}' - y_k\right)^2
    \\ & \qquad\qquad\qquad\qquad\qquad\qquad\quad + \frac{2}{t} \sum_{k=1}^t \left(\phi(x_k,a_k)^\transp (\theta - {\theta}')\right)\left(\phi(x_k,a_k)^\transp {\theta}' - y_k\right)
    \\ &\geq E_t(\phi,{\theta}') + \frac{2}{t} \sum_{k=1}^t \left(\phi(x_k,a_k)^\transp (\theta - {\theta}')\right)\underbrace{\left(\phi(x_k,a_k)^\transp {\theta}' - y_k\right)}_{|\cdot|\leq 2}
    \\ &\geq E_t(\phi,{\theta}') - \frac{4}{t} \sum_{k=1}^t \|\phi(x_k,a_k)\|_{2} \|\theta - {\theta}'\|_2 \geq  E_t(\phi,{\theta}') - 4L_\phi\xi.
    \end{align*}
    Similarly, one can prove that
    \begin{align*}
        P_t(\phi,\theta) &= \sum_{k=1}^t \mathbb{E}_k\left[\left(\phi(x_k,a_k)^\transp \theta - \mu(x_k,a_k)\right)^2\right]
        \\ &\leq 2\sum_{k=1}^t \mathbb{E}_k\left[\left(\phi(x_k,a_k)^\transp \theta - \phi(x_k,a_k)^\transp \theta'\right)^2\right] +  2\sum_{k=1}^t \mathbb{E}_k\left[\left(\phi(x_k,a_k)^T \theta' - \mu(x_k,a_k)\right)^2\right]
        \\ &\leq 2P_t(\phi,\theta') + 2\sum_{k=1}^t \mathbb{E}_k\left[\|\phi(x_k,a_k)\|_{2}^2\right] \|\theta - {\theta}'\|_2^2 \leq 2P_t(\phi,\theta') + 2L_\phi^2 \xi^2t.
    \end{align*}
    Let us define a sequence of deterministic covers $(\cC_{\phi,t})_{t\geq 1}$ such that $\cC_{\phi,t}$ is a $\xi_t$-cover with $\xi_t = \frac{1}{4L_\phi t}$. Let $\delta'_t = \frac{\delta}{2|\Phi|^2(12L_\phi S_\phi t)^{d_\phi}}$ and note that $\alpha_{t,\delta}(\phi) := \frac{40}{t}\log\frac{4t^3}{\delta_t'} + \frac{2}{t}$. Then,
    \begin{align*}
    &\mathbb{P}\left(\exists t\geq 1, \phi\in\Phi, \phi^\star\in\Phi^\star, \theta\in\mathcal{B}_\phi : E_t(\phi^\star, \theta^\star_{\phi^\star}) > E_t(\phi,\theta) - \frac{P_t(\phi,\theta)}{4t} + \frac{40}{t}\log\frac{4t^3}{\delta'_t} + \frac{2}{t}\right)
    \\ &\leq \sum_{t=1}^\infty\sum_{\phi\in\Phi}\sum_{\phi^\star\in\Phi^\star}\mathbb{P}\left(\exists \theta\in\mathcal{B}_\phi : E_t(\phi^\star, \theta^\star_{\phi^\star}) > E_t(\phi,\theta) - \frac{P_t(\phi,\theta)}{4t} + \frac{40}{t}\log\frac{4t^3}{\delta'_t} + \frac{2}{t}\right)
    \\ &\leq \sum_{t=1}^\infty\sum_{\phi\in\Phi}\sum_{\phi^\star\in\Phi^\star}\mathbb{P}\left(\exists {\theta}'\in\mathcal{C}_\phi : E_t(\phi^\star, \theta^\star_{\phi^\star}) > E_t(\phi,{\theta'}) - \frac{1}{t} - \frac{2P_t(\phi,\theta') + 1 / (8t)}{4t} + \frac{40}{t}\log\frac{4t^3}{\delta'_t} + \frac{2}{t}\right)
    \\ &\leq \sum_{t=1}^\infty\sum_{\phi\in\Phi}\sum_{\phi^\star\in\Phi^\star}\sum_{{\theta}'\in\mathcal{C}_\phi}\mathbb{P}\left( E_t(\phi^\star, \theta^\star_{\phi^\star}) > E_t(\phi,{\theta}') - \frac{P_t(\phi,\theta')}{2t} + \frac{40}{t}\log\frac{4t^3}{\delta'_t}\right) 
    \\ &\leq \sum_{t=1}^\infty\sum_{\phi\in\Phi}\sum_{\phi^\star\in\Phi^\star}\sum_{{\theta}'\in\mathcal{C}_\phi} \frac{\delta'_t}{t^2}
    \leq |\Phi|^2\sum_{t=1}^\infty \frac{\delta'_t}{t^2}(12L_\phi B_\phi t)^{d_\phi} \leq \delta.
    \end{align*}
    Here the first inequality is from the union bound, the second one follows by relating $\theta$ with its closest vector in the cover as above, the third one is from another union bound, the fourth one uses Lemma \ref{lemma:mse-single}, the fifth one is from the maximum size of the cover, and the last one uses the definition of $\delta'_t$.
    \end{proof}

    \begin{corollary}\label{cor:mse-multi}
        For each $\delta \in (0,1)$,
        \begin{align*}
        \mathbb{P}\left(\exists t\geq 1,\phi\in\Phi, \phi^\star\in\Phi^\star, \theta\in\mathcal{B}_\phi : E_t(\phi^\star, \theta^\star_{\phi^\star}) > E_t(\phi,\theta) + \alpha_{t,\delta}(\phi)\right) \leq \delta.
        \end{align*}
    \end{corollary}
    \begin{proof}
        This is trivial from Lemma \ref{lemma:mse-multi} since $P_t(\phi,\theta) > 0$.
    \end{proof}

\subsection{Decomposition into phases}\label{app:phase-decomp}

For $j \geq 1$, let $t_j = 2^j$ be the time at which the $(j+1)$-th phase ends (i.e., when the algorithm selects a new representation for the $(j+1)$-th time). Let $t_0 = 0$. Note that, on the interval $[t_j+1, t_{j+1}]$ the algorithm uses a fixed representation $\phi_{j}$ selected at time $t_j$. In the remaining, we shall overload the notation used in the main paper and denote all quantities with a time subscript. Therefore, for $t\in [t_j+1, t_{j+1}]$, $\phi_{t-1} = \phi_{t_j}$ denotes the representation used a time $t$, i.e., $\phi_j$.

Recall that $G_t$ denotes the event under which the GLRT does not trigger at round $t$ (i.e., the base algorithm is called). Then, for each $j \geq 0$, the quantity
\begin{align*}
  \sum_{t=t_j+1}^{t_{j+1}} \indi{G_t} \Delta(x_t,a_t) 
\end{align*}
denotes the regret suffered by the base algorithm in phase $j$.

\subsection{Good events}

We define the following events
\begin{align*}
  \cE_1 &= \left\{\forall t\in\bN,\phi\in\Phi^\star : \| {\theta}_{\phi,t} - \theta^\star_\phi \|_{V_{t}(\phi)} \leq \beta_{t,\delta/|\Phi|}(\phi) \right\}, \\
  \cE_2 &= \Big\{ \forall t\in\bN,\phi\in\Phi : V_{t}(\phi) \succeq t\EV_{x \sim \rho}[\phi(x,\pi^\star(x))\phi(x,\pi^\star(x))^\transp]
  \\ & \qquad\qquad\qquad\qquad\qquad\qquad + \left( \lambda - L_\phi^2 S_t - 8L_\phi^2\sqrt{t\log(4d_\phi |\Phi|t/\delta)} \right) I_{d_\phi}\Big\},\\
  \cE_3 &= \Big\{ \forall t\in\bN,\phi\in\Phi : V_{t}(\phi) \preceq t\EV_{x \sim \rho}[\phi(x,\pi^\star(x))\phi(x,\pi^\star(x))^\transp]
  \\ & \qquad\qquad\qquad\qquad\qquad\qquad + \left( \lambda + L_\phi^2 S_t + 8L_\phi^2\sqrt{t\log(4d_\phi |\Phi|t/\delta)} \right) I_{d_\phi}\Big\},\\
  \cE_4 &= \left\{\forall t\in\bN,\phi\in\Phi, \phi^\star\in\Phi^\star, \theta\in\mathcal{B}_\phi : E_t(\phi^\star, \theta^\star_{\phi^\star}) \leq E_t(\phi,\theta) - \frac{P_t(\phi,\theta)}{4t} + \alpha_{t,\delta}(\phi) \right\},\\
  \cE_5 &= \left\{\forall j\in\bN, T\leq t_{j+1} : \sum_{t=t_j+1}^{T} \indi{G_t} \Delta(x_t,a_t) \leq \wb{R}_{\mathfrak{A}}\big( N_j(T), \phi_{t_j}, \delta_j/|\Phi| \big)\right\},
\end{align*}
We define the good event $\cE := \cE_1 \cap \cE_2 \cap \cE_3 \cap \cE_4 \cap \cE_5$.

\begin{lemma}[Good event]\label{lem:good-event-proba}
  We have $\bP(\cE) \geq 1 - 4\delta$.
\end{lemma}
\begin{proof}
    By using Theorem 2 in \citep{Abbasi-YadkoriPS11} together with a union bound over $\Phi$, $\bP(\cE_1) \geq 1-\delta$. Similarly, by Lemma \ref{lem:bound-design}, $\bP(E_2 \cap E_3) \geq 1-\delta$. Event $\cE_4$ holds with probability at least $1-\delta$ by Lemma \ref{lemma:mse-multi}. 
    
    We finally bound the probability of $\cE_5$ failing. We have
    \begin{align*}
      \bP(\neg \cE_5) \leq \sum_{j\in\bN} \bP\left\{ \exists T\leq t_{j+1} : \sum_{t=t_j+1}^{T} \indi{G_t} \Delta(x_t,a_t) > \wb{R}_{\mathfrak{A}}\big( T, \phi_{t_j}, \delta_j/|\Phi| \big)\right\} \leq \sum_{j\in\bN} \delta_j \leq \delta,
    \end{align*}
    where the first inequality is from a union bound over $j$, the second holds from the anytime no-regret assumption (Assumption \ref{asm:no-regret-algo}) together with a union bound over $\Phi$, while the last one holds by definition of $\delta_j$. A union bound over the 5 events proves the statement.
\end{proof}

\begin{lemma}\label{lemma:mse-correct}[Correctness of MSE elimination]
  Under event $\cE$, for each $t\geq 1$ any realizable representation  $\phi^\star\in\Phi^\star$ satisfies the constraint, i.e.,  $\phi^\star\in\Phi_t$.
\end{lemma}
\begin{proof}
  Under $\cE_4$,
  \begin{align*}
    \min_{\theta\in\cB_{\phi^\star}}E_t(\phi^\star, \theta) \leq E_t(\phi^\star, \theta^\star_{\phi^\star}) \leq \min_{\phi\in\Phi}\min_{\theta\in\cB_\phi} \left(E_t(\phi,\theta) + \alpha_{t,\delta}(\phi) \right).
  \end{align*}
  This implies the statement.
\end{proof}

\subsection{Generalized Likelihood Ratio Test}

For any $\phi\in\Phi$ and $x\in\cX$, let us define the \emph{generalized likelihood ratio} as
\begin{align*}
  \mathrm{GLR}_t(x;\phi) := \min_{a\neq \pi^\star_t(x;\phi)} \frac{\big(\phi(x,\pi^\star_t(x;\phi)) - \phi(x,a)\big)^\transp\theta_{\phi,t}}{\|\phi(x, \pi^\star_t(x;\phi)) - \phi(s,a)\|_{V_{t}(\phi)^{-1}}}.
\end{align*}
It is known \citep[e.g.,][]{hao2020adaptive,tirinzoni2020asymptotically} that 
\begin{align*}
  \mathrm{GLR}_t(x;\phi) = \inf_{\theta \in \Lambda_t(x;\phi)} \| {\theta}_{\phi,t} - \theta \|_{V_{t}(\phi)},
\end{align*}
where $\Lambda_t(x;\phi) := \{\theta\in\bR^{d_\phi} \mid \exists a \neq \pi^\star_t(x;\phi) : \phi(x,a)^\transp\theta > \phi(x,\pi^\star_t(x;\phi))^\transp\theta\}$ is the set of parameters for which the optimal arm in context $x$ is different from the one of $\theta_{\phi,t}$. In turns, the squared objective above is equivalent to
\begin{align*}
  \frac{1}{2}\| {\theta}_{\phi,t} - \theta \|_{V_{t}(\phi)}^2 = \frac{1}{2}\sum_{k=1}^t \left(\phi(x_k,a_k)^\transp\theta_{\phi,t} - \phi(x_k,a_k)^\transp\theta\right)^2,
\end{align*}
which is equal to the expected (under the conditional reward distribution) log-likelihood ratio between the observations in the bandit model given by $(\phi,\theta_{\phi,t})$ and the one given by $(\phi,\theta)$ if these were Gaussians with unit variance. This is the reason why $\mathrm{GLR}_t(x;\phi)$ is called the generalized likelihood ratio between the bandit model $(\phi,\theta_{\phi,t})$ and \emph{any} other bandit model with a different optimal arm in context $x$. The generalized likelihood ratio test (GLRT) consists in checking whether
\begin{align*}
  \mathrm{GLR_{t}(x;\phi)} > \beta_{t,\delta}(\phi).
\end{align*}
When this happens, we have enough confidence to conclude that $\theta^\star_\phi \notin \Lambda_t(x;\phi)$, i.e., that $\pi^\star(x) = \pi^\star_t(x;\phi)$.

\algo computes, at each step, the GLRT for the currently selected representation. We can easily prove that the test is \emph{correct} under the good event $\cE$ if the selected representation is realizable.
\begin{lemma}[Correctness of GLRT]\label{lem:glrt-correct-multi}
  Under the good event $\cE$, for any time $t$, if $\mathrm{GLR_{t-1}}(x_t;\phi_{t-1}) > \beta_{t-1,\delta/|\Phi|}(\phi_{t-1})$ and $\phi_{t-1}\in\Phi^\star$, then $\pi^\star(x_t) = \pi^\star_t(x_t;\phi_t)$.
\end{lemma}
\begin{proof}
  By contradiction, suppose that the statement does not hold. This means that there exists a time $t$, realizable feature $\phi\in\Phi^\star$, and context $x$ such that $\pi^\star(x) \neq \pi^\star_t(x;\phi)$ while the test triggers for context $x$ and feature $\phi$. By definition, this implies that $\theta^\star_\phi \in \Lambda_t(x;\phi)$ since $\pi^\star$ is the greedy policy for the (realizable) model $(\phi,\theta^\star_\phi)$. Thus,
\begin{align*}
  \beta_{t,\delta/|\Phi|}(\phi) < \mathrm{GLR}_t(x;\phi) = \inf_{\theta \in \Lambda_t(x;\phi)} \| {\theta}_{\phi,t} - \theta \|_{V_{t}(\phi)} \leq  \| {\theta}_{\phi,t} - \theta^\star_\phi \|_{V_{t}(\phi)} \leq \beta_{t,\delta/|\Phi|}(\phi),
\end{align*}
where the last inequality is from event $\cE_1$. This is clearly a contradiction.
\end{proof}

\subsection{Eliminating misspecified representations}

\begin{lemma}\label{lem:active-missp}
  Let $\phi\in\Phi$ be any misspecified representation (i.e., $\phi\notin\Phi^\star$). Under event $\cE$, if $\phi\in\Phi_t$ for some $t$, then
  \begin{align*}
    \min_{\theta\in\cB_\phi} P_t(\phi,\theta) \leq D_t(\phi) + \min_{\phi^\star\in\Phi^\star} D_t(\phi^\star) + 328\log\frac{8|\Phi|^2t^3}{\delta},
  \end{align*}
  where $D_t(\phi) := 160 d_\phi \log (12L_\phi B_\phi t)$.
  \end{lemma}
  \begin{proof}
    Recall that, from Lemma \ref{lemma:mse-correct}, under $\cE$, any $\phi^\star\in\Phi^\star$ is always in $\Phi_t$. Take any arbitrary $\phi^\star\in\Phi^\star$ and let $\theta^\star := \theta^\star_{\phi^\star}$. Then, by definition of $\Phi_t$, 
    \begin{align*}
      \min_{\theta\in\cB_\phi}E_{t}(\phi, \theta) \leq \min_{\phi'\in\Phi}\min_{\theta\in\mathcal{B}_{\phi'}} \big\{ E_{t}(\phi',\theta) + \alpha_{t,\delta}(\phi') \big\} \leq  E_{t}(\phi^\star,\theta^\star) + \alpha_{t,\delta}(\phi^\star).
    \end{align*}
    Similarly, under $\cE_4$ we have that
    \begin{align*}
      E_t(\phi^\star, \theta^\star) \leq \min_{\theta\in\cB_\phi} \left( E_t(\phi,\theta) - \frac{P_t(\phi,\theta)}{4t} \right) + \alpha_{t,\delta}(\phi) \leq \min_{\theta\in\cB_\phi} E_t(\phi,\theta) - \frac{\min_{\theta\in\cB_\phi} P_t(\phi,\theta)}{4t} + \alpha_{t,\delta}(\phi).
    \end{align*}
    Combining these two inequalities, we find that
    \begin{align*}
      \frac{\min_{\theta\in\cB_\phi} P_t(\phi,\theta)}{4t} \leq \alpha_{t,\delta}(\phi) + \alpha_{t,\delta}(\phi^\star).
    \end{align*}
    Expanding the definition of $\alpha$, rearranging, and optimizing over $\phi^\star$,
    \begin{align*}
      \min_{\theta\in\cB_\phi} P_t(\phi,\theta) \leq D_t(\phi) + \min_{\phi^\star\in\Phi^\star} D_t(\phi^\star) + 320\log\frac{8|\Phi|^2t^3}{\delta} + 16.
    \end{align*}
    The proof is concluded by noting that $\log\frac{8|\Phi|^2t^3}{\delta} \geq 2$ and, thus, $16 \leq 8 \log\frac{8|\Phi|^2t^3}{\delta}$.
  \end{proof}

  \begin{lemma}[Elimination]\label{lem:elim-strong-missp}
    Under event $\cE$, we have $\Phi_t = \Phi^\star$ for all $t\geq \tau_{\mathrm{elim}}$, where
    \begin{align*}
      \tau_{\mathrm{elim}} := \min_{t \in \bN} \left\{ t \mid \exists j\in\bN_{>0} : t=2^j,  t > \max_{\phi\notin\Phi^\star}\frac{1}{\epsilon_\phi}\left( D_t(\phi) + \min_{\phi^\star\in\Phi^\star} D_t(\phi^\star) + 328\log\frac{8|\Phi|^2t^3}{\delta} \right) \right\}.
    \end{align*}
    Let $\tau_{\mathrm{elim}} = 0$ when $\Phi = \Phi^\star$.
  \end{lemma}
  \begin{proof}
    Let $\pi_k$ be the policy played by the algorithm at round $k$. First note that, 
    \begin{align*}
      \min_{\theta\in\cB_\phi} P_t(\phi,\theta)
       &= \min_{\theta\in\cB_\phi} \sum_{k=1}^t \mathbb{E}_k\left[\left(\phi(x_k,a_k)^T \theta - \mu(x_k,a_k)\right)^2\right]
       \\ &= \min_{\theta\in\cB_\phi} \sum_{k=1}^t \mathbb{E}_{x\sim\rho}\left[\left(\phi(x,\pi_k(x))^T \theta - \mu(x,\pi_k(x))\right)^2\right]
       \\ &\geq \min_{\theta\in\cB_\phi} \sum_{k=1}^t \min_\pi\mathbb{E}_{x\sim\rho}\left[\left(\phi(x,\pi(x))^T \theta - \mu(x,\pi(x))\right)^2\right]
       \\ &= t \min_{\theta\in\cB_\phi}\min_\pi\mathbb{E}_{x\sim\rho}\left[\left(\phi(x,\pi(x))^T \theta - \mu(x,\pi(x))\right)^2\right] \geq t \epsilon_\phi.
    \end{align*}
    Then, under $\cE$, from Lemma \ref{lem:active-missp}, if $\phi\in\Phi_t$ and $\phi\notin\Phi^\star$,
    \begin{align*}
      t \leq \frac{1}{\epsilon_\phi}\left( D_t(\phi) + \min_{\phi^\star\in\Phi^\star} D_t(\phi^\star) + 200\log\frac{8|\Phi|^2t^3}{\delta} \right).
    \end{align*}
    The result follows by finding the first time $t$ at which a representation update is performed (i.e., $t=2^j$ for some $j$) and the condition above is violated for all $\phi\notin\Phi^\star$.
  \end{proof}

  \subsection{Regret bound without HLS representations}

  We first prove a general regret bound that holds for any realizable problem (in the sense of Assumption \ref{asm:set.contains.realizable.phi}) without requiring the presence of HLS representations.

\begin{theorem}\label{th:regret-strong-missp-nohls}
  Under event $\cE$ (i.e., with probability at least $1-4\delta$), for any $T\in\mathbb{N}$, the regret of Algorithm \ref{alg:replearnin.icml.asm} with $\gamma=2$ and arbitrary loss $\cL_t(\phi)$ can be bounded as
  \begin{align*}
    R_T \leq 2\tau_{\mathrm{elim}} + \max_{\phi\in\Phi^\star} \wb{R}_{\mathfrak{A}}(T - \tau_{\mathrm{elim}}, \phi, \delta_{\log_2(T)}/|\Phi|) \log_2(T),
  \end{align*}
  where $\tau_{\mathrm{elim}}$ is defined in Lemma \ref{lem:elim-strong-missp}
\end{theorem}
\begin{proof}
  Let $\bar{j}$ be such that $\tau_{\mathrm{elim}} = 2^{\bar{j}}$ (which exists by definition). Using the decomposition into phases of Appendix \ref{app:phase-decomp},
  \begin{align*}
    R_T := \sum_{t=1}^T \Delta(x_t,a_t) &= \sum_{j=0}^{\bar{j}-1} \sum_{t=t_j+1}^{t_{j+1} \wedge T} \Delta(x_t,a_t) + \sum_{j=\bar{j}}^{\lfloor\log_2(T)\rfloor} \sum_{t=t_j+1}^{t_{j+1} \wedge T} \Delta(x_t,a_t)
    \\ &\leq 2\tau_{\mathrm{elim}} + \sum_{j=\bar{j}}^{\lfloor\log_2(T)\rfloor} \sum_{t=t_j+1}^{t_{j+1} \wedge T} \Delta(x_t,a_t),
  \end{align*}
  where the second inequality holds by definition of $\bar{j}$ and because the rewards are bounded in $[-1,1]$. It only remains to bound the regret on phases after $\bar{j}$. By Lemma \ref{lem:elim-strong-missp}, we have $\phi_t\in\Phi^\star$ at all times in such phases.
  
  Let $G_t := \{ \mathrm{GLR_{t-1}(x_t;\phi_{t-1})} \leq \beta_{t-1,\delta/|\Phi|}(\phi_{t-1}) \}$ be the event under which the GLRT does not trigger at time $t$. For any $j \geq \bar{j}$,
  \begin{align*}
    \sum_{t=t_j+1}^{t_{j+1} \wedge T} \Delta(x_t,a_t) = \sum_{t=t_j+1}^{t_{j+1} \wedge T} \indi{G_t} \Delta(x_t,a_t) + \sum_{t=t_j+1}^{t_{j+1} \wedge T} \indi{\neg G_t} \Delta(x_t,a_t) = \sum_{t=t_j+1}^{t_{j+1} \wedge T} \indi{G_t} \Delta(x_t,a_t),
  \end{align*}
  where the last equality holds since, under $\cE$, if $G_t$ does not hold, then the GLRT triggers, $a_t = \pi^\star_{t-1}(x_t;\phi_{t-1})$, and $\pi^\star_{t-1}(x_t;\phi_{t-1}) = \pi^\star(x_t)$ by Lemma \ref{lem:glrt-correct-multi}. Let $N_{j} := \sum_{t=t_j+1}^{t_{j+1} \wedge T} \indi{G_t}$ be the total number of times the base algorithm $\mathfrak{A}$ is called in phase $j$. By event $\cE_5$, the regret of $\mathfrak{A}$ on such steps is bounded as
  \begin{align*}
    \sum_{t=t_j+1}^{t_{j+1} \wedge T} \indi{G_t} \Delta(x_t,a_t) \leq \wb{R}_{\mathfrak{A}}(N_j(t_{j+1} \wedge T), \phi_{t_j}, \delta_j/|\Phi|).
  \end{align*}
  Note that, for all $j\geq\bar{j}$, $N_j(t_{j+1} \wedge T) \leq t_{j+1} \wedge T - t_j \leq T - t_{\bar{j}} = T - \tau_{\mathrm{elim}}$. Morevoer, the number of phases is $j \leq \log_2(T)$. Therefore, by the fact that $\wb{R}_{\mathfrak{A}}(\cdot, \phi, \cdot)$ is non-decreasing in the first and third argument,
  \begin{align*}
    R_T &\leq 2\tau_{\mathrm{elim}} + \sum_{j=\bar{j}}^{\lfloor\log_2(T)\rfloor} \wb{R}_{\mathfrak{A}}(T - \tau_{\mathrm{elim}}, \phi_{t_j}, \delta_{\log_2(T)}/|\Phi|)
    \\ &\leq 2\tau_{\mathrm{elim}} + \max_{\phi\in\Phi^\star}\wb{R}_{\mathfrak{A}}(T - \tau_{\mathrm{elim}}, \phi, \delta_{\log_2(T)}/|\Phi|)\log_2(T).
  \end{align*}
\end{proof}

\begin{lemma}[Bound on sub-optimal pulls]\label{lem:suboptimal-pulls-strong-missp}
  Under the same conditions as Theorem \ref{th:regret-strong-missp-nohls}, under event $\cE$ (i.e., with probability at least $1-4\delta$), for any $T\in\bN$,
  \begin{align*}
    S_T = \sum_{t=1}^T \indi{a_t \neq \pi^\star(x_t)} \leq \frac{2\tau_{\mathrm{elim}} + \max_{\phi\in\Phi^\star} \wb{R}_{\mathfrak{A}}(T, \phi, \delta_{\log_2(T)}/|\Phi|) \log_2(T)}{\Delta} =: g_T(\Phi,\Delta, \delta).
  \end{align*}
\end{lemma}
\begin{proof}
  Note that, since the minimum gap is at least $\Delta$, the event $\{a_t \neq \pi^\star(x_t)\}$ implies that $\Delta(x_t,a_t) \geq \Delta$. Then,
  \begin{align*}
    S_T \leq \sum_{t=1}^T \indi{\Delta(x_t,a_t) \geq \Delta} &\leq \sum_{t=1}^T \frac{\Delta(x_t,a_t)}{\Delta}
    \\ &\leq \frac{2\tau_{\mathrm{elim}} + \max_{\phi\in\Phi^\star} \wb{R}_{\mathfrak{A}}(T, \phi, \delta_{\log_2(T)}/|\Phi|) \log_2(T)}{\Delta},
  \end{align*}
  where the last inequality holds by Theorem \ref{th:regret-strong-missp-nohls}.
\end{proof}

\subsection{Regret bound with HLS representations}

\begin{lemma}[Selecting the HLS representation]\label{lem:select-hls}
  Suppose Algorithm \ref{alg:replearnin.icml.asm} is run with $\gamma=2$ and $\cL_t(\phi) = -\lambda_{\min}(V_{t}(\phi) - \lambda I_{d_\phi})/L_\phi^2$. Suppose that there exists a unique $\phi^\star\in\Phi^\star$ such that $\phi^\star$ is HLS. Then, under event $\cE$ (i.e., with probability at least $1-4\delta$), $\phi_t = \phi^\star$ for all $t \geq \tau_{\mathrm{hls}} \vee \tau_{\mathrm{elim}}$, where
  \begin{align*}
    \tau_{\mathrm{hls}} := \min_{t \in \bN} \left\{ t \mid \exists j\in\bN_{>0} : t=2^j,  t > \frac{2L_{\phi^\star}^2}{\lambda^\star(\phi^\star)}\left(g_t(\Phi,\Delta, \delta) + 8\sqrt{t\log\frac{4 |\Phi|t \max_{\phi\in\Phi^\star} d_{\phi}}{\delta}}\right) \right\}.
  \end{align*}
\end{lemma}
\begin{proof}
  Take any time $t\geq \tau_{\mathrm{elim}}$. By Lemma \ref{lem:elim-strong-missp}, we have $\Phi_t=\Phi^\star$ and, thus, $\phi^\star$ is the only active HLS representation. By the min-max theorem, $A\preceq B$ implies $\lambda_k(A)\le\lambda_k(B)$ where $\lambda_k$ is the $k$-th largest eigenvalue of the matrix. Then, from event $\cE$, we have that, for all $t$,
  \begin{align*}
    \lambda_{\min}(V_{t}(\phi^\star) - \lambda I_{d_{\phi^\star}}) &\geq t\lambda^\star(\phi^\star) - L_{\phi^\star}^2 S_t - 8L_{\phi^\star}^2\sqrt{t\log(4d_{\phi^\star} |\Phi|t/\delta)},\\
    \lambda_{\min}(V_{t}(\phi) - \lambda I_{d_\phi}) &\leq L_\phi^2 S_t + 8L_\phi^2\sqrt{t\log(4d_\phi |\Phi|t/\delta)} \quad \forall \phi\in\Phi^\star,\phi\neq\phi^\star.
   \end{align*}
    If $t=2^j$ for some $j\in\bN$ (i.e., a time where representation selection is performed), $\phi^\star$ is selected if
    \begin{align*}
      \lambda_{\min}(V_{t}(\phi^\star) - \lambda I_{d_{\phi^\star}})/L_{\phi^\star}^2
      > \max_{\phi\in\Phi^\star,\phi\neq\phi^\star} \lambda_{\min}(V_{t}(\phi) - \lambda I_{d_\phi})/L_\phi^2.
    \end{align*}
    A sufficient condition based on the bounds above is
    \begin{align*}
      t\frac{\lambda^\star(\phi^\star)}{L_{\phi^\star}^2} > 2S_t + 8\sqrt{t\log(4d_{\phi^\star} |\Phi|t/\delta)} + \max_{\phi\in\Phi^\star,\phi\neq\phi^\star} \left(8\sqrt{t\log(4d_\phi |\Phi|t/\delta)} \right).
    \end{align*}
    This, in turns, yields the simpler sufficient condition
    \begin{align*}
      t\frac{\lambda^\star(\phi^\star)}{L_{\phi^\star}^2} > 2S_t + 16\sqrt{t\log\left(\frac{4 |\Phi|t \max_{\phi\in\Phi^\star} d_{\phi}}{\delta}\right)}.
    \end{align*}
    Finally, using Lemma \ref{lem:suboptimal-pulls-strong-missp} to bound $S_t$, it is sufficient that
    \begin{align*}
      t\frac{\lambda^\star(\phi^\star)}{L_{\phi^\star}^2} > 2g_t(\Phi,\Delta, \delta)  + 16\sqrt{t\log\left(\frac{4 |\Phi|t \max_{\phi\in\Phi^\star} d_{\phi}}{\delta}\right)}.
    \end{align*}
    The right-hand side is a sub-linear function of $t$. The proof is concluded by rearringing this inequality and defining the first update time that satisfies it.
\end{proof}

\begin{lemma}[Triggering the GLRT]\label{lem:trigger-glrt}
  Suppose Algorithm \ref{alg:replearnin.icml.asm} is run with $\gamma=2$ and $\cL_t(\phi) = -\lambda_{\min}(V_{t}(\phi) - \lambda I_{d_\phi})/L_\phi^2$. Suppose that there exists a unique $\phi^\star\in\Phi^\star$ such that $\phi^\star$ is HLS. Then, under the good event $\cE$, the GLRT triggers for all for all $t \geq \tau_{\mathrm{glrt}} \vee \tau_{\mathrm{hls}} \vee \tau_{\mathrm{elim}}$, where
  \begin{align*}
    \tau_{\mathrm{glrt}} := \min_{t \in \bN} \left\{ t \mid t \geq \frac{L_{\phi^\star}^2}{\lambda^\star(\phi^\star)} 
    \left(\frac{16\beta_{t,\delta/|\Phi|}(\phi^\star)^2}{\Delta^2} + g_t(\Phi,\Delta, \delta) + 8\sqrt{t\log(4d_{\phi^\star} |\Phi|t/\delta)}\right) + 1  \right\}.
  \end{align*}
\end{lemma}
\begin{proof}
From Lemma \ref{lem:select-hls}, we know that $\phi_t = \phi^\star$ for all $t \geq \tau_{\mathrm{hls}} \vee \tau_{\mathrm{elim}}$.  For simplicity, let us call $\phi := \phi^\star$.
Take any time step $t \geq \tau_{\mathrm{hls}} \vee \tau_{\mathrm{elim}}$ (for which $\phi_t=\phi$), any $x\in\mathcal{X}$, and any $a\neq \pi^\star_t(x;\phi)$. Then, by the good event $\cE$,
\begin{align*}
\|\phi(x, \pi^\star_t(x;\phi)) - \phi(s,a)\|_{V_{t}(\phi)^{-1}} &\leq \frac{2L_\phi}{\sqrt{\lambda_{\min}(V_{t}(\phi))}} 
\\ &\leq \frac{2L_\phi}{\sqrt{t\lambda^\star(\phi) + \lambda - L_{\phi}^2 S_t - 8L_{\phi}^2\sqrt{t\log(4d_{\phi} |\Phi|t/\delta)}}}.
\end{align*}
Similarly,
\begin{align*}
  \big(\phi(x,\pi^\star_t(x;\phi)) - \phi(x,a)\big)^\transp\theta_{\phi,t} &\geq \big(\phi(x,\pi^\star(x)) - \phi(x,a)\big)^\transp\theta_{\phi,t}
  \\ &= \Delta(x,a) + \big(\phi(x,\pi^\star(x)) - \phi(x,a)\big)^\transp(\theta_{\phi,t}-\theta^\star_\phi)
  \\ &\geq \Delta(x,a) - \| \phi(x,\pi^\star(x)) - \phi(x,a) \|_{V_t(\phi)^{-1}}\|\theta_{\phi,t}-\theta^\star_\phi\|_{V_t(\phi)}
  \\ & \geq \Delta(x,a) - \frac{2L_\phi\beta_{t,\delta/|\Phi|}(\phi)}{\sqrt{\lambda_{\min}(V_{t}(\phi))}}
  \\ &\geq \Delta(x,a) - \frac{2L_\phi\beta_{t,\delta/|\Phi|}(\phi)}{\sqrt{t\lambda^\star(\phi) + \lambda - L_{\phi}^2 S_t - 8L_{\phi}^2\sqrt{t\log(4d_{\phi} |\Phi|t/\delta)}}}
  \\ &\geq \Delta - \frac{2L_\phi\beta_{t,\delta/|\Phi|}(\phi)}{\sqrt{t\lambda^\star(\phi) + \lambda - L_{\phi}^2 S_t - 8L_{\phi}^2\sqrt{t\log(4d_{\phi} |\Phi|t/\delta)}}}.
\end{align*}
Now suppose $t$ is large enough so that the right-hand side is at least $\Delta/2$. Then, using the two inequalities above,
\begin{align*}
  \mathrm{GLR}_t(x;\phi) &= \min_{a\neq \pi^\star_t(x;\phi)} \frac{\big(\phi(x,\pi^\star_t(x;\phi)) - \phi(x,a)\big)^\transp\theta_{\phi,t}}{\|\phi(x, \pi^\star_t(x;\phi)) - \phi(s,a)\|_{V_{t}(\phi)^{-1}}}
   \\ &\geq \frac{\Delta}{4L_\phi}\sqrt{t\lambda^\star(\phi) + \lambda - L_{\phi}^2 S_t - 8L_{\phi}^2\sqrt{t\log(4d_{\phi} |\Phi|t/\delta)}}.
\end{align*}
Thus, a sufficient condition for the test trigger at time $t+1$ (recall that at time $t+1$ we perform the test with the statistics up to time $t$) is that the right-hand side above is larger than $\beta_{t,\delta/|\Phi|}(\phi)$. Therefore, for the test to trigger forever, we need simultaneously that
\begin{align*}
  \frac{\Delta}{4L_\phi}\sqrt{t\lambda^\star(\phi) + \lambda - L_{\phi}^2 S_t - 8L_{\phi}^2\sqrt{t\log(4d_{\phi} |\Phi|t/\delta)}} \geq \beta_{t,\delta/|\Phi|}(\phi)
\end{align*}
and that $t \geq \tau_{\mathrm{hls}} \vee \tau_{\mathrm{elim}}$. Note that this condition implies that the empirical gap is at least $\Delta/2$ as we required above. Using Lemma \ref{lem:suboptimal-pulls-strong-missp} to bound $S_t$ and rearranging concludes the proof.
\end{proof}

\begin{theorem}[Regret bound with HLS representation]\label{th:regret-strong-missp-hls}
  Suppose Algorithm \ref{alg:replearnin.icml.asm} is run with $\gamma=2$ and $\cL_t(\phi) = -\lambda_{\min}(V_{t}(\phi) - \lambda I_{d_\phi})/L_\phi^2$. Suppose $\phi^\star$ is the unique HLS representation in $\Phi^\star$. Under event $\cE$ (i.e., with probability at least $1-4\delta$), for any $T\in\bN$,
  \begin{align*}
    R_T \leq 2\tau_{\mathrm{elim}} + \max_{\phi\in\Phi^\star} \wb{R}_{\mathfrak{A}}(\tau - \tau_{\mathrm{elim}}, \phi, \delta_{\log_2(\tau)}/|\Phi|) \log_2(\tau),
  \end{align*}
  where $\tau := \tau_{\mathrm{glrt}} \vee \tau_{\mathrm{hls}} \vee \tau_{\mathrm{elim}}$.
  \end{theorem}
  \begin{proof}
    Under $\cE$, Lemma \ref{lem:trigger-glrt} ensures that the GLRT triggers for $t \geq \tau_{\mathrm{glrt}} \vee \tau_{\mathrm{hls}} \vee \tau_{\mathrm{elim}}$ with a realizable representation and, thus, the regret is zero for those times. Then, the result follows by using Theorem \ref{th:regret-strong-missp-nohls} to bound the regret up to time $\tau_{\mathrm{glrt}} \vee \tau_{\mathrm{hls}} \vee \tau_{\mathrm{elim}}$.
\end{proof}




\subsection{Finding explicit bounds}

\begin{lemma}\label{lem:ineq-log-sqrt}
  For $x\in\bR$ and $c_1,c_2,c_3,c_4 \geq 0$, consider the inequality $x \leq c_1 + c_2\sqrt{x} + c_3\sqrt{x\log(x)} + c_4\log(x)$. Then, $x \lesssim c_1 + c_2^2 + c_3^2 + c_4$, where the $\lesssim$ notation hides constant and logarithmic terms.
\end{lemma}
\begin{proof}
  We can start by finding a crude bound on $x$ by using the inequality $\log(x) \leq x^{\alpha}/\alpha$ for any $x,\alpha \geq 0$. Using it for $\alpha = 1/2$, we obtain
  \begin{align*}
    x \leq c_1 + c_2\sqrt{x} + \sqrt{2} c_3 x^{3/4} + 2c_4\sqrt{x}.
  \end{align*}
  Suppose that $x \geq 1$. Then, $x \leq (c_1 + c_2 + \sqrt{2} c_3 + 2c_4) x^{3/4}$, which implies that $x \leq (c_1 + c_2 + \sqrt{2} c_3 + 2c_4)^4$. Therefore, we have $x\leq C$ for $C := \max\{(c_1 + c_2 + \sqrt{2} c_3 + 2c_4)^4, 1\}$. Plugging this into the logarithms in our initial inequality,
  \begin{align*}
    x \leq c_1 + (c_2 + c_3\sqrt{\log(C)})\sqrt{x} + c_4\log(C).
  \end{align*}
  Solving this second-order inequality in $\sqrt{x}$ and using $(a+b)^2 \leq 2a^2 + 2b^2$, we obtain
  \begin{align*}
    x &\leq \left( \frac{c_2 + c_3\sqrt{\log(C)}}{2} + \sqrt{\frac{(c_2 + c_3\sqrt{\log(C)})^2}{4} + c_1 + c_4\log(C)} \right)^2
    \\ &\leq (c_2 + c_3\sqrt{\log(C)})^2 + 2c_1 + 2c_4\log(C) \lesssim c_1 + c_2^2 + c_3^2 + c_4.
  \end{align*}
\end{proof}

\begin{lemma}\label{lem:scale-tau-elim}
  The elimination time $\tau_{\mathrm{elim}}$ defined in Lemma \ref{lem:elim-strong-missp} satisfies
  \begin{align*}
    \tau_{\mathrm{elim}} \lesssim \frac{d\log(|\Phi|/\delta)}{\min_{\phi\notin\Phi^\star}\epsilon_\phi}.
  \end{align*}
\end{lemma}
\begin{proof}
  We know that $\tau_{\mathrm{elim}} = 2^j$ for some specific $j$. Let $t=2^{j-1}$ be the time at which the last update before $\tau_{\mathrm{elim}}$ was performed. By definition, we have that
  \begin{align*}
    t &\leq \max_{\phi\notin\Phi^\star}\frac{1}{\epsilon_\phi}\left( D_t(\phi) + \min_{\phi^\star\in\Phi^\star} D_t(\phi^\star) + 328\log\frac{8|\Phi|^2t^3}{\delta} \right)
    \\ &\leq \frac{320d\log(12BL) + 320d\log(t) + 328d\log(8|\Phi|^2/\delta) + 984\log(t)}{\min_{\phi\notin\Phi^\star}\epsilon_\phi},
  \end{align*}
  where we used some simple crude bounds in the second inequality. Then, by Lemma \ref{lem:ineq-log-sqrt}, $t \lesssim \frac{d\log(|\Phi|/\delta)}{\min_{\phi\notin\Phi^\star}\epsilon_\phi}$ and the same holds for $\tau_{\mathrm{elim}}$ since $\tau_{\mathrm{elim}} = 2t$. 
 \end{proof}

 \begin{lemma}\label{lem:scale-tau-hls}
  The time $\tau_{\mathrm{hls}}$ defined in Lemma \ref{lem:select-hls} satisfies
  \begin{align*}
    \tau_{\mathrm{hls}} \lesssim \tau_{\mathrm{alg}} + \frac{L_{\phi^\star}^4\log(|\Phi|/\delta)}{\lambda^\star(\phi^\star)^2} + \frac{\tau_{\mathrm{elim}}L_{\phi^\star}^2}{\lambda^\star(\phi^\star)\Delta},
  \end{align*}
  where
  \begin{align*}
    \tau_{\mathrm{alg}} := \min_{t \in \bN} \left\{ t \mid \exists j\in\bN_{>0} : t=2^j,  t > \frac{8L_{\phi^\star}^2\log_2(t)}{\lambda^\star(\phi^\star)\Delta}\max_{\phi\in\Phi^\star} \wb{R}_{\mathfrak{A}}(t, \phi, \delta_{\log_2(t)}/|\Phi|)\right\}.
  \end{align*}
\end{lemma}
\begin{proof}
  By definition of $\tau_{\mathrm{hls}}$,
  \begin{align*}
    \tau_{\mathrm{hls}} \leq \min_{t \in \bN} \left\{ t \mid \exists j\in\bN_{>0} : t=2^j,  t > 2\max\left(\frac{2L_{\phi^\star}^2}{\lambda^\star(\phi^\star)}g_t(\Phi,\Delta, \delta), \frac{16L_{\phi^\star}^2}{\lambda^\star(\phi^\star)} \sqrt{t\log\frac{4 |\Phi|t \max_{\phi\in\Phi^\star} d_{\phi}}{\delta}}\right) \right\}.
  \end{align*}
  Thus, $\tau_{\mathrm{hls}} \lesssim \tau_{\mathrm{hls}}' + \tau_{\mathrm{hls}}''$, where
  \begin{align*}
    \tau_{\mathrm{hls}}' &:= \min_{t \in \bN} \left\{ t \mid \exists j\in\bN_{>0} : t=2^j,  t > \frac{4L_{\phi^\star}^2}{\lambda^\star(\phi^\star)}g_t(\Phi,\Delta, \delta)\right\},\\
    \tau_{\mathrm{hls}}'' &:= \min_{t \in \bN} \left\{ t \mid \exists j\in\bN_{>0} : t=2^j,  t > \frac{32L_{\phi^\star}^2}{\lambda^\star(\phi^\star)} \sqrt{t\log\frac{4 |\Phi|t \max_{\phi\in\Phi^\star} d_{\phi}}{\delta}}\right\}.
  \end{align*}
  We now bound $\tau_{\mathrm{hls}}''$. We know that $\tau_{\mathrm{hls}}'' = 2^j$ for some specific $j$. Let $t=2^{j-1}$ be the time at which the last update before $\tau_{\mathrm{hls}}''$ was performed. By definition, we have that
  \begin{align*}
    t &\leq \frac{32L_{\phi^\star}^2}{\lambda^\star(\phi^\star)} \sqrt{t\log\frac{4 |\Phi|t \max_{\phi\in\Phi^\star} d_{\phi}}{\delta}}
     \leq \frac{32L_{\phi^\star}^2}{\lambda^\star(\phi^\star)} \left(\sqrt{t\log\frac{4 |\Phi| d}{\delta}} + \sqrt{t\log(t)}\right)
    \lesssim \frac{L_{\phi^\star}^4\log(|\Phi|/\delta)}{\lambda^\star(\phi^\star)^2},
  \end{align*}
  where we used Lemma \ref{lem:ineq-log-sqrt}. The same holds for $\tau_{\mathrm{hls}}''$ since $\tau_{\mathrm{hls}}'' = 2t$. We can now apply the same trick to $\tau_{\mathrm{hls}}'$ by expanding the definition of $g_t(\Phi,\Delta, \delta)$. This yields
  \begin{align*}
    \tau_{\mathrm{hls}}' \lesssim \tau_{\mathrm{alg}} + \frac{\tau_{\mathrm{elim}}L_{\phi^\star}^2}{\lambda^\star(\phi^\star)\Delta}.
  \end{align*}
 \end{proof}

 \begin{lemma}\label{lem:scale-tau-glrt}
  The time $\tau_{\mathrm{glrt}}$ defined in Lemma \ref{lem:trigger-glrt} satisfies
  \begin{align*}
    \tau_{\mathrm{glrt}} \lesssim \tau_{\mathrm{alg}} + \frac{L_{\phi^\star}^4\log(|\Phi|/\delta)}{\lambda^\star(\phi^\star)^2} + \frac{\tau_{\mathrm{elim}}L_{\phi^\star}^2}{\lambda^\star(\phi^\star)\Delta} + \frac{L_{\phi^\star}^2 d_{\phi^\star}\log(|\Phi|/\delta)}{\lambda^\star(\phi^\star)\Delta^2},
  \end{align*}
  where $ \tau_{\mathrm{alg}}$ is defined in Lemma \ref{lem:scale-tau-hls}.
\end{lemma}
\begin{proof}
  As we did in the proof of Lemma \ref{lem:scale-tau-hls}, we can bound $\tau_{\mathrm{glrt}} \lesssim \tau_{\mathrm{glrt}}' + \tau_{\mathrm{glrt}}'' + \tau_{\mathrm{glrt}}'''$, where
  \begin{align*}
    \tau_{\mathrm{glrt}}' &:= \min_{t \in \bN} \left\{ t \mid t \geq \frac{L_{\phi^\star}^2\beta_{t,\delta/|\Phi|}(\phi^\star)^2}{\lambda^\star(\phi^\star)\Delta^2}  \right\},
    \\ \tau_{\mathrm{glrt}}'' &:= \min_{t \in \bN} \left\{ t \mid t \geq \frac{L_{\phi^\star}^2}{\lambda^\star(\phi^\star)} 
    g_t(\Phi,\Delta, \delta)  \right\},
    \\ \tau_{\mathrm{glrt}}''' &:= \min_{t \in \bN} \left\{ t \mid t \geq \frac{L_{\phi^\star}^2}{\lambda^\star(\phi^\star)} 
    \sqrt{t\log(4d_{\phi^\star} |\Phi|t/\delta)}  \right\}.
  \end{align*}
  As before, we have
  \begin{align*}
    \tau_{\mathrm{glrt}}'' \lesssim \tau_{\mathrm{alg}} + \frac{\tau_{\mathrm{elim}}L_{\phi^\star}^2}{\lambda^\star(\phi^\star)\Delta} \quad \text{and} \quad \tau_{\mathrm{glrt}}''' \lesssim \frac{L_{\phi^\star}^4\log(|\Phi|/\delta)}{\lambda^\star(\phi^\star)^2}.
  \end{align*}
  Regarding the first term, since $\beta_{t,\delta/|\Phi|}(\phi^\star)^2$ is of order $d_{\phi^\star}\log(t|\Phi|/\delta)$, by Lemma \ref{lem:ineq-log-sqrt},
  \begin{align*}
    \tau_{\mathrm{glrt}}' \lesssim \frac{L_{\phi^\star}^2 d_{\phi^\star}\log(|\Phi|/\delta)}{\lambda^\star(\phi^\star)\Delta^2}.
  \end{align*}
 \end{proof}

 \subsection{Proof of the main theorems}

 The proof of Theorem \ref{th:icmlams.regret.lambda_min.hls} easily follows by using Lemma \ref{lem:scale-tau-elim}, \ref{lem:scale-tau-hls}, \ref{lem:scale-tau-glrt} to simplify the expressions of the constant times in Theorem \ref{th:regret-strong-missp-hls}.

 Corollary \ref{cor:single-repr} can be proved analogously to Theorem \ref{th:regret-strong-missp-nohls} and \ref{th:regret-strong-missp-hls} while noting that, since $|\Phi|=1$, the base algorithm is never reset (hence we can simply use confidence $\delta$ and remove the extra $\log_2(T)$ term) and $\tau_{\mathrm{elim}} = \tau_{\mathrm{hls}} = 0$.

 Corollary \ref{th:icmlams.regret.lambda_min.nohls} is simply a restatement of Theorem \ref{th:regret-strong-missp-nohls}.

\section{Variants of \algo}\label{app:algo.variations}

\subsection{\algo: alternative losses}

\subsubsection{Obtaining best-in-class regret}

Suppose that the upper bound $\wb{R}_{\mathfrak{A}}(T, \phi, \delta)$ to the regret of the base algorithm contains only known quantities (e.g., it could be a worst-case regret bound). Moreover, assume that the minimum gap $\Delta$ is known. This is only to simplify the notation in what follows, as we shall see at the end of this section that $\Delta$ can be estimated with a decreasing schedule without significantly altering the results. We consider the following alternative representation selection loss. For $j\in\bN$,
\begin{align*}
    \cL_{\mathrm{bic},{t_j}}(\phi) = \wb{R}_{\mathfrak{A}}(t_j, \phi, \delta_j/|\Phi|) - \left[ \frac{\lambda_{\min}(V_{t_j}(\phi) - \lambda I_{d_\phi})}{L_\phi^2} - g_{t_j}(\Phi,\Delta, \delta)  - 8\sqrt{t_j\log(4d_\phi |\Phi|t_j/\delta)} \right]_+
\end{align*}
where $[x]_+ := \max(x, 0)$. We show that with this selection loss we can achieve the best-in-class regret bound when no HLS realizable representation exists while preserving the constant-regret result when such a representation does exist.

\begin{theorem}\label{th:regret-nohls-bic}
    Suppose that $\Phi^\star$ does not contain any HLS representation. Under event $\cE$ (i.e., with probability at least $1-4\delta$), for any $T\in\mathbb{N}$, the regret of Algorithm \ref{alg:replearnin.icml.asm} with $\gamma=2$ and loss $\cL_{\mathrm{bic},t}(\phi)$ can be bounded as
    \begin{align*}
      R_T \leq 2\tau_{\mathrm{elim}} + \min_{\phi\in\Phi^\star} \wb{R}_{\mathfrak{A}}(T, \phi, \delta_{\log_2(T)}/|\Phi|) \log_2(T),
    \end{align*}
    where $\tau_{\mathrm{elim}}$ is defined in Lemma \ref{lem:elim-strong-missp}
  \end{theorem}
  \begin{proof}
    Using exactly the same steps as in the proof of Theorem \ref{th:regret-strong-missp-nohls}, we have
    \begin{align*}
        R_T \leq 2\tau_{\mathrm{elim}} + \sum_{j=\bar{j}}^{\lfloor\log_2(T)\rfloor} \wb{R}_{\mathfrak{A}}(N_j(t_{j+1} \wedge T), \phi_{t_j}, \delta_j/|\Phi|),
    \end{align*}
    where we recall that $\bar{j}$ is such that $\tau_{\mathrm{elim}} = 2^{\bar{j}}$. Note that $N_j(t_{j+1} \wedge T) \leq t_{j+1} - t_j = t_j$. Moreover, under $\cE$, for all $j\geq\bar{j}$, we have that $\Phi_{t_j} = \Phi^\star$ and, since $\Phi^\star$ does not contain any HLS representation,
    \begin{align*}
        \frac{\lambda_{\min}(V_{t_j}(\phi) - \lambda I_{d_\phi})}{L_\phi^2} - g_{t_j}(\Phi,\Delta, \delta)  - 8\sqrt{t_j\log(4d_\phi |\Phi|t_j/\delta)} \leq 0.
    \end{align*}
    This implies that $\cL_{\mathrm{bic},{t_j}}(\phi) = \wb{R}_{\mathfrak{A}}(t_j, \phi, \delta_j/|\Phi|)$ in such phases. Therefore,
    \begin{align*}
        R_T &\leq 2\tau_{\mathrm{elim}} + \sum_{j=\bar{j}}^{\lfloor\log_2(T)\rfloor} \wb{R}_{\mathfrak{A}}(t_j, \phi_{t_j}, \delta_j/|\Phi|)
         = 2\tau_{\mathrm{elim}} + \sum_{j=\bar{j}}^{\lfloor\log_2(T)\rfloor} \cL_{\mathrm{bic},{t_j}}(\phi_{t_j})
         \\ &= 2\tau_{\mathrm{elim}} + \sum_{j=\bar{j}}^{\lfloor\log_2(T)\rfloor} \min_{\phi\in\Phi^\star} \cL_{\mathrm{bic},{t_j}}(\phi)
         = 2\tau_{\mathrm{elim}} + \sum_{j=\bar{j}}^{\lfloor\log_2(T)\rfloor} \min_{\phi\in\Phi^\star} \wb{R}_{\mathfrak{A}}(t_j, \phi, \delta_j/|\Phi|).
    \end{align*}
    The proof is concluded by noting that $\delta_j \geq \delta_{\log_2(T)}$ and $t_j \leq T$, so that, by the properties $\wb{R}_{\mathfrak{A}}$, $\sum_{j=\bar{j}}^{\lfloor\log_2(T)\rfloor} \min_{\phi\in\Phi^\star} \wb{R}_{\mathfrak{A}}(t_j, \phi, \delta_j/|\Phi|) \leq \min_{\phi\in\Phi^\star} \wb{R}_{\mathfrak{A}}(T, \phi, \delta_{\log_2(T)}/|\Phi|)\log_2(T)$.
  \end{proof}

Let us now derive the constant regret bound when a HLS representation exists. Note that, since we only changed the selection loss, Theorem \ref{th:regret-strong-missp-nohls} and Lemma \ref{lem:suboptimal-pulls-strong-missp} still hold. The only change is in the time $\tau_{\mathrm{hls}}$ at which the HLS representation is selected. Theorem \ref{th:regret-strong-missp-hls} also continues to hold with the following redefinition of such time.

\begin{lemma}[Selecting the HLS representation with BIC loss]\label{lem:select-hls-bic}
    Suppose Algorithm \ref{alg:replearnin.icml.asm} is run with $\gamma=2$ and $\cL_t(\phi) = \cL_{\mathrm{bic},t}(\phi)$. Suppose that there exists a unique $\phi^\star\in\Phi^\star$ such that $\phi^\star$ is HLS. Then, under event $\cE$ (i.e., with probability at least $1-4\delta$), $\phi_t = \phi^\star$ for all $t \geq \tau_{\mathrm{hls}} \vee \tau_{\mathrm{elim}}$, where
    \begin{align*}
      \tau_{\mathrm{hls}} := \min_{t \in \bN} \Bigg\{ t \mid \exists j\in\bN_{>0} : t=2^j,  t > \frac{L_{\phi^\star}^2}{\lambda^\star(\phi^\star)}&\Bigg( \wb{R}_{\mathfrak{A}}(t, \phi^\star, \delta_{\log_2(t)}/|\Phi|)
      \\ & + g_t(\Phi,\Delta, \delta) + 8\sqrt{t\log\frac{4 |\Phi|t \max_{\phi\in\Phi^\star} d_{\phi}}{\delta}} \Bigg) \Bigg\}.
    \end{align*}
  \end{lemma}
  \begin{proof}
    Take any time $t_j\geq \tau_{\mathrm{elim}}$. By Lemma \ref{lem:elim-strong-missp}, we have $\Phi_{t_j}=\Phi^\star$ and, thus, $\phi^\star$ is the only active HLS representation. Using the good event $\cE$, we can easily see that $\cL_{\mathrm{bic},{t_j}}(\phi) \leq \wb{R}_{\mathfrak{A}}(t_j, \phi, \delta_j/|\Phi|)$ for all $\phi\in\Phi^\star, \phi\neq\phi^\star$. Moreover,
    \begin{align*}
        \frac{\lambda_{\min}(V_{t_j}(\phi) - \lambda I_{d_\phi})}{L_\phi^2} \geq t_j\frac{\lambda^\star(\phi^\star)}{L_{\phi^\star}^2} - g_{t_j}(\Phi,\Delta, \delta) - 8\sqrt{t_j\log(4d_{\phi^\star} |\Phi|t_j/\delta)}
    \end{align*}
    and, thus,
    \begin{align*}
        \cL_{\mathrm{bic},{t_j}}(\phi^\star) \geq \wb{R}_{\mathfrak{A}}(t_j, \phi^\star, \delta_j/|\Phi|) - \left[ t_j\frac{\lambda^\star(\phi^\star)}{L_{\phi^\star}^2} - 2g_{t_j}(\Phi,\Delta, \delta)  - 16\sqrt{t_j\log \frac{4|\Phi|t_j \max_{\phi\in\Phi^\star}d_\phi}{\delta}} \right]_+
    \end{align*}
    Therefore, a sufficient condition for selecting $\phi^\star$ is
    \begin{align*}  
          t_j\frac{\lambda^\star(\phi^\star)}{L_{\phi^\star}^2} - 2g_{t_j}(\Phi,\Delta, \delta)  - 16\sqrt{t_j\log \frac{4|\Phi|t_j \max_{\phi\in\Phi^\star}d_\phi}{\delta}} > \wb{R}_{\mathfrak{A}}(t_j, \phi^\star, \delta_j/|\Phi|).
    \end{align*}
    The proof is concluded by rearringing this inequality.
  \end{proof}

  \paragraph{Dealing with unknown $\Delta$}

  If the minimum gap $\Delta$ is unknown, it can be easily guessed by a decreasing schedule $(1/t^\ell)_{t\geq 1}$. Then, we can replace the unknown term $g_{t_j}(\Phi, \Delta, \delta)$ in $\cL_{\mathrm{bic},{t_j}}(\phi)$ with $g_{t_j}(\Phi,1/t_j^\ell, \delta)$. Since
  \begin{align*}
    g_{t_j}(\Phi,1/t_j^\ell, \delta) = 2t_j^\ell\tau_{\mathrm{elim}} + t_j^\ell \max_{\phi\in\Phi^\star} \wb{R}_{\mathfrak{A}}(t_j, \phi, \delta_{\log_2(t_j)}/|\Phi|) \log_2(t_j),
  \end{align*}
  we only need $t_j^\ell \max_{\phi\in\Phi^\star} \wb{R}_{\mathfrak{A}}(t_j, \phi, \delta_{\log_2(t_j)}/|\Phi|)$ to be sub-linear to derive our constant-regret result. For instance, if $\wb{R}_{\mathfrak{A}}(t_j, \phi, \delta_{\log_2(t_j)}/|\Phi|)$ is an $\widetilde{O}(\sqrt{t_j})$ regret bound, we can set $\ell = 1/4$. Then, the proofs of the two results above are the same except that we add a linear regret term $1/\Delta^{1/\ell}$ for the first time steps where $1/t^\ell > \Delta$.

\subsubsection{Weak-\hls Loss}\label{app:weak}
In Section~\ref{sec:exp.and.practical.algo}, we introduced an alternative loss $\cL_{\mathrm{weak},t}(\phi) = -\min_{s\leq t} \big\{\phi(x_s,a_s)^\transp (V_t(\phi) - \lambda I_{d_{\phi}}) \phi(x_s,a_s) / L_{\phi}^2 \big\}$, which is motivated by the notion of ``weak-\hls'' representations from~\citep{PapiniTRLP21hlscontextual} and appears to perform well in practice. In this section, \textbf{we will consider a slight variant} 
\[
  \overline{\cL}_{\mathrm{weak},t}(\phi) = -\min_{s\leq t} \big\{\phi(x_s,a_s)^\transp (V_t(\phi) - \lambda I_{d_{\phi}}) \phi(x_s,a_s) / \norm{\phi(x_s,a_s)}^2 \big\}
\] 
where the features are normalized to have norm equal to one. The loss used in the experiments is $\cL_{\mathrm{weak},t}$ as defined in the main text.

We will show that $\overline{\cL}_{\mathrm{weak},t}$ does indeed select weak-\hls representations. We will assume throughout this section that both $\mathcal{X}$ and $\mathcal{A}$ are finite and $\mathrm{supp}(\rho)=\mathcal{X}$. Let us first recall the definition of weak \hls. We abbreviate $\mathrm{span}(\phi) =  \mathrm{span}\{\phi(x,a)\mid x\in\mathcal{X},a\in\mathcal{A}\}$ and $\mathrm{span}(\phi^\star) = \mathrm{span}\{\phi(x,a^\star_x)\mid x\in\mathcal{X}\}$.

\begin{definition}[Weak-\hls Representation]
A representation $\phi$ is weak-\hls if $\mathrm{span}(\phi^\star) = \mathrm{span}(\phi)$.
\end{definition}

The following characterization of the weak \hls property will be useful. We abbreviate $M_\phi^\star = \EV_{x\sim\rho}\left[\phi(x,a_x^\star)\phi(x,a_x^\star)^\transp\right]$.
\begin{lemma}\label{lem:weak_hls_char}
    A representation $\phi$ is weak-\hls if and only if
    \begin{equation}\label{eq:weak_hls_char}
     \min_{x\in\mathcal{X},a\in\mathcal{A}}\left\{\frac{\phi(x,a)^\transp M_\phi^\star\phi(x,a)}{\norm{\phi(x,a)}^2}\right\} > 0.
    \end{equation}
\end{lemma}
\begin{proof}
  We denote by $\mathrm{Im}(A)$ the column space of a symmetric matrix $A$, and by $\mathrm{ker}(A)$ its kernel. Under our assumption that $\rho$ is full-support, it is easy to see that $\mathrm{span}(\phi^\star)=\mathrm{Im}(M_\phi^\star)$.
  If $\phi$ is weak-\hls, then
  \begin{align}
      \min_{x\in\mathcal{X},a\in\mathcal{A}}\left\{\frac{\phi(x,a)^\transp M_\phi^\star\phi(x,a)}{\norm{\phi(x,a)}^2}\right\} &\ge  \min_{v\in\mathrm{span}(\phi),\norm{v}=1}\left\{v^\transp M_\phi^\star v\right\} \\
      &=\min_{v\in\mathrm{span}(\phi^\star),\norm{v}=1}\left\{v^\transp M_\phi^\star v\right\} \\
      &=\min_{v\in\mathrm{Im}(M_\phi^\star),\norm{v}=1}\left\{v^\transp M_\phi^\star v\right\},
  \end{align}
  and the latter is positive since it is the definition of the minimum \emph{nonzero} eigenvalue of a positive semidefinite matrix. 
  
  Now assume~\eqref{eq:weak_hls_char} holds. We just need to show $\mathrm{span}(\phi)\subseteq\mathrm{span}(\phi^\star)$, since the other inclusion is trivial.
  By diagonalization, it is easy to show that the solution space of $\phi(x,a)^\transp M_\phi^\star\phi(x,a)=0$ is $\ker(M_\phi^\star)$. Hence,~\eqref{eq:weak_hls_char} implies $\phi(x,a)\in\mathrm{Im}(M_\phi^\star)=\mathrm{span}(\phi^\star)$ for all $x\in\mathcal{X}$ and $a\in\mathcal{A}$. In turn, this implies $\mathrm{span}(\phi)\subseteq\mathrm{span}(\phi^\star)$, concluding the proof. 
  
\end{proof}

We can now show that our alternative loss does indeed select weak-\hls representations. 

\begin{lemma}
    Assume $\rho_{\min}>0$ is the minimum probability $\rho$ assigns to any context, and $K=|\mathcal{A}|$.
    For any representation $\phi$, $\epsilon$-greedy with $\epsilon_t=t^{-1/3}$ guarantees that the following hold simultaneously with probability $1-5\delta$ for all $t\ge\left(\frac{K}{\rho_{\min}}\log\frac{1}{\delta}\right)^{3/2}$:
    \begin{align}
    &\overline{\cL}_{\mathrm{weak},t}(\phi) \le - t \min_{x\in\mathcal{X},a\in\mathcal{A}}\left\{\frac{\phi(x,a)^\transp M_\phi^\star\phi(x,a)}{\norm{\phi(x,a)}^2}\right\} + o(t) &&\text{and}\label{eq:weak_lower}\\
    &\overline{\cL}_{\mathrm{weak},t}(\phi) \ge - t \min_{x\in\mathcal{X},a\in\mathcal{A}}\left\{\frac{\phi(x,a)^\transp M_\phi^\star\phi(x,a)}{\norm{\phi(x,a)}^2}\right\} - o(t)\label{eq:weak_upper}
    \end{align}
\end{lemma}
\begin{proof}
  From Lemma~\ref{lem:good-event-proba}, the good event $\mathcal{E}$ holds with probability $1-4\delta$. By $\mathcal{E}_2$, since Loewner ordering induces the same ordering on all quadratic forms:
  \begin{align}
      \overline{\cL}_{\mathrm{weak},t}(\phi) 
      &= -\min_{s\leq t} \left\{\frac{\phi(x_s,a_s)^\transp (V_t(\phi) - \lambda I_{d_{\phi}}) \phi(x_s,a_s)}{\norm{\phi(x_s,a_s)}^2}\right\} \\
      &\le - \min_{x\in\mathcal{X},a\in\mathcal{A}}\left\{\frac{\phi(x,a)^\transp (V_t(\phi) - \lambda I_{d_{\phi}})\phi(x,a)}{\norm{\phi(x,a)}^2}\right\}\\
      &\le - t\min_{x\in\mathcal{X},a\in\mathcal{A}}\left\{\frac{\phi(x,a)^\transp M_\phi^\star\phi(x,a)}{\norm{\phi(x,a)}^2}\right\} + o(t),
  \end{align}
  where we have also used Lemma~\ref{lem:suboptimal-pulls-strong-missp} to bound the number of suboptimal pulls. 
  Similarly, by $\mathcal{E}_3$:
    \begin{align}
      \overline{\cL}_{\mathrm{weak},t}(\phi) 
      &\ge -\min_{s\leq t} \left\{\frac{\phi(x_s,a_s)^\transp M_\phi^\star \phi(x_s,a_s)}{\norm{\phi(x_s,a_s)}^2}\right\} - o(t).
  \end{align}
  Let $(\overline{x},\overline{a})\in\arg\min_{x\in\mathcal{X},a\in\mathcal{A}}\left\{\frac{\phi(x,a)^\transp M_\phi^\star\phi(x,a)}{\norm{\phi(x,a)}^2}\right\}$. Under our assumption, $\epsilon$-greedy selects each context-action pair with probability at least $q=\rho_{\min}/(Kt^{1/3})$. After $t$ rounds, the probability that it has not yet selected $(\overline{x},\overline{a})$ is at most $(1-q)^t$. A simple calculation shows that, by $t\ge \left(\frac{K}{\rho_{\min}}\log\frac{1}{\delta}\right)^{3/2}$, the algorithm has selected $(\overline{x},\overline{a})$ at least once with probability $1-\delta$, hence
  \begin{equation}
      \min_{s\leq t} \left\{\frac{\phi(x_s,a_s)^\transp M_\phi^\star \phi(x_s,a_s)}{\norm{\phi(x_s,a_s)}^2}\right\} = \frac{\phi(\overline{x},\overline{a})^\transp M_\phi^\star\phi(\overline{x},\overline{a})}{\norm{\phi(\overline{x},\overline{a})}^2} = \min_{x\in\mathcal{X},a\in\mathcal{A}}\left\{\frac{\phi(x,a)^\transp M_\phi^\star\phi(x,a)}{\norm{\phi(x,a)}^2}\right\}.
  \end{equation}
  A union bound concludes the proof with an overall probability of $1-5\delta$.
\end{proof}

Now let $\phi_1$ be a weak-\hls representation. Lemma~\ref{lem:weak_hls_char} and Equation~\ref{eq:weak_lower} show that, with high probability, $\overline{\cL}_{\mathrm{weak},t}(\phi_1) \le -t\wt{\lambda} + o(t)$ for some constant $\wt{\lambda}>0$. From the proof of Lemma~\ref{lem:weak_hls_char} we can deduce that this $\wt{\lambda}$ is the minimum nonzero eigenvalue\footnote{Of course, an \hls representation is also weak-\hls, and $\wt{\lambda}=\lambda
^\star>0$. The converse is not true. Note also that the minimum nonzero eigenvalue $\wt{\lambda}$ is well-defined and positive for \emph{all} representations, but it can only play the role of $\lambda^\star$ when the representation is weak-\hls.} of $M_{\phi_1}^\star$. On the other hand, consider a representation $\phi_2$ that does not have the weak-\hls property. The other direction of Lemma~\ref{lem:weak_hls_char} and Equation~\ref{eq:weak_upper} show that, with high probability, $\overline{\cL}_{\mathrm{weak},t}(\phi_1) \ge -o(t)$. Hence, the loss for the weak-\hls representations decreases (towards $-\infty$) much faster than representations that do not have this property. This justifies the use of $\overline{\cL}$ as a loss in the \algo algorithm, when $\epsilon$-greedy is used as a base algorithm. A more sophisticated argument allows to extend this result to any no-regret algorithm, by using the fact that they eventually sample all (finite) state-action pairs to ensure sufficient exploration.

When $\mathrm{span}(\phi) = \Reals^d$, there is no distinction between \hls and weak-\hls. Moreover,~\cite{PapiniTRLP21hlscontextual} show that weak-\hls is enough for \linucb to achieve constant regret. We could generalize the constant-regret result from this paper to weak-\hls in a similar fashion.

\paragraph{Empirical evaluation.} We empirically compare $\overline{\cL}_{\mathrm{weak},t}$ and ${\cL}_{\mathrm{weak},t}$ on the same set of experiments reported in the main article. Fig.~\ref{fig:vardim.appendix} shows that the loss ${\cL}_{\mathrm{weak},t}$ outperforms the theoretically grounded $\overline{\cL}_{\mathrm{weak},t}$ loss. We leave as open question whether the loss $\cL_{\mathrm{weak},t}$ is theoretically sound or not.

\begin{figure}[h]
  \hspace{-0.3in}\includegraphics[width=1.03\textwidth]{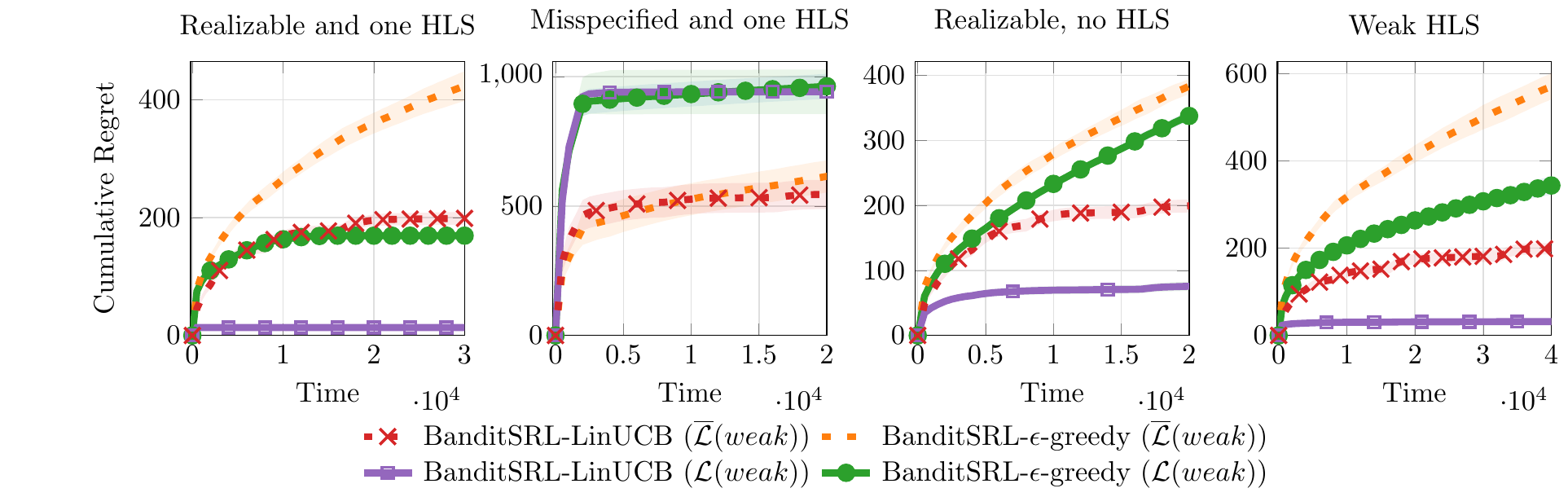}
  \vspace{-.1in}
  \caption{\small
              Varying dimension experiment with all realizable representations (left), misspecified representations (center-left), realizable non-\hls{} representations (center-right) and weak-\hls{} (right). Experiments are averaged over $40$ repetitions as in the main paper.
          }
  \label{fig:vardim.appendix}
\end{figure}

\subsection{\deepalgo: representation learning through neural networks}

\begin{algorithm}[t]
    \caption{\deepalgo}\label{alg:deep.algo}
    \begin{algorithmic}[1]
        \STATE \textbf{Input:} Neural network $f$ with last layer $\phi : \cX \times \cA \to \mathbb{R}^d$, no-regret algorithm $\mathfrak{A}$, confidence $\delta \in (0,1)$, update schedule $\gamma > 1$, regularizer $\lambda > 0$ and $c_{\mathrm{reg}} >0$
        \STATE Initialize $j=0$, $f_j$ arbitrarily, $b_{t}(\phi_j) = 0$, $V_0(\phi_j) = \lambda I$
        \FOR{$t = 1, \ldots$}
            \STATE Observe context $x_t$
            \IF{$\mathrm{GLR_{t-1}(x_t;\phi_j)} > \beta_{t-1,\delta/|\Phi|}(\phi_{j})$}
            \STATE Play $a_t = \argmax_{a\in\cA} \big\{ \phi_{j}(x_t,a)^\transp \theta_{\phi_j,t-1} \big\}$ and observe reward $y_t$
            \STATE $\cD_{\mathrm{glrt},t} = \cD_{\mathrm{glrt},t-1} \cup \{x_t,a_t,y_t\}$
            \ELSE
            \STATE Play $a_t = \mathfrak{A}_t(x_t;\phi_{j},\delta)$, observe reward $y_t$, and feed it into $\mathfrak{A}$ 
            \STATE $\cD_{\mathfrak{A},t} = \cD_{\mathfrak{A},t-1} \cup \{x_t,a_t,y_t\}$
            \ENDIF
            \STATE Let $\cD_t = \cD_{\mathfrak{A},t} \cup \cD_{\mathrm{glrt},t}$
            \STATE Compute $V_t(\phi_j) = V_t(\phi_j) + \phi_j(x_t,a_t)\phi(x_t,a_t)^\transp$, $b_{t}(\phi_j) = b_t(\phi_j) + \phi_j(x_t,a_t) y_t$ and $\theta_{\phi_j,t} = V_t(\phi_j)^{-1} b_t(\phi_j)$
            \IF{$t = \lceil \gamma t_j \rceil$}
                \STATE Set $j = j +1$ and $t_j = t$
                \STATE Compute $\phi_{j} = \argmin_{\phi}\min_f \big\{ \cL_t(\phi) + c_{\mathrm{reg}} \wb{E}_t(f) \big\}$ (see Eq.~\ref{eq:opt.unconstrained.appendix}) and reset $\mathfrak{A}$ \label{line:deep.reglos}
                \STATE Recompute least-square on the linear embedding $\phi_j$ using all samples
                \begin{align*}
                    V_{t}(\phi_j)   = \lambda I + \sum_{x,a,y \in \cD_{t}} \phi_j(x,a)\phi_j(x,a)^\transp,
                    \quad
                    b_t(\phi_j) = \sum_{x,a,y \in \cD_{t}} \phi_j(x,a) y
                \end{align*}
                and  $\theta_{\phi_j,t} = V_t(\phi_j)^{-1} b_t(\phi_j)$
            \ENDIF
        \ENDFOR
    \end{algorithmic}
\end{algorithm}

We recall that we consider a representation space $\Phi$ defined by the last layer of a Neural Network (NN). We denote by $\phi : \cX \times \cA \to \mathbb{R}^d$ the last layer and by $f(x,a) = \phi(x,a)^\transp \omega$ the full NN, where $\omega$ are the last-layer weights. We report the pseudo code of \deepalgo{} in Alg.~\ref{alg:deep.algo}. The structure of \deepalgo{} is identical to the one of \algo{}, showing the generality and flexibility of the theoretical algorithm. 

The GLRT is the same reported in Eq.~\ref{eq:glrt.main.paper}. It leverages the current representation $\phi_j$ learnt by the NN and the regularized least squares parameters $V_{t}(\phi_j)$ and $\theta_{\phi_j,t}$. Note that, similarly to~\citep{xu2020neuralcb}, we keep a separate estimate of the weights of the linear fitting ($\theta$ vs. $\omega$). While the NN weights $\omega$ are learnt through the regularization loss (line~\ref{line:deep.reglos} in Alg.~\ref{alg:deep.algo}), we compute $\theta_{\phi_j,t} = \argmin_{\theta} \Big\{\frac{1}{t} \sum_{k=1}^t (\phi_{j_t}(x_t,a_t)^\transp \theta - y_t)^2 + \lambda \|\theta\|_2^2\Big\}$ by RLS at each time $t$. This allows us to compute the best linear fit at each time $t$ using efficient incremental updates (e.g., we can use Sherman-Morrison formula for computing directly $V_t(\phi_j)^{-1}$) and avoid to retrain the network after observing a new sample $(x_t,a_t,y_t)$.
An alternative approach is to train only the NN weights $\omega$ (i.e., keeping fix the representation $\phi$) by stochastic gradient at each step, leading to an approximation of the RLS solution.

The phases scheme of \algo pairs very well with NN since it allows to perform the computationally costly operation of full NN training only $\log_\gamma(T)$ times. The NN is trained through a regression problem with an auxiliary representation loss promoting \hls-like representations. At the beginning of phase $j$, we solve the following problem
\begin{equation}\label{eq:opt.unconstrained.appendix}
\begin{aligned}
    f_,\phi_j 
    &= \argmin_{\phi,f} \left\{ \cL_t(\phi) + c_{\mathrm{reg}}\, \wb{E}_{t}(f) \right\}\\
    &= \argmin_{\phi, \omega} \left\{ \cL_t(\phi) + \frac{c_{\mathrm{reg}}}{|\cD_{\mathfrak{A},t_j}|} \sum_{(x,a,y) \in \cD_{\mathfrak{A},t_j}} \Big( \underbrace{\phi(x,a)^\transp \omega}_{:=f(x,a)} - y \Big)^2 \right\}\\
    &= \argmin_{\phi, \omega} \left\{ c_{\mathrm{reg},\cL}\,\cL_t(\phi) + \frac{1}{|\cD_{\mathfrak{A},t_j}|} \sum_{(x,a,y) \in \cD_{\mathfrak{A},t_j}} \Big( \underbrace{\phi(x,a)^\transp \omega}_{:=f(x,a)} - y \Big)^2 \right\}.
\end{aligned}
\end{equation}
for some $c_{\mathrm{reg},\cL}, c_{\mathrm{reg}}>0$.\footnote{In the experiments, we use scaling of the representation loss instead of MSE.}
We recall that we compute the MSE regression loss using the explorative samples $\cD_{\mathfrak{A},t_j}$ collected when playing the base algorithm $\mathfrak{A}$. As mentioned in the main paper, we use this separation to prevent the NN $f(x,a)$ to focus only on predicting optimal rewards when the the empirical distribution of the samples collapses towards the optimal actions (i.e., catastrophic forgetting).
On the other hand, we can use all the samples $\cD_t = \cD_{\mathfrak{A},t} \cup \cD_{\mathrm{glrt},t}$ to compute the loss, where we want to leverage the bias/shift of the empirical distribution towards optimal actions to compute the empirical design matrix $V_t(\phi)$. 

Concerning the loss $\cL_t$, we leverage the same concepts used in \algo but we slightly modify them to make it more amenable for NN training. To optimize $\cL_{\mathrm{eig},t}(\phi)$ we leverage the fact that $\lambda_{\min}(M) = \min_z R(M,z)$, where $R(M,z) = \frac{z^\transp M z}{z^\transp z}$ is the Rayleigh quotient. We thus threat $z$ as a parameter and optimize it by gradient descent, leading to 
\begin{align}\label{eq:ray.loss.app}
    \cL_{\mathrm{ray},t}(\phi) = \frac{-1}{|\cD_{t_j}| }\min_{z \in \mathbb{R}^d} 
    \frac{z^\transp}{\|z\|_2} \left( \lambda I_d +
        \sum_{(x,a,y) \in \cD_t} \frac{\phi(x,a)\phi^\transp(x,a)}{\|\phi(x,a)\|_2^2}
        \right) \frac{z}{\|z\|_2}
\end{align}
We normalize the empirical design matrix to prevent features norms to grow unbounded. On the other hand, since the idea behind $\cL_{\mathrm{weak},t}(\phi)$ is to force the optimal features to span all the features we use a mixed approach to compute the loss. We leverage all the samples to compute the matrix $V_{t}(\phi)$, while we use the explorative samples $\cD_{\mathfrak{A},t}$ to compute the quadratic form in $V_t$ and avoid it collapses to evaluate only optimal actions. Then,
\begin{align}\label{eq:weak.loss.app}
    \cL_{\mathrm{weak},t}(\phi) = \frac{-1}{|\cD_{t_j}| } \min_{(\wb{x},\wb{a},\wb{y}) \in \cD_{\mathfrak{A},t}}   & \texttt{stop-grad} \left( \frac{\phi(\wb{x},\wb{a})^\transp}{\|\phi(\wb{x},\wb{a})\|_2} \right) \left( \lambda I_d +
    \sum_{(x,a,y) \in \cD_t} \frac{\phi(x,a)\phi^\transp(x,a)}{\|\phi(x,a)\|_2^2}
    \right) \nonumber \\ 
    & \texttt{stop-grad} \left( \frac{\phi(\wb{x},\wb{a})}{\|\phi(\wb{x},\wb{a})\|_2} \right) 
\end{align}
Where we apply the $\texttt{stop-grad}$ operator on the outer features to only backpropagate gradient through the covariance matrix.
We notice that the loss $\cL_{\mathrm{weak},t}$ resemble the $\cL_{\mathrm{eig},t}$ loss with the difference of being evaluated on the observed features rather than all the possible vectors in $\mathbb{R}^d$. We can optimize Eq.~\ref{eq:opt.unconstrained.appendix} by stochastic gradient descent using mini-batches but \emph{we don't compute the gradient w.r.t.\ the outer features} $\phi(\wb{x}, \wb{a})$.

Finally, nothing changes in term of base algorithm $\mathfrak{A}$ that now receives in input the trained NN $f_j$ that can be used to extract the representation $\phi_j$ (that is fix through the entire phase). In the experiments, we use the standard \linucb and $\epsilon$-greedy algorithms to perform exploration given the representation $\phi_j$.

\section{Experiments}\label{app:experiments}
In this section, we report additional information about the experiments. We recall that in all the experiments, we do a warm start of the base algorithm $\mathfrak{A}$ every time the representation changes using all the samples $\cD_t$.

\subsection{Linear Benchmarks}

\paragraph{Parameters.} In all the experiments, we consider all the theoretical parameters, e.g., $\gamma=2$, $\delta =0.01$ and $\lambda =1$. For $\epsilon$-greedy we use the schedule $\epsilon_t = t^{-1/3}$. For all the algorithms based on upper-confidence bound, we use the theoretical UCB value:
\begin{equation}\label{eq:linucb.ucb}
    \mathrm{UCB}_t(x,a,\phi) = \phi(x,a)^\transp \theta_{\phi,t-1} + C_{\mathrm{UCB},t} \|\phi(x,a)\|_{V_{t-1}^{-1}(\phi)}
\end{equation}
where $C_{\mathrm{UCB},t} = \alpha_{\mathrm{UCB}}\,\sigma\sqrt{2\ln\left(\frac{\det(V_{t-1}(\phi))^{1/2}\det(\lambda I_{d_{\phi}})^{-1/2}}{\delta}\right)} + \sqrt{\lambda} B_{\phi}$, $\alpha_{\mathrm{UCB}}=1$ and $\sigma$ is the standard deviation of the reward noise.

\paragraph{Varying dimension experiment.}
We providing additional information about the ``varying dimension'' problem introduced in~\citep{PapiniTRLP21hlscontextual}. This problem consists of six realizable representations with dimension from $2$ to $6$. Of the two representations of dimension $d = 6$, one is \hls. In addition seven misspecified representations are available: one considering half of the features of the \hls{} representation, one with a third of the same representation, and the five remaining are randomly generated representations with dimensions $3$, $9$, $12$, $12$, $18$. The reward noise is drawn from a zero-mean Gaussian distribution with standard deviation $\sigma=0.3$. All the results of the experiments can be found in the Sec.~\ref{sec:exp.and.practical.algo}.

\paragraph{Mixing Representations.}
To provide a fair and comprehensive analysis, we also report the performance of the algorithms when none of the representations is \hls{} but a combination of them is. We consider the same problem in~\citep{PapiniTRLP21hlscontextual}, where there are six realizable representations of the same dimension $d=6$, none of which is \hls{}, but a mixture of them is \hls. We set $\sigma=0.3$ for the reward noise. In this case, \leader{} outperforms \algo{} and achieves constant regret (see Fig.~\ref{fig:mixing}). While \leader{} is able to select a different representation for each context and mix them, \algo{} is only able to select a single representation for all the contexts and suffers sublinear regret. As mentioned before, this is both an advantage and drawback of \leader{} since it needs to solve an optimization problem over representations for each context.

\begin{figure}[t]
    \centering
    \includegraphics[width=.65\textwidth]{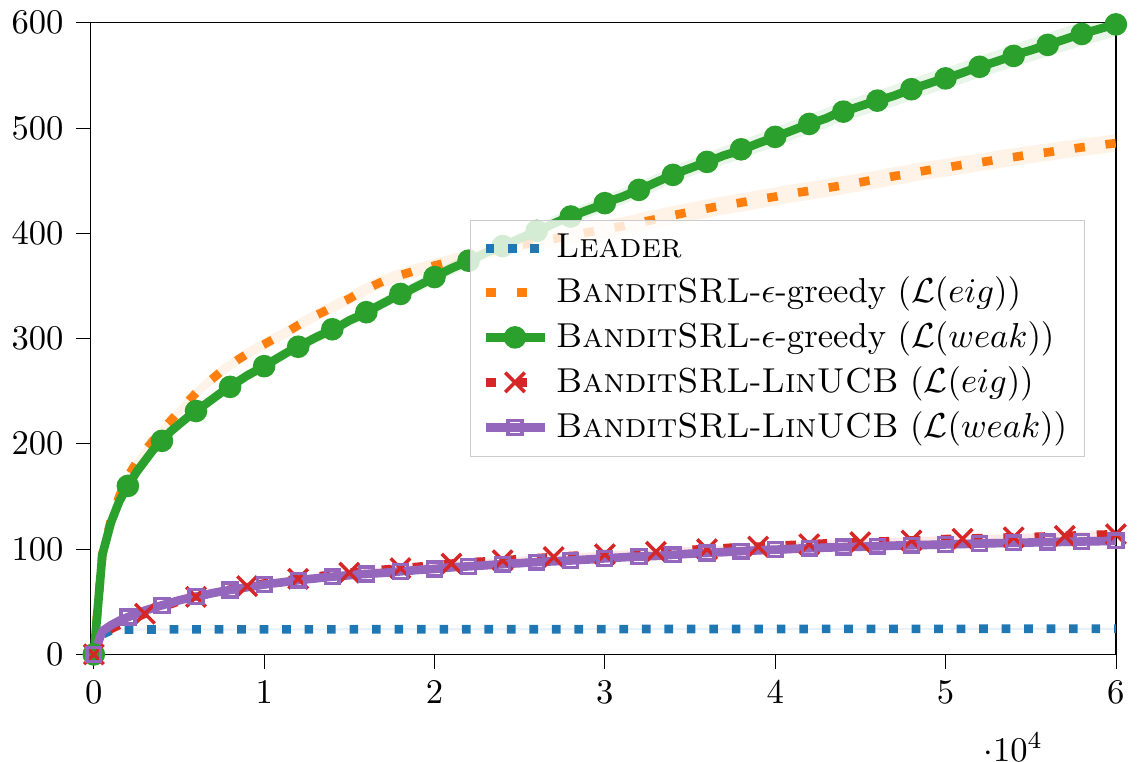}
\caption{Cumulative regret of the algorithms in the mixing representation experiment, averaged over $40$ repetitions.}
\label{fig:mixing}
\end{figure}

\subsection{Non-Linear Benchmarks}

\paragraph{Baselines.}
As baselines we consider \linucb and $\epsilon$-greedy with neural network and Random Fourier Features, the inverse gap weighting (IGW) strategy~\citep[e.g.,][]{Foster2020beyond,SimchiLevi2020falcon}, NeuralUCB~\citep{Zhou2020neural} and Neural-ThomposonSampling~\citep{RiquelmeTS18}.
All the algorithms are implemented using the same phased schema of \deepalgo{}. 

\emph{Neural-LinUCB} fits a model to minimize the MSE and compute the UCB on the last layer of the NN.

\emph{NeuralTS} performs randomized exploration on the last layer of the neural network, trained to minimize the MSE or our regularized problem. The exploration strategy is defined by the following two steps:
\begin{align*}
    &\wt{\theta} \sim \cN(\theta_{\phi,t-1}, C_{\mathrm{UCB},t}^2 V_{t-1}^{-1}(\phi)),\\
    &a_t = \argmax_a \phi(x_t,a)^\transp \wt{\theta}
\end{align*}

The \emph{IGW strategy}~\citep[e.g.,][]{Foster2020beyond,SimchiLevi2020falcon} trains the network to minimize the MSE and, at each time $t$, it plays an action $a_t$ sampled from the following distribution
\[
    p_t(a) = \begin{cases} 
        \frac{1}{A + \gamma_1 t^{\gamma_2} (\max_{a'} f_{j_t}(x,a') - f_{j_t}(x,a))} & \text{if } a \neq a^+_x := \argmax_{a'} f_{j_t}(x,a')\\
        1 - \sum_{a \neq a^+_x} p_t(a) & \text{otherwise}
    \end{cases}
\]
Note that the network is kept fix during a phase, i.e., we do not refit the linear part at each step. We also tested the variant of IGW where we refit the last layer at each time step (see Fig.~\ref{fig:ablation.igw}).
We did not use the theoretical scaling factor (encoded here by $\gamma_1$ and $\gamma_2$) since it would be prohibitively large.

\emph{NeuralUCB}~\citep{Zhou2020neural} is similar to Neural-\linucb but uses a bonus constructed with the whole gradient of the neural network. It thus selects the action that maximizes the following index
\begin{equation}
    \label{eq:neuralucb.ucb}
    \mathrm{UCB}_t^{\mathrm{NeuralUCB}}(x,a) = f_{j_t}(x,a) + \alpha_{\mathrm{UCB}}^{\mathrm{NeuralUCB}} \|\nabla f_{j_t}(x,a)\|_{V_{t-1}^{-1}}
\end{equation}
where $V_{t-1}^{-1}(f) = \sum_{k=1}^{t-1} \mathrm{diag}\Big( \nabla f_{j_k}(x_k,a_k) \nabla f_{j_k}(x_k,a_k)^\transp \Big)$. While we use the theoretical bonus factor for Neural-\linucb{} and \deepalgo{}, here we treat the bonus factor completely as an hyperparameter since the true factor is prohibitively large. This is a clear advantage we provide to NeuralUCB.

We further compare our algorithm against stochastic linear bandit algorithms (i.e., $\epsilon$-greedy and \linucb) using random Fourier features~\citep{RahimiR07}. We define $\phi(x,a) = W \,[x,a] + b$ with $[x,a] \in \mathbb{R}^m$ being the vector obtained from the concatenation of $x$ and $a$, $W \in \mathbb{R}^{d \times m}$ is random matrix and $b \in \mathbb{R}^d$ is a random vector.

\paragraph{\deepalgo.} We tested our algorithm with standard baseline methods: LinUCB, $\epsilon$-greedy and IGW. LinUCB uses the theoretical parameters (see~\eqref{eq:linucb.ucb}) while the parameters for the other methods are reported below. As explained, we fix the representation $\phi_j$ for the epoch but we refit the linear parameter at each step.

\paragraph{Parameters.}
In all the experiments, we used the following parameters:

\begin{center}
    \small
    \begin{tabular}{lc}
        \hline
        Name & Value\\
        \hline
        Phase schedule $\gamma$  & $1.2$\\
        Bonus parameter $\sigma$ & $0.2$ for wheel, $0.5$ for datasets\\
        Scale factor GLRT (i.e., $\alpha_{\mathrm{GLRT}} \beta_{t-1,\delta}(\phi)$) & $\{1,2,5,10,15\}$\\
        Scale factor UCB  (i.e., $\alpha_{\mathrm{UCB}}$ in Eq.~\ref{eq:linucb.ucb}) & $\{1,2\}$\\
        $\epsilon_t$ for $\epsilon$-greedy & $\{t^{-1/3}, t^{-1/2}\}$\\
        Loss regularization for \deepalgo ($c_{\mathrm{reg},\cL}$) & $1$\footnotemark\\
        NN layers & $[50,50,50,50,10,1]$\\
        NN activation & ReLu\\
        Batch size & $128$\\
        Optimizer & SGD with learning rate $0.001$  ($0.0001$ for Covertype)\\
        Regularizer least-square & $\lambda=1$\\
        Buffer capacity & $T$\\
        Scale factor for IGW (i.e., $\gamma_1$) & $\{1,10,50,100\}$\\
        Exploration rate for IGW (i.e., $\gamma_2$) & $\{1/3, 1/2\}$\\
        Scale factor for NeuralUCB ($\alpha_{\mathrm{UCB}}^{\mathrm{NeuralUCB}}$ in Eq.~\ref{eq:neuralucb.ucb}) & $\{0.1,1,2,5\}$\\
        Random Fourier Features dimension ($d$) & $\{100, 300\}$\\
        \hline
    \end{tabular}
\end{center}
\footnotetext{
    Note that in the code we add the regularization on the loss $\cL_t$ and not on the MSE.
}

All the algorithms are implemented using Pytorch~\citep{PaszkeGMLBCKLGA19pytorch}.

\paragraph{Domains.} We considered the standard domains used in previous papers~\citep[e.g.,][]{RiquelmeTS18,Zhou2020neural}.

\textit{Wheel domain.} In~\citep{RiquelmeTS18}, the authors designed a synthetic non-linear contextual bandit problem where exploration is fundamental. Contexts are samples uniformly from the unit circle in $\mathbb{R}^2$ and $|\cA|=5$ are available. The first action $a_1$ has reward $\mu(x,a_1) = \mu_1$ for all $x$. The other actions have reward $\mu_i$ when $\|x\|_2 \leq C_r$. If $\|x\|_2 > C_r$, the sign of $x_1 x_2$ defines the optimal action. For example, $a_2$ is optimal when $x_1,x_2>0$, $a_3$ if $x_1 >0$ and  $x_2 <0$ and so on. When an action $a_i \neq a_1$ is optimal the reward is $\mu_3$, otherwise is $\mu_2$ ($a_1$ has always reward $\mu_1$). We set $\mu_1=1,\mu_2=0.8,\mu_3=1.2$ and $C_r=0.5$. The reward noise is drawn from a zero-mean Gaussian distribution with standard deviation $\sigma=0.2$. For the experiments, we consider a finite subset of contexts by sampling $X=100$ contexts at the beginning of the experiment. All the repetitions are done with the same bandit problem (i.e., contexts are fix).  We samples contexts accordingly to a uniform distribution $\rho = U(\{1,\ldots, X\})$.
The features $\phi$ are obtained by concatenating the context with a one-hot encoding of the action ($d_{\phi}=7$). Let $1_i$ be the vector of dimension $5$ with all zeros except a one in position $i$, then $\phi(x,a_i) = [x,1_{i-1}]$, for all $x$, $i=1,\ldots,5$.

\textit{Dataset-based domain.} We evaluate our algorithm on standard dataset-based environments~\citep[e.g][]{RiquelmeTS18,Zhou2020neural} from the UCI repository~\citep{Blackard1998cover,Bock2004telescope,schlimmer1987concept,Dua:2019}: MAGIC Gamma Telescope Data Set, Mushroom, Statlog (Shuttle) Data Set, Covertype Data Set. We use the classical multiclass-to-bandit conversion. We use noisy rewards with Bernoulli distribution $Bern(p)$ where $p=0.9$ if the action is equal to the correct label for the sample $x$, $p=0.1$ otherwise. The features are obtained by replicating the context $|\cA|$-times, leading to a dimension $d = d_{\cX}|\cA|$ where $d_{\cX}$ is the dimension of the context. We samples contexts accordingly to a uniform distribution $\rho = U(\cX)$. We report the characteristic of the datasets after an initial preprocessing.

\begin{center}
    \begin{tabular}{ccccc}
        \hline
        &Covertype & Magic & Mushroom & Statlog (Shuttle)\\
        \hline
        Number of contexts $|\cX|$ & 581012 & 19020 & 8124 & 58000\\
        Context dimension $d_{\cX}$ & 54 & 10 & 22 & 9\\
        Number of actions $|\cA|$ & 7 & 2 & 2 & 7\\
        Feature dimension $d$ & 378 & 20 & 44 & 63\\
        \hline
    \end{tabular}
\end{center}

\begin{figure}[t]
    \centering
    \fbox{\includegraphics[width=\textwidth]{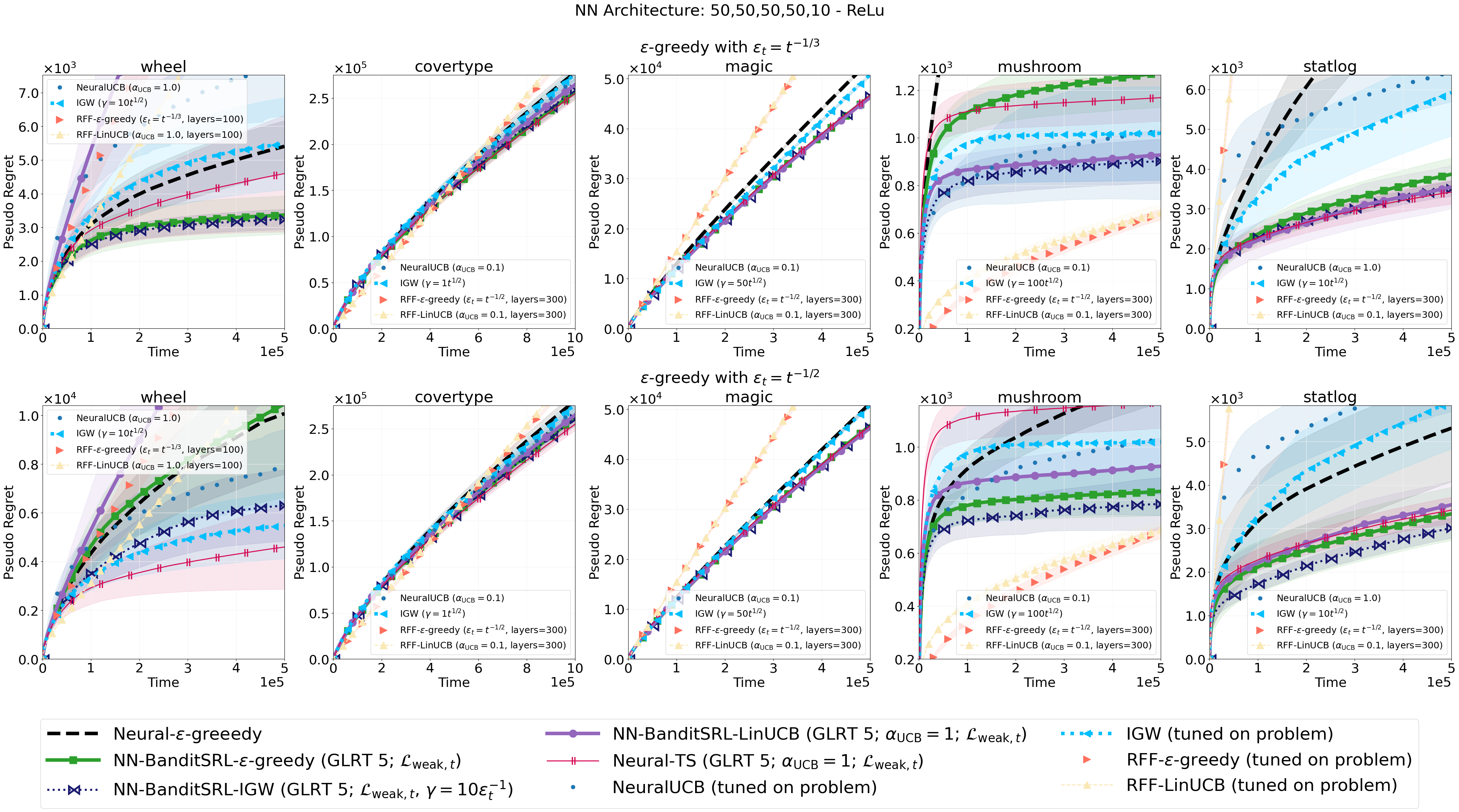}}\\
    \fbox{\includegraphics[width=\textwidth]{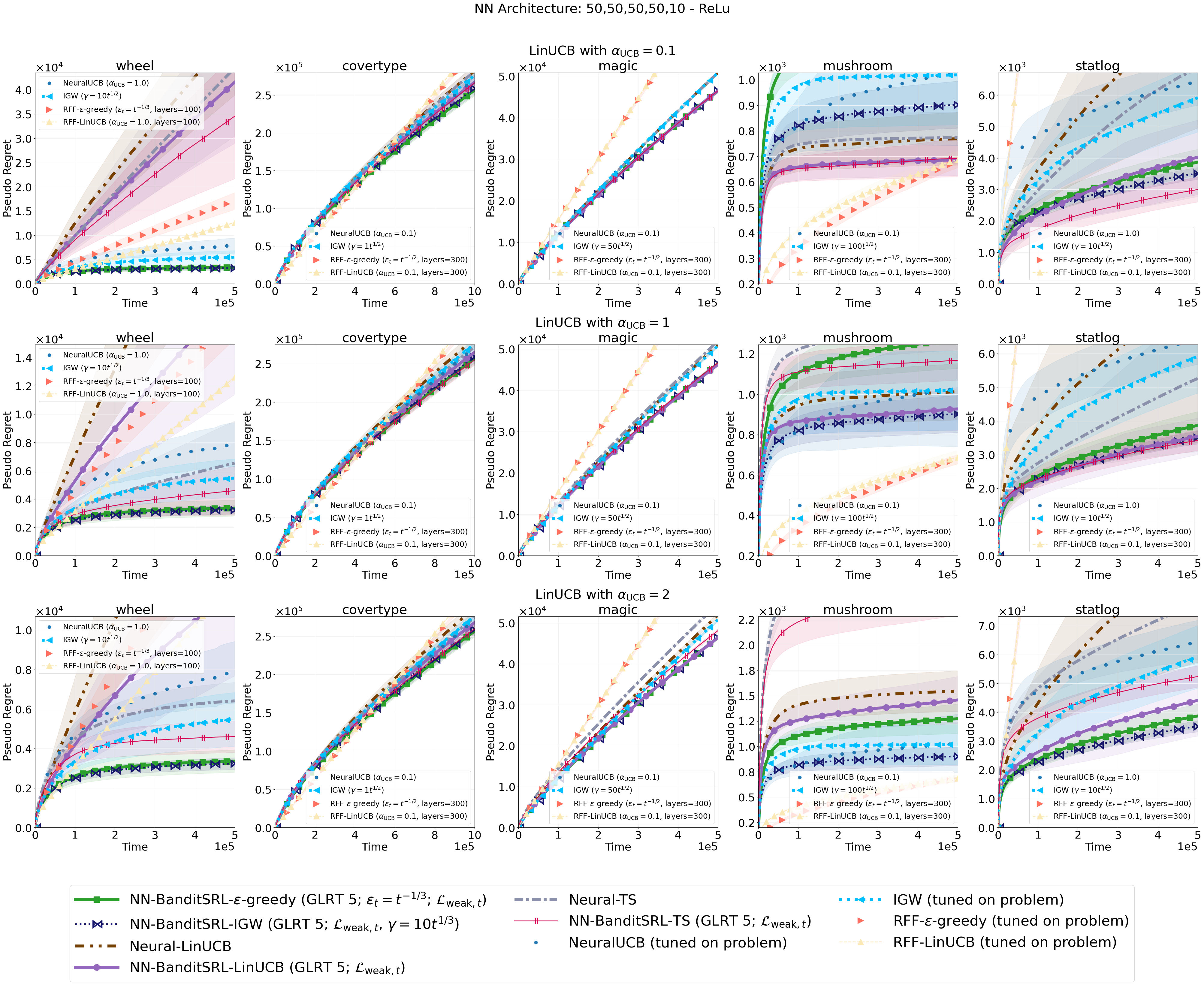}}
    \caption{Ablation study of \deepalgo with $\alpha_{\mathrm{GLRT}}=5$ and different base algorithms (i.e., $\alpha_{\mathrm{UCB}} \in \{0.1,1,2\}$, $\epsilon_t \in \{t^{-1/3}, t^{-1/2}\}$). Results are averaged over $20$ runs. We report the performance of \deepalgo{} against the best configuration of the baselines.}
    \label{fig:ablation.fixglrt}
\end{figure}

\begin{figure}[t]
    \centering
    \fbox{\includegraphics[width=\textwidth]{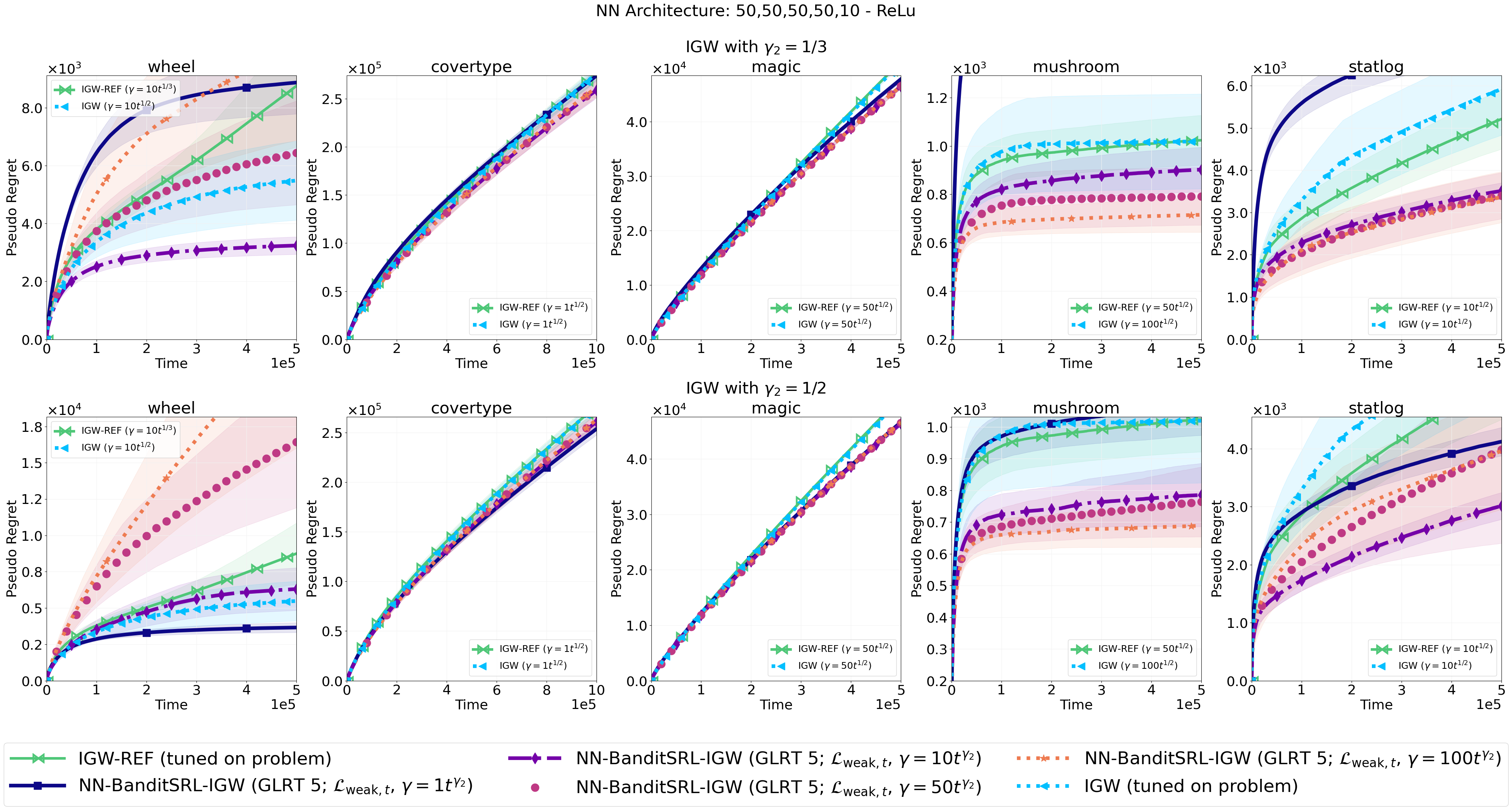}}
    \caption{Ablation study of \deepalgo with $\alpha_{\mathrm{GLRT}}=5$ and IGW strategy for different values of $\gamma_1$ and $\gamma_2$ IGW-REF denotes the variant of IGW where we refit the last layer of the NN at each time step. Results are averaged over $20$ runs.}
    \label{fig:ablation.igw}
\end{figure}

\begin{figure}[t]
    \centering
    \fbox{\includegraphics[width=\textwidth]{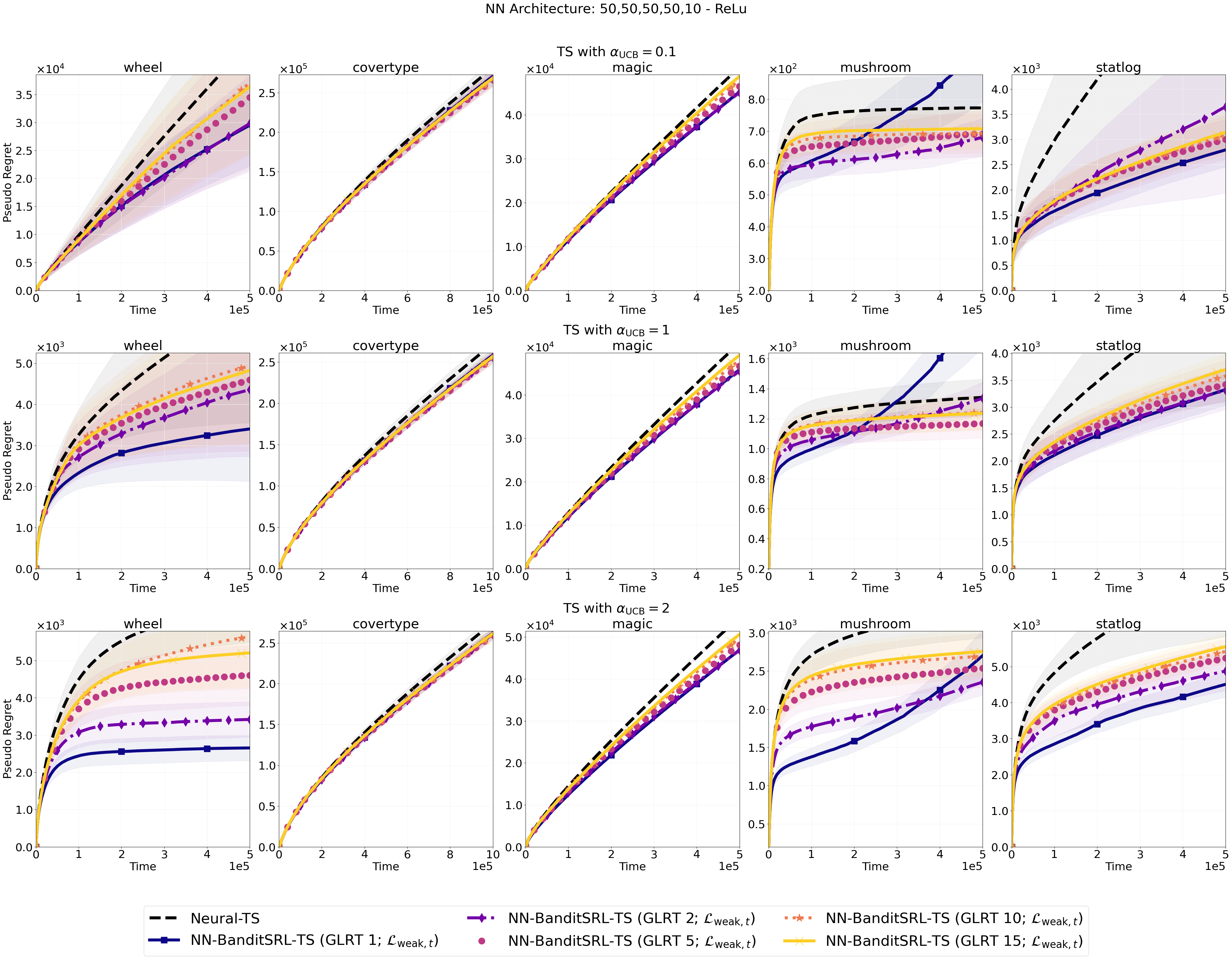}}
    \caption{Ablation study of \deepalgo with different GLRT values ($\alpha_{\mathrm{GLRT}}\in \{1,2,5,10,15\}$) for Thompson Sampling. Results are averaged over $20$ runs.}
    \label{fig:ablation.multiglrt.ts}
\end{figure}

\begin{figure}[t]
    \centering
    \fbox{\includegraphics[width=\textwidth]{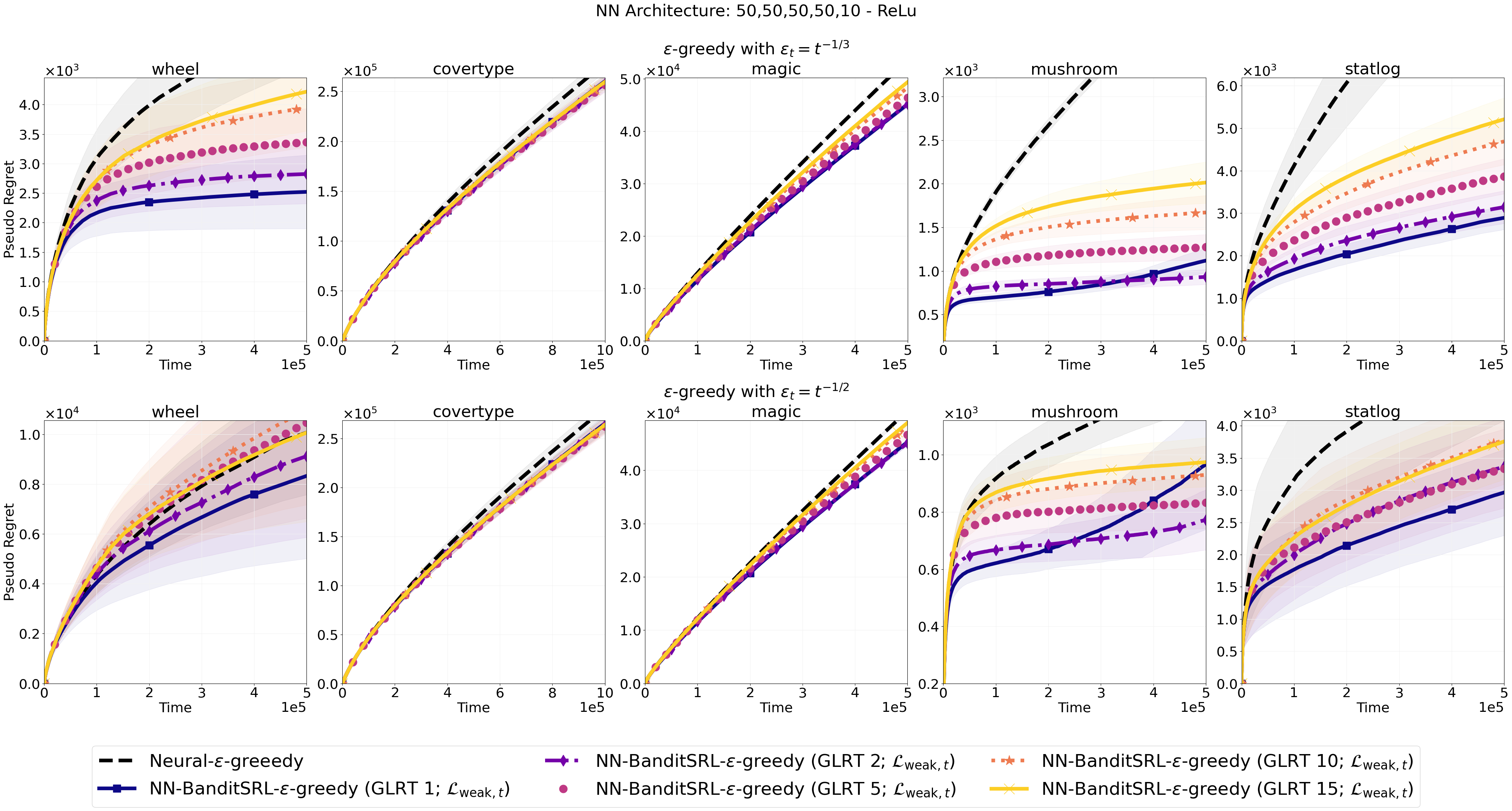}}\\
    \fbox{\includegraphics[width=\textwidth]{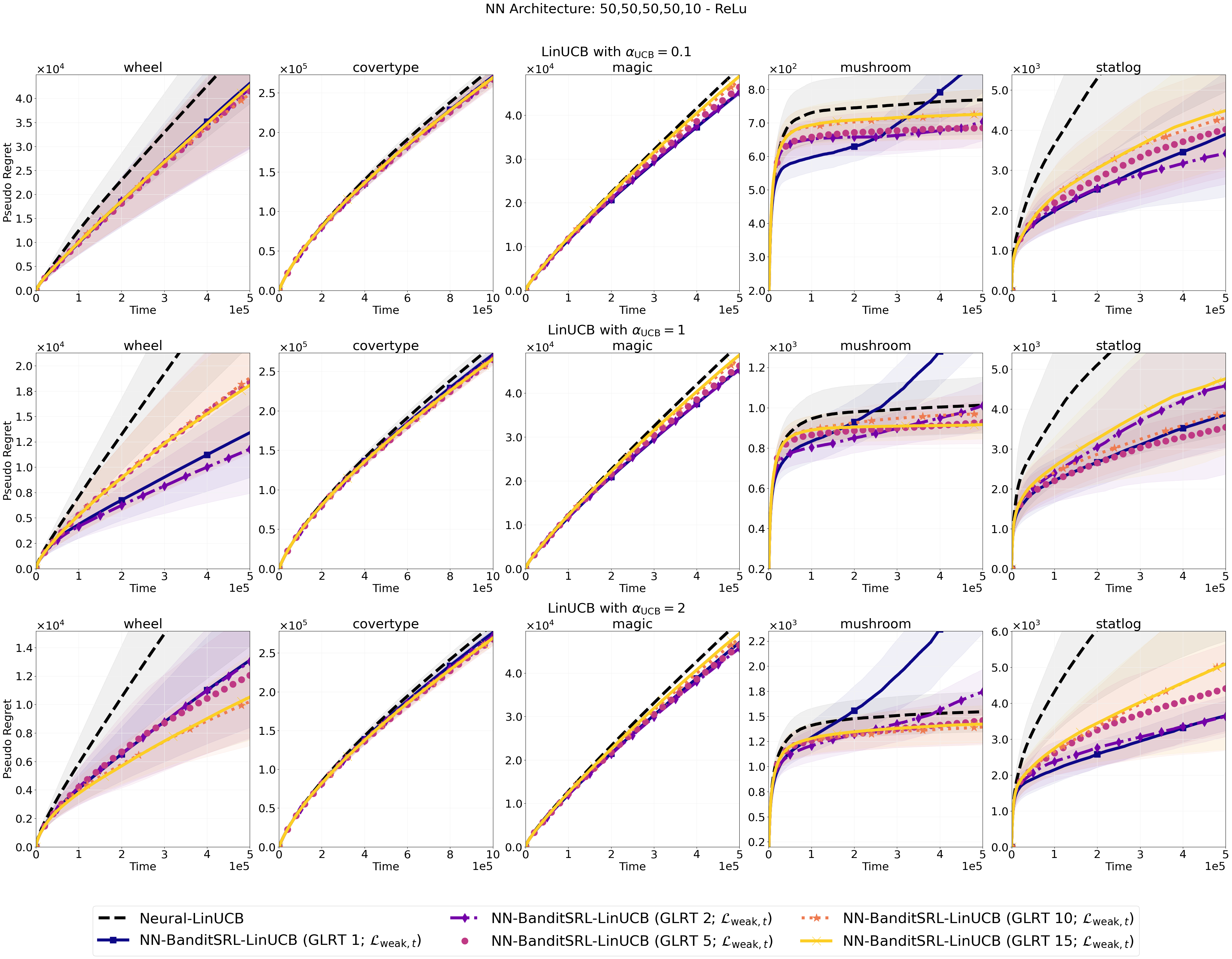}}
    \caption{Ablation study of \deepalgo with different GLRT values ($\alpha_{\mathrm{GLRT}}\in \{1,2,5,10,15\}$) and base algorithms (i.e., $\alpha_{\mathrm{UCB}} \in \{1,2\}$, $\epsilon_t \in \{t^{-1/3}, t^{-1/2}\}$). Results are averaged over $20$ runs.}
    \label{fig:ablation.multiglrt}
\end{figure}

\subsubsection{Additional Experiments and Ablation}
In this section we provide additional experiments and comparisons for \deepalgo. The overall message is that there always exists a configuration of \deepalgo{} that works well across domains and outperforms the base algorithms.

We start noticing that $\epsilon$-greedy often outperforms \linucb. Randomization at the level of actions is particularly efficient in these domains since the dimension of the output layer of the NN is always larger than the number of actions. This provides an advantage to $\epsilon$-greedy since it needs to perform less exploration. Furthermore, the GLRT prevents $\epsilon$-greedy to over explore.

In the main paper we have only reported results using the theoretical configurations of the base algorithms ($\epsilon_t = t^{-1/3}$ and $\alpha_{\mathrm{UCB}}=1$). Fig.~\ref{fig:ablation.fixglrt} shows that \deepalgo{} with $\alpha_{\mathrm{GLRT}}=5$ is robust to variations of the base algorithm. In particular, it outperforms or performs comparably to the base algorithm and the baselines in all the experiments. The interesting thing to notice is that the different domains require a different level of exploration. The wheel domain requires a high level of exploration ($\alpha_{\mathrm{UCB}}=2$ and $\epsilon_t =t^{-1/3}$), while the algorithms performs better with little exploration in mushroom ($\alpha_{\mathrm{UCB}}=0.1$ and $\epsilon_t =t^{-1/2}$). We can notice that Random Fourier Features performs poorly in almost all the experiments, supporting the need of representation learning. It may be however possible to obtain better performance by using a much higher number of features. Finally, Fig.~\ref{fig:ablation.igw} shows the behavior of \deepalgo{} with IGW strategy for different values of $\gamma_1$ and $\gamma_2$. Interestingly, it outperforms the best version of the IGW strategy based MSE.

The second experiment aims to highlight the impact of the GLRT on the behavior of \deepalgo (Fig.~\ref{fig:ablation.multiglrt}).
We can notice that the GLRT plays an important role in Neural-$\epsilon$-greedy (see also Fig.~\ref{fig:appendix_glrt_loss_ab}), in particular when using the theoretical exploration rate $t^{-1/3}$ where it significantly improve the performance. On the other hand, the GLRT may trigger too many times when $\alpha_{\mathrm{GLRT}}=1$, leading to under-exploration and worse regret. Note that there are potentially other confounding factors leading to this undesired behavior. For example, the fact we use only exploratory data may lead to suboptimal fitting of the reward if the GLRT triggers too early.
Indeed, as soon as we increase the GLRT scale factor (i.e., $\alpha_{\mathrm{GLRT}} \geq 2$), we do not see anymore a negative impact.
In general, better and more consistent results are obtained with the theoretical exploration rate $t^{-1/3}$ where over exploration is prevented by the GLRT.
The GLRT plays a milder role for \linucb-based algorithms (see also Fig.~\ref{fig:appendix_glrt_loss_ab}). Indeed, \citep{PapiniTRLP21hlscontextual} showed that \linucb is able to take advantage of the \hls{} property and does not requires a GLRT mechanism to achieve constant regret. The overall message is to set the GLRT scale factor to a value larger than the theoretical one (and larger than the one used for \linucb-based algorithms). Similar results can be derived for Thompson Sampling.

To further investigate the behavior of \deepalgo, we performed an ablation study w.r.t.\ the losses $\cL_{\mathrm{ray}}$ and $\cL_{\mathrm{weak}}$ (see Eq.~\ref{eq:ray.loss.app}-\ref{eq:weak.loss.app}) and the contribution of the GLRT (i.e., $\alpha_{\mathrm{GLRT}} \in \{0,5\}$), see Fig.~\ref{fig:appendix_glrt_loss_ab}-\ref{fig:appendix_glrt_loss_ab.ts}. We can see for Neural-$\epsilon$-greedy that the GLRT plays a fundamental role in avoiding over exploration. Furthermore, the regularization improves or at least does not degrade the performance of the algorithm. As mentioned before for \linucb-based algorithms, the GLRT does not play an important role. On the other hand, these experiments show the importance of the spectral regularization. We can indeed notice a clear separation between the performance of the algorithm with and without regularization.

\begin{figure}[t]
    \centering
    \fbox{\includegraphics[width=\textwidth]{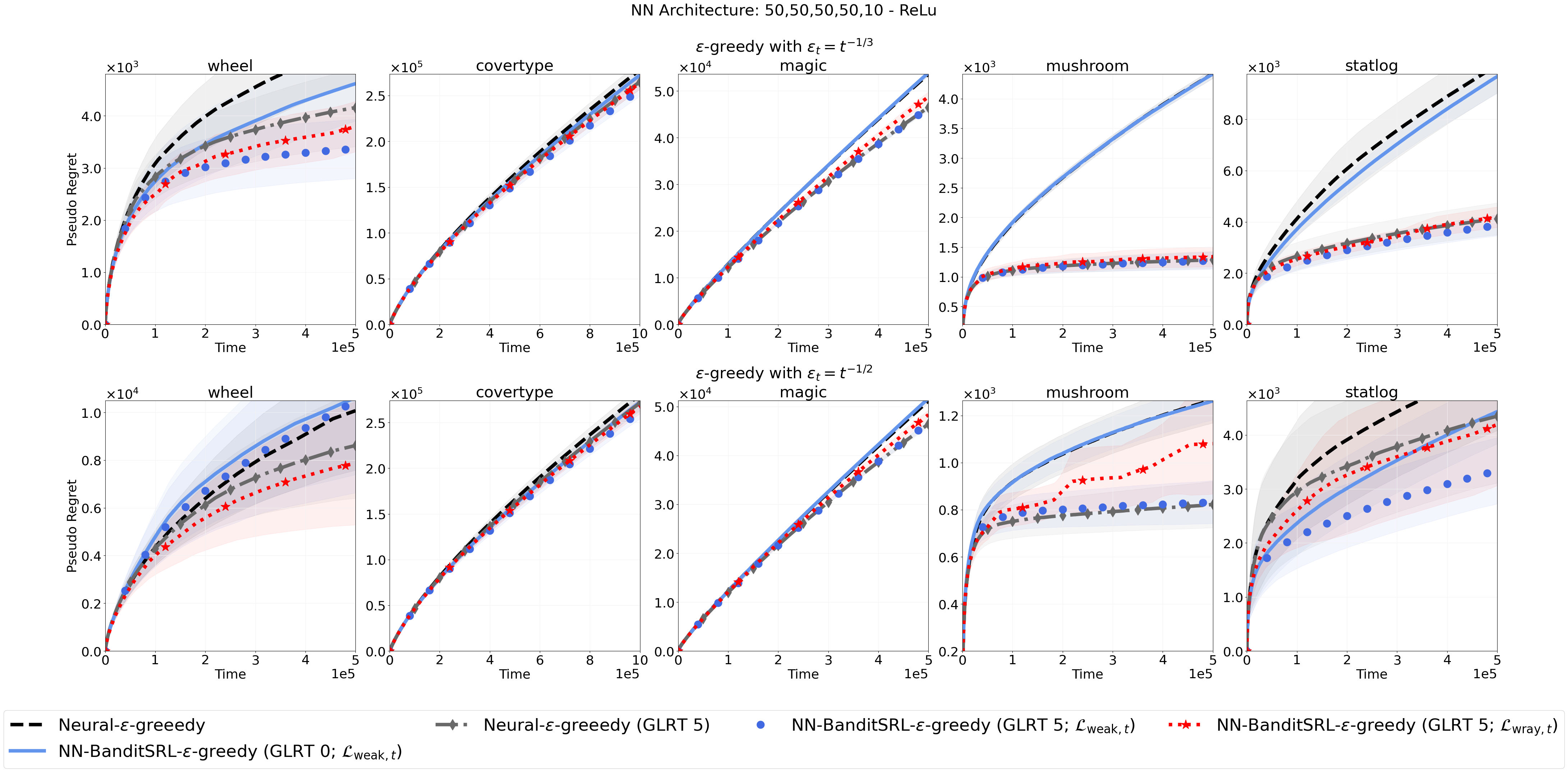}}\\
    \fbox{\includegraphics[width=\textwidth]{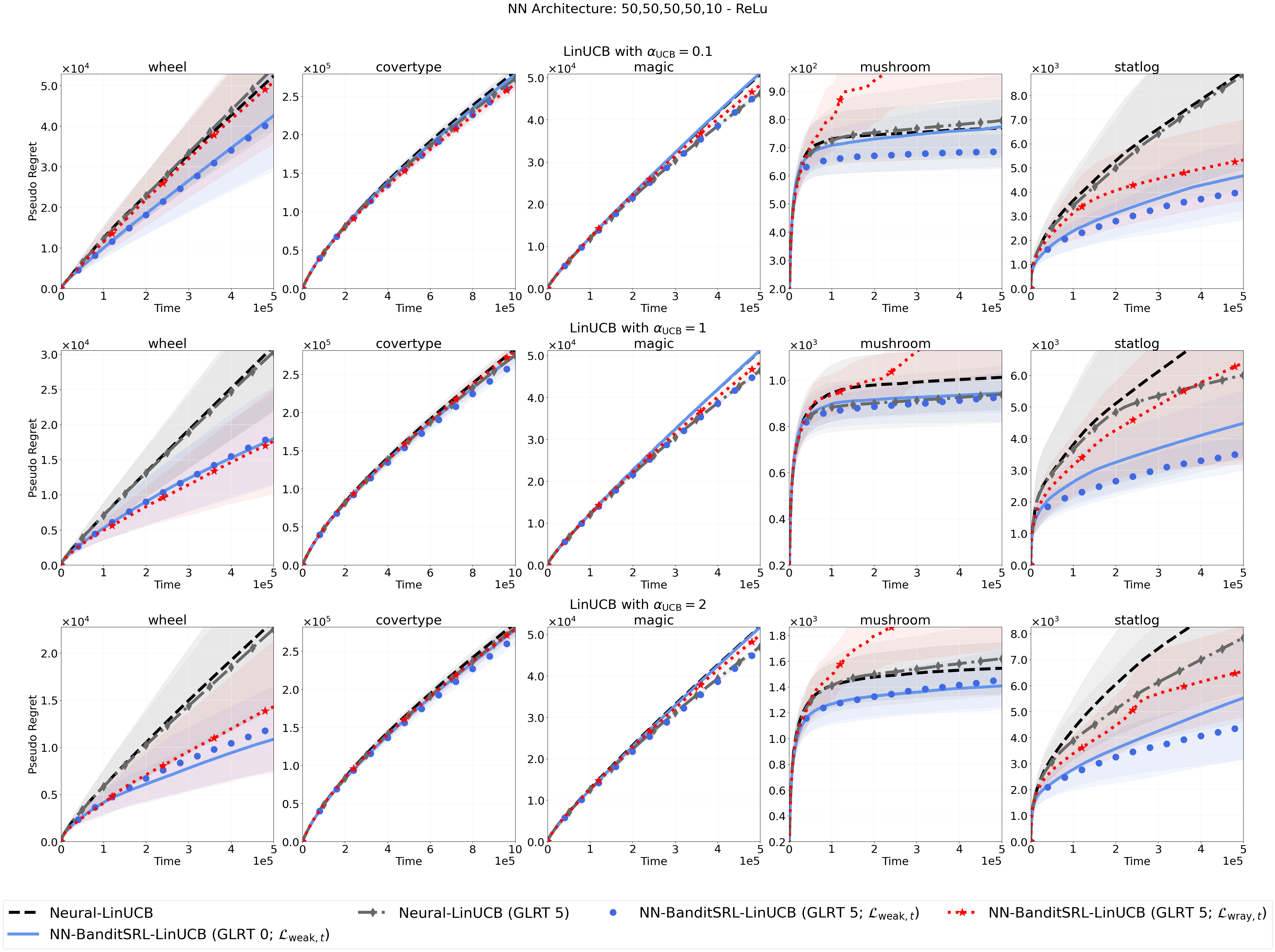}}
    \caption{Ablation study of \deepalgo with different GLRT values ($\alpha_{\mathrm{GLRT}}\in \{0,5\}$), base algorithms (i.e., $\alpha_{\mathrm{UCB}} \in \{1,2\}$, $\epsilon_t \in \{t^{-1/3}, t^{-1/2}\}$) and regularization loss. Results are averaged over $20$ runs.}
    \label{fig:appendix_glrt_loss_ab}
\end{figure}
\begin{figure}[t]
    \centering
    \fbox{\includegraphics[width=\textwidth]{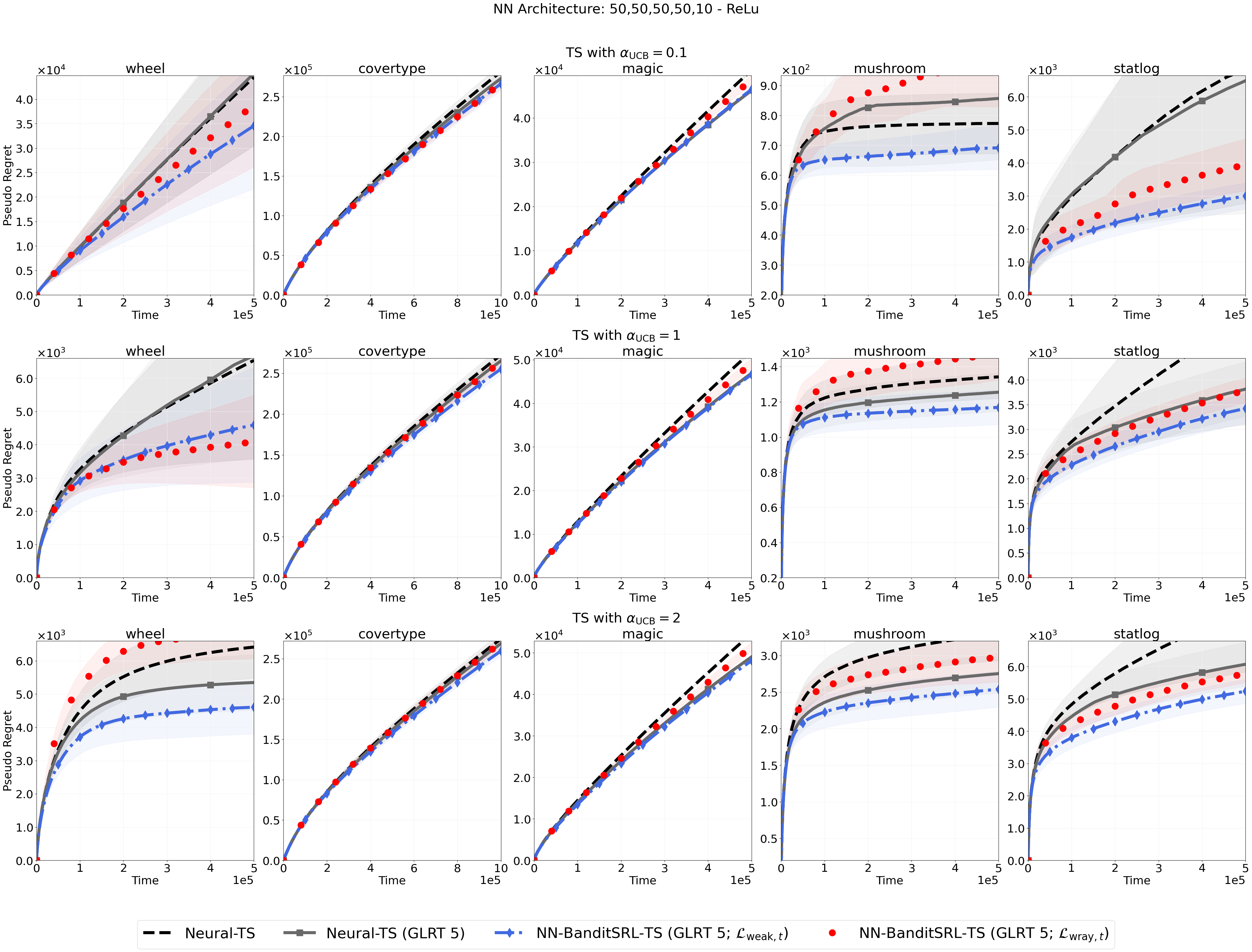}}
    \caption{Ablation study of \deepalgo with different GLRT values ($\alpha_{\mathrm{GLRT}}\in \{0,5\}$) and regularization loss for Thompson Sampling. Results are averaged over $20$ runs.}
    \label{fig:appendix_glrt_loss_ab.ts}
\end{figure}

\clearpage

\subsubsection{Network study on the Wheel Domain}
To further investigate the behavior of \deepalgo, we performed an ablation study w.r.t.\ the network structure.

Let's start considering $\epsilon$-greedy algorithms. Fig.~\ref{fig:netablation.wheel.greedy.linucb} that the performance of these algorithms does not vary much across the experiments. However, there are interesting things to notice. When the embedding layer is large (1000,100), the regularization and GLRT do not help and \deepalgo{} behaves as the Neural-$\epsilon$-greedy algorithm. Indeed it may be difficult to recover spectral properties for such a large representation (the original feature dimension is 7). Similarly the GLRT scales with the dimension $d$, the higher $d$ the larger may be the time to trigger the test. When the embedding dimension is smaller, we can see an improved performance for \deepalgo{} compared to the base algorithm. The best regret is obtained with the deepest network and smallest embedding dimension (i.e., 10). In particular, we can see a flattening curve for \deepalgo{} with net $[50,50,50,50,10]$ that is not observe with embedding dimension $50$.

\linucb-based algorithms suffer when the embedding dimension is large (i.e., 1000, 100) since it needs to perform much more exploration compared to $\epsilon$-greedy. Indeed, $\epsilon$-greedy only needs to do exploration at the level of the 5 actions, while \linucb needs to explore the $d$-dimensional space. An interesting behavior is observed with deeper networks. In particular, we observe a better performance with embedding dimension 50 rather than 10. We think that with dimension 10 the network has a larger misspecification that compromises the exploration performed by $\linucb$-based algorithms. Indeed, Fig.~\ref{fig:netablation.wheel.linucb.ts} shows that both \deepalgo{} and Neural-\linucb{} show a linear regret. This demonstrates that i) \linucb-based algorithms are much more sensible to the misspecification than $\epsilon$-greedy; ii) it is important to carefully select the embedding dimension $d$ (the larger the higher the level of exploration but the smaller the misspecification). On the other hand, when $d=50$, \linucb-based algorithms perform comparably to $\epsilon$-greedy. While with a shallow network (i.e., $[50,50,50]$) we observe a small improvement in using \deepalgo{}, the advantages of \deepalgo{} becomes extremely clear with the deep network (i.e., $[50,50,50,50,50]$) where it achieves more than half of the regret of Neural-\linucb.

Finally, Fig.~\ref{fig:netablation.wheel.linucb.ts} shows that, similarly to $\epsilon$-greedy, Thompson Sampling works better with smaller dimensions (in particular 10) where we can always observe a smaller regret for \deepalgo{}.

\begin{figure}[t]
    \centering
    \includegraphics[width=.98\textwidth]{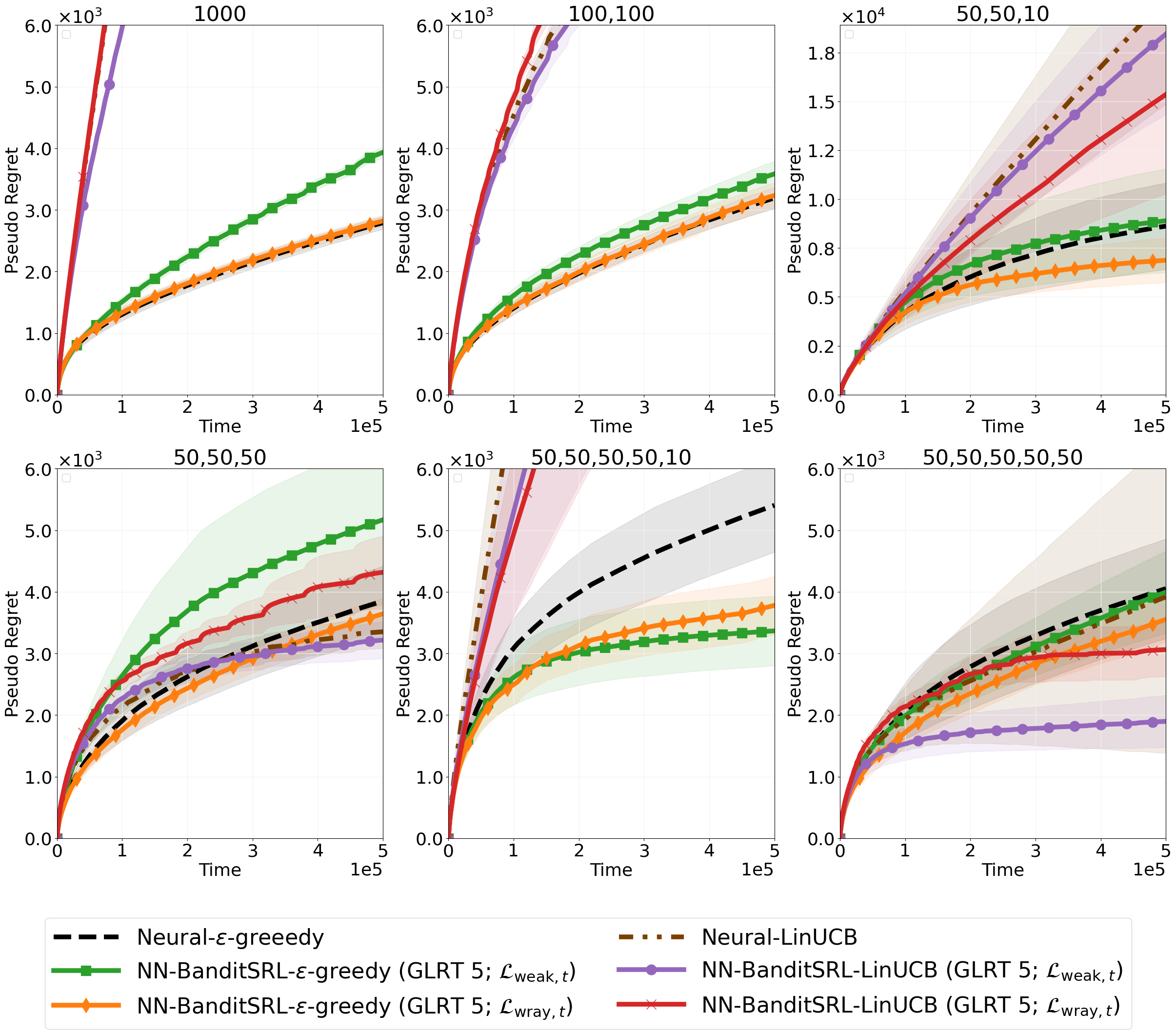}
    \caption{Ablation study of \deepalgo with $\epsilon$-greedy and \linucb{} on the Wheel domain. Results are averaged over $20$ runs. The figure title corresponds to the network dimension.}
    \label{fig:netablation.wheel.greedy.linucb}
\end{figure}

\begin{figure}[t]
    \centering
    \includegraphics[width=.98\textwidth]{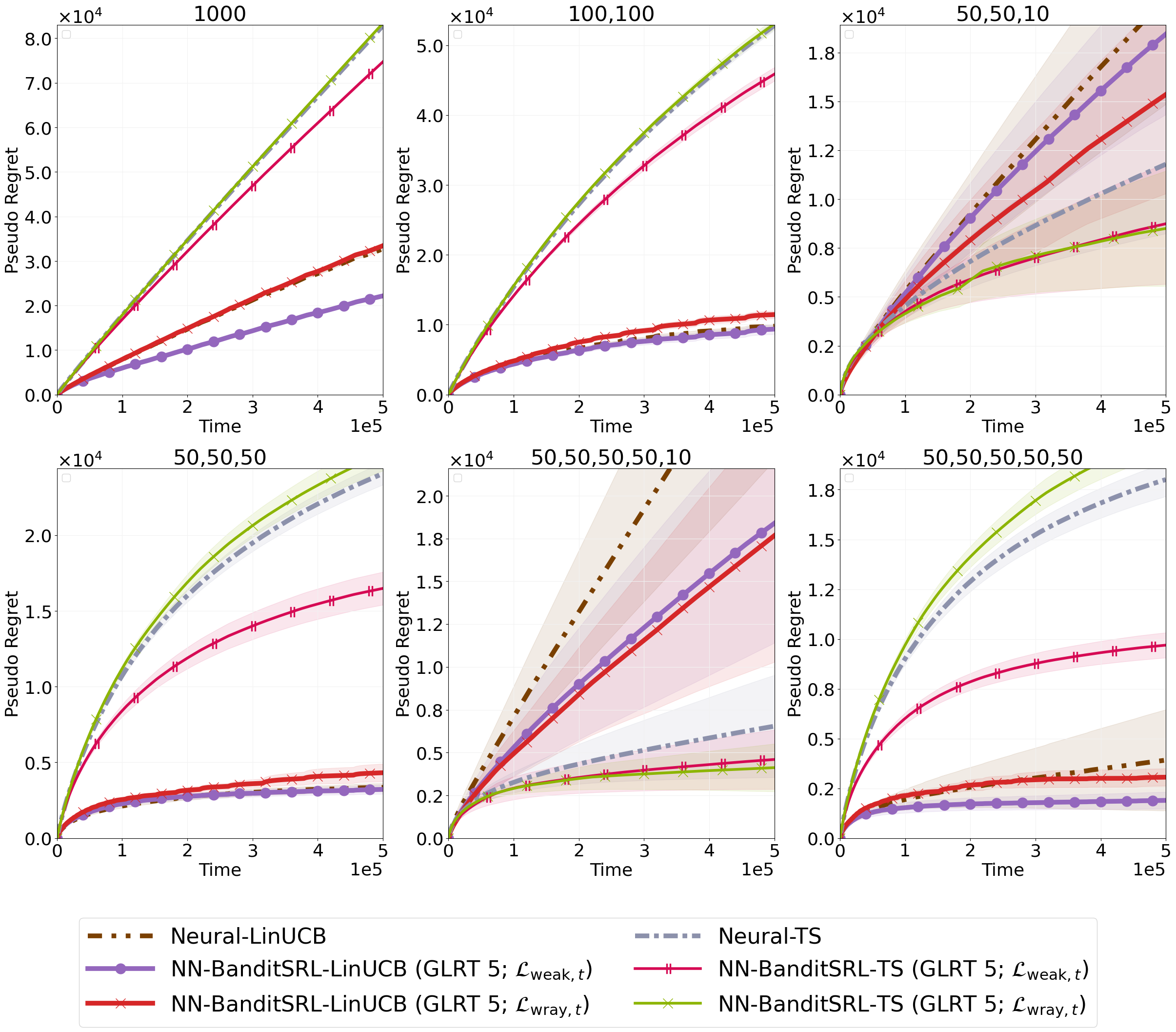}
    \caption{Ablation study of \deepalgo with \linucb{} and TS on the Wheel domain. Results are averaged over $20$ runs. The figure title corresponds to the network dimension.}
    \label{fig:netablation.wheel.linucb.ts}
\end{figure}


\section{Examples of No-regret Algorithms}

We prove that LinUCB and $\epsilon$-greedy satisfy Assumption \ref{asm:no-regret-algo}. Then, we instantiate our general regret bounds (i.e., we bound $\tau_{\mathrm{alg}}$ defined in Lemma \ref{lem:scale-tau-hls}) for these specific algorithms.

\subsection{LinUCB}

\begin{theorem}[Regret bound of anytime LinUCB, Prop. 1 in \citep{PapiniTRLP21hlscontextual}]\label{th:regret-linucb}
    Let $\phi\in \Phi^\star$ be any realizable representation. With probability $1-\delta$, for any $T\in\bN$, the regret of anytime LinUCB run with representation $\phi$, confidence $\delta$, and threshold $\beta_{t,\delta}(\phi)$ is bounded as
    \begin{align*}
        R_T \leq \wb{R}_{\mathrm{LinUCB}}(T, \phi, \delta), =: \frac{128\lambda B_\phi^2\sigma^2\left(2\log(1/\delta)+d_\phi \log(1+TL_\phi^2/(\lambda d_\phi))\right)^2}{\Delta}.
    \end{align*}
\end{theorem}
\begin{proof}
    Just apply Proposition 1 in \citep{PapiniTRLP21hlscontextual} while noting that the maximum per-step regret is $2$ in our context.
\end{proof}

\begin{lemma}\label{lem:tau-alg-linucb}
    When using the LinUCB algorithm, we have
    \begin{align*}
        \tau_{\mathrm{alg}} \lesssim \frac{L_{\phi^\star}^2 d^2 \log(|\Phi|/\delta)^2}{\lambda^\star(\phi^\star)\Delta^2}.
      \end{align*}
\end{lemma}
\begin{proof}
    First note that, by Theorem \ref{th:regret-linucb},
    \begin{align*}
        \wb{R}_{\mathrm{LinUCB}}(t, \phi, \delta_{\log_2(t)}/|\Phi|) \lesssim \frac{d_\phi^2\log(t|\Phi|/\delta)^2}{\Delta}.
    \end{align*}
    Then, the result follows by applying Lemma \ref{lem:ineq-log-sqrt}.
\end{proof}

\subsection{$\epsilon$-greedy}\label{app:egreedy.analysis}

\begin{theorem}[Regret bound of $\epsilon$-greedy]\label{th:regret-epsgreedy}
    Let $\phi\in \Phi^\star$ be any realizable representation. With probability $1-\delta$, for any $T\in\bN$, the regret of $\epsilon$-greedy run with representation $\phi$, confidence $\delta$, and forcing schedule $(\epsilon_t)_{t\geq 1}$ with $\epsilon_t = 1/t^{1/3}$ is bounded as
    \begin{align*}
        R_T \leq \wb{R}_{\epsilon\mathrm{-greedy}}(T, \phi, \delta), &=: 2\beta_{T,\delta/3}(\phi) \left( \frac{L_\phi}{\sqrt{\lambda}} \left(\frac{128L_\phi^2 A\sqrt{\log(12d_\phi/\delta)}}{\Gamma(\phi)}\right)^8 + \frac{2L_\phi}{\sqrt{\lambda}} + \frac{3L_\phi \sqrt{A} T^{2/3}}{\sqrt{\Gamma(\phi)}} \right)
        \\ &+ 2\sqrt{T\log(6T/\delta)} + 3 T^{2/3},
    \end{align*}
    where $\Gamma(\phi) := \lambda_{\min}\left(\bE_{x\sim\rho}\left[\sum_{a\in\cA}\phi(x,a)\phi(x,a)^\transp\right]\right)$ and $\beta_{T,\delta}(\phi)  := \sigma\sqrt{2\log(1/\delta)+d_{\phi}\log(1+TL_{\phi}^2/(\lambda d_{\phi}))} + \sqrt{\lambda}B_{\phi}$.
\end{theorem}
\begin{proof}
    Let $F_t$ be the event under which the algorithm plays greedily at time $t$. Then,
    \begin{align*}
        R_T = \underbrace{\sum_{t=1}^T \indi{F_t} \Delta(x_t,a_t)}_{(a)} + \underbrace{\sum_{t=1}^T \indi{\neg F_t} \Delta(x_t,a_t)}_{(b)}.
    \end{align*}
    Let us start from (a). With probability at least $1-\delta$, we have that, under $F_t$,
    \begin{align*}
        & \Delta(x_t,a_t) = \max_{a\in\cA}\mu(x_t,a) - \mu(x_t,a_t)
        \\ & \quad \leq \max_{a\in\cA} \left(\langle \theta_{\phi,t-1}, \phi(x_t,a)\rangle + \beta_{t-1,\delta}(\phi)\|\phi(x_t,a)\|_{V_{t-1}^{-1}(\phi)}\right) - \langle \theta_{\phi,t-1}, \phi(x_t,a_t)\rangle + \beta_{t-1,\delta}(\phi)\|\phi(x_t,a_t)\|_{V_{t-1}^{-1}(\phi)}
        \\ & \quad\leq \max_{a\in\cA} \langle \theta_{\phi,t-1}, \phi(x_t,a)\rangle - \langle \theta_{\phi,t-1}, \phi(x_t,a_t)\rangle + 2\max_{a\in\cA}\beta_{t-1,\delta}(\phi)\|\phi(x_t,a)\|_{V_{t-1}^{-1}(\phi)}
        \\ & \quad = 2\max_{a\in\cA}\beta_{t-1,\delta}(\phi)\|\phi(x_t,a)\|_{V_{t-1}^{-1}(\phi)},
    \end{align*}
    where the last equality is because $a_t$ is greedy w.r.t. $\theta_{\phi,t-1}$ under $F_t$. Then,
    \begin{align*}
        (a) &\leq 2\beta_{T,\delta}(\phi) \sum_{t=1}^T \indi{F_t} \max_{a\in\cA}\|\phi(x_t,a)\|_{V_{t-1}^{-1}(\phi)} \leq 2\beta_{T,\delta}(\phi) \sum_{t=1}^T \indi{F_t}  \frac{L_\phi}{\sqrt{\lambda_{\min}(V_{t-1}(\phi))}}.
    \end{align*}
    Let $\bE_t$ be the expectation operator conditioned on the full history up to round $t-1$ and $\pi_t(a | x) = (1-\epsilon_t)\indi{a = \argmax_{a\in\cA}\langle \theta_{\phi,t-1}, \phi(x_t,a)\rangle } + \frac{\epsilon_t}{|\cA|}$ be the stochastic policy played at time $t$. By Matrix Azuma inequality (Lemma \ref{lem:mazuma}) and a union bound on time, with probability at least $1-\delta$,
    \begin{align*}
        & \lambda_{\min}(V_{t-1}(\phi)) \geq \lambda + \lambda_{\min}\left(\sum_{k=1}^{t-1} \bE_{k}\left[\phi(x,a)\phi(x,a)^\transp\right]\right) - 8L_\phi^2\sqrt{(t-1)\log(4d_\phi (t-1)/\delta)}
        \\ & \quad = \lambda + \lambda_{\min}\left(\sum_{k=1}^{t-1} \bE_{x\sim\rho,a\sim\pi_k(\cdot|x)}\left[\phi(x,a)\phi(x,a)^\transp\right]\right) - 8L_\phi^2\sqrt{(t-1)\log(4d_\phi (t-1)/\delta)}
        \\ & \quad \geq \lambda + \lambda_{\min}\left(\sum_{k=1}^{t-1} \epsilon_k \bE_{x\sim\rho,a\sim\cU(\cA)}\left[\phi(x,a)\phi(x,a)^\transp\right]\right) - 8L_\phi^2\sqrt{(t-1)\log(4d_\phi (t-1)/\delta)}
        \\ & \quad = \lambda + \frac{\Gamma(\phi)}{A} \sum_{k=1}^{t-1} \epsilon_k - 8L_\phi^2\sqrt{(t-1)\log(4d_\phi (t-1)/\delta)}
        \\ & \quad \geq \lambda + \frac{\Gamma(\phi)}{A} (t-1)^{2/3} - 8L_\phi^2\sqrt{(t-1)\log(4d_\phi (t-1)/\delta)},
    \end{align*}
    where in the last step we used the definition of $\epsilon_k$. We now seek a condition on $t$ such that $8L_\phi^2\sqrt{(t-1)\log(4d_\phi (t-1)/\delta)} \leq \frac{\Gamma(\phi) (t-1)^{2/3}}{2A}$, so that we have $\lambda_{\min}(V_{t-1}(\phi)) \geq \lambda + \frac{\Gamma(\phi) (t-1)^{2/3}}{2A}$. By the crude bound $\log(x) \leq x^\alpha/\alpha$, we have 
    \begin{align*}
        8L_\phi^2\sqrt{(t-1)\log(4d_\phi (t-1)/\delta)} \leq 8L_\phi^2\sqrt{(t-1)\log(4d_\phi/\delta)} + 8L_\phi^2\sqrt{(t-1)^{1+\alpha}/\alpha}.
    \end{align*}
    Thus, a sufficient condition is that
    \begin{align*}
        8L_\phi^2\sqrt{(t-1)\log(4d_\phi/\delta)} &\leq \frac{\Gamma(\phi) (t-1)^{2/3}}{4A} \implies (t-1) \geq \left(\frac{32L_\phi^2 A\sqrt{\log(4d_\phi/\delta)}}{\Gamma(\phi)}\right)^6,
        \\ 8L_\phi^2\sqrt{(t-1)^{1+\alpha}/\alpha} &\leq \frac{\Gamma(\phi) (t-1)^{2/3}}{4A} \implies (t-1) \geq \left(\frac{32L_\phi^2 A\sqrt{1/\alpha}}{\Gamma(\phi)}\right)^\frac{6}{4-3(1+\alpha)}.
    \end{align*}
    Setting $\alpha=1/12$, we have $\frac{6}{4-3(1+\alpha)} = 8$. Then, a sufficient condition is
    \begin{align*}
        t \geq z := \left(\frac{128L_\phi^2 A\sqrt{\log(4d_\phi/\delta)}}{\Gamma(\phi)}\right)^8 + 1.
    \end{align*}
    Then, 
    \begin{align*}
        \sum_{t=1}^T \indi{F_t}  \frac{L_\phi}{\sqrt{\lambda_{\min}(V_{t-1}(\phi))}} \leq z \frac{L_\phi}{\sqrt{\lambda}} + \sum_{t=1}^T \frac{L_\phi}{\sqrt{\lambda + \frac{\Gamma(\phi) (t-1)^{2/3}}{2A}}} &\leq (z+1) \frac{L_\phi}{\sqrt{\lambda}} + \frac{\sqrt{2A}}{\sqrt{\Gamma(\phi)}}\sum_{t=1}^T \frac{L_\phi}{t^{1/3}}
        \\ & \leq (z+1) \frac{L_\phi}{\sqrt{\lambda}} + \frac{3L_\phi \sqrt{A}T^{2/3}}{\sqrt{\Gamma(\phi)}}.
    \end{align*}
    Thus,
    \begin{align*}
        (a) \leq 2\beta_{T,\delta}(\phi) \left( \frac{L_\phi}{\sqrt{\lambda}} \left(\frac{128L_\phi^2 A\sqrt{\log(4d_\phi/\delta)}}{\Gamma(\phi)}\right)^8 + \frac{2L_\phi}{\sqrt{\lambda}} + \frac{3L_\phi \sqrt{A}T^{2/3}}{\sqrt{\Gamma(\phi)}} \right).
    \end{align*}
    Let us bound (b). By Azuma's inequality (Lemma \ref{lemma:azuma}), with probability at least $1-\delta$,
    \begin{align*}
        (b) &\leq 2\sum_{t=1}^T \indi{\neg F_t} = 2\sum_{t=1}^T \Big(\indi{\neg F_t} - \bP(\neg F_t)\Big) + 2 \sum_{t=1}^T \bP(\neg F_t)
        \\ & \leq 2\sqrt{T\log(2T/\delta)} +  2\sum_{t=1}^T \epsilon_t =  2\sqrt{T\log(2T/\delta)} +  2\sum_{t=1}^T \frac{1}{t^{1/3}} \leq 2\sqrt{T\log(2T/\delta)} + 3 T^{2/3}.
    \end{align*}
    Summing the bounds on (a) and (b) yields a regret bound that holds with probability at least $1-3\delta$ by the three concentration events used above. Then, the result follows by a union bound, i.e., by re-defining $\delta \rightarrow \delta/3$.
\end{proof}

\begin{lemma}\label{lem:tau-alg-epsgreedy}
    When using the $\epsilon$-greedy algorithm (same conditions as in Theorem \ref{th:regret-epsgreedy}), we have
    \begin{align*}
        \tau_{\mathrm{alg}} \lesssim \frac{L_{\phi^\star}^6 (dA)^{3/2} L^3 \log(|\Phi|/\delta)^3}{\lambda^\star(\phi^\star)^3\Delta^3}.
      \end{align*}
\end{lemma}
\begin{proof}
    First note that, by Theorem \ref{th:regret-epsgreedy},
    \begin{align*}
        \wb{R}_{\epsilon\mathrm{-greedy}}(t, \phi, \delta_{\log_2(t)}/|\Phi|) \lesssim L_\phi \sqrt{d_\phi A} \log(t|\Phi|/\delta) t^{2/3},
    \end{align*}
    where we kept only the higher-order dependences. Then, with similar steps as in the proof of Lemma \ref{lem:ineq-log-sqrt}, one can easily show that $\tau_{\mathrm{alg}}$ requires solving the inequality
    \begin{align*}
        t \lesssim \frac{L_{\phi^\star}^2}{\lambda^\star(\phi^\star)\Delta} \max_{\phi\in\Phi^\star} L_\phi \sqrt{d_\phi A} \log(|\Phi|/\delta) t^{2/3},
    \end{align*}
    which proves the statement.
\end{proof}

\section{Auxiliary Results}

\subsection{Bounding the eigenvalues of the design matrices}

The following result holds for any algorithm (i.e., any arm selection rule) any any representation $\phi$ (even non-realizable). It is an extension of Lemma 9 in \citep{PapiniTRLP21hlscontextual}.
\begin{lemma}\label{lem:bound-design}
Under the assumption that the optimal policy is unique, with probability $1-\delta$, for all $t$ and $\phi\in\Phi$,
\begin{equation}
 V_{t}(\phi) \succeq t\EV_{x \sim \rho}[\phi(x,\pi^\star(x))\phi(x,\pi^\star(x))^\transp] + \left( \lambda - L_\phi^2 S_t - 8L_\phi^2\sqrt{t\log(4d_\phi |\Phi|t/\delta)} \right) I_{d_\phi},
\end{equation}
\begin{equation}
  V_{t}(\phi) \preceq t\EV_{x \sim \rho}[\phi(x,\pi^\star(x))\phi(x,\pi^\star(x))^\transp] + \left( \lambda + L_\phi^2 S_t + 8L_\phi^2\sqrt{t\log(4d_\phi |\Phi|t/\delta)} \right) I_{d_\phi},
 \end{equation}
where $S_t := \sum_{k=1}^t \indi{a_k\neq \pi^\star(x_k)}$.
\end{lemma}
\begin{proof}
  The lower bound holds with probability $1-\delta/2$ by \cite[][Lemma 9]{PapiniTRLP21hlscontextual}. Let us prove the upper bound. We have
  \begin{align*}
    V_{t}(\phi) &- \lambda I_{d_\phi} =\sum_{k=1}^t\phi(x_k,a_k)\phi(x_k,a_k)^\transp \\
    &= \sum_{k=1}^t\indi{a_k\neq\pi^\star(x_k)}\phi(x_k,a_k)\phi(x_k,a_k)^\transp + \sum_{k=1}^t \indi{a_k=\pi^\star(x_k)} \phi(x_k,a_k)\phi(x_k,a_k)^\transp \\
    &\preceq \sum_{k=1}^t\indi{a_k\neq\pi^\star(x_k)}\phi(x_k,a_k)\phi(x_k,a_k)^\transp + \sum_{k=1}^t \phi(x_k,\pi^\star(x_k))\phi(x_k,\pi^\star(x_k))^\transp \\
    &\preceq L_\phi^2 S_t I_{d_\phi}
    +  \sum_{k=1}^t\phi(x_k,\pi^\star(x_k))\phi(x_k,\pi^\star(x_k))^\transp \\
    &\preceq L_\phi^2 S_t I_{d_\phi} + t\EV_{x \sim \rho}[\phi(x,\pi^\star(x))\phi(x,\pi^\star(x))^\transp] + 8L_\phi^2\sqrt{t\log(4d_\phi t/\delta)} I_{d_\phi},
\end{align*}
where the second-last inequality uses the boundedness of $\phi$, while the last one holds with probability $1-\delta/2$ for all $t$ by Lemma \ref{lem:mazuma} and a union bound. The result follows by a union bound on $\Phi$ and on the two sides of the inequality.
\end{proof}

\subsection{Martingale concentration}

We restate some well-known martingale concentration bounds.

\begin{lemma}[Azuma's inequality]\label{lemma:azuma}
  Let $\{(Z_t,\mathcal{F}_t)\}_{t\in\mathbb{N}}$ be a martingale difference sequence such that $|Z_t| \leq a$ almost surely for all $t\in\mathbb{N}$. Then, for all $\delta \in (0,1)$,
  \begin{align*}
  \mathbb{P}\left(\forall t \geq 1 : \left|\sum_{k=1}^t Z_k \right| \leq a\sqrt{t \log(2t/\delta)} \right) \geq 1-\delta.
  \end{align*}
  \end{lemma}
  
\begin{lemma}[Freedman's inequality]\label{lemma:freedman}
  Let $\{(Z_t,\mathcal{F}_t)\}_{t\in\mathbb{N}}$ be a martingale difference sequence such that $|Z_t| \leq a$ almost surely for all $t\in\mathbb{N}$. Then, for all $\delta \in (0,1)$,
  \begin{align*}
  \mathbb{P}\left(\forall t \geq 1 : \left|\sum_{k=1}^t Z_k \right| \leq 2\sqrt{\sum_{k=1}^t \mathbb{V}_k[Z_k] \log(4t/\delta)} + 4a\log(4t/\delta) \right) \geq 1-\delta.
  \end{align*}
\end{lemma}

\begin{lemma}[Matrix Azuma's inequality]\label{lem:mazuma}
	Let $\{X_k\}_{k=1}^t$ be a finite adapted sequence of symmetric matrices of dimension $d$, and $\{C_k\}_{k=1}^t$ a sequence of symmetric matrices such that for all $k$, $\EV_{k}[X_k]=0$ and $X_k^2\preceq C_k^2$ almost surely. Then, with probability at least $1-\delta$,
	\begin{equation}
	\lambda_{\max}\left(\sum_{k=1}^tX_k\right) \le \sqrt{8\norm{\sum_{k=1}^tC_k^2}\log(d/\delta)}.
	\end{equation}
\end{lemma}

\end{document}